\documentclass{article}

\usepackage[final,nonatbib]{nips_2018}

\usepackage[utf8]{inputenc} % allow utf-8 input
\usepackage[T1]{fontenc}    % use 8-bit T1 fonts
\usepackage{hyperref}       % hyperlinks
\usepackage{url}            % simple URL typesetting
\usepackage{booktabs}       % professional-quality tables
\usepackage{amsfonts}       % blackboard math symbols
\usepackage{nicefrac}       % compact symbols for 1/2, etc.
\usepackage{microtype}      % microtypography
\usepackage{amsmath}
\usepackage{graphicx}
\usepackage{subfigure}
\usepackage{amssymb}
\usepackage{cancel}
\usepackage{enumitem}
\usepackage{color}
\usepackage[ruled,linesnumbered]{algorithm2e}
\usepackage{multirow}

\title{Efficient Loss-Based Decoding on Graphs for Extreme Classification}

\author{
  Itay Evron \\
  Computer Science Dept.\\
  The Technion, Israel \\
  \texttt{evron.itay@gmail.com} \\
  \And
  Edward Moroshko \\
  Electrical Engineering Dept.\\
  The Technion, Israel \\
  \texttt{edward.moroshko@gmail.com} \\
  \And
  Koby Crammer \\
  Electrical Engineering Dept.\\
  The Technion, Israel \\
  \texttt{koby@ee.technion.ac.il} \\
}

\begin{document}

\maketitle

\begin{abstract}
  In {\em extreme classification} problems, learning algorithms are
  required to map instances to labels from an extremely large label set.
  We build on a recent extreme classification framework with logarithmic time and space \cite{Jasinska2016},
  and on a general approach for error correcting output coding (ECOC) with loss-based decoding \cite{Allwein2001_new},
  and introduce a flexible and efficient approach accompanied by theoretical bounds.
  Our framework employs output codes induced by graphs,
  for which we show how to perform efficient loss-based decoding
  to potentially improve accuracy.
  In addition,
  our framework offers a tradeoff between accuracy, model size and prediction time.
  We show how to find the sweet spot of this tradeoff using only the training data.
  Our experimental study demonstrates the validity of our assumptions and claims,
  and shows that our method is competitive with state-of-the-art algorithms.
\end{abstract}

%
%  Macros for Thesis
%%%%%%%%%%%%%%%%%%%%%%%%%%%%%%%%%%%%%%%%%%%%%%%%%%%%%%%%%%%%%%%%%%%%%%%%%%%%%

%\newtheorem{theorem}{Theorem}
%\newtheorem{lemma}[theorem]{Lemma}
% \newtheorem{definition}[theorem]{Definition}
% \newtheorem{claim}[theorem]{Claim}
% \newtheorem{corollary}[theorem]{Corollary}

\newtheorem{theorem}{Theorem}
\newtheorem{lemma}[theorem]{Lemma}
\newtheorem{corollary}[theorem]{Corollary}
\newtheorem{definition}{Definition}
\newtheorem{Remark}{Remark}
\def\proof{\par\penalty-1000\vskip .5 pt\noindent{\bf Proof\/: }}
\def\proofsketch{\par\penalty-1000\vskip .5 pt\noindent{\bf Proof sketch\/: }}
\def\ProofSketch{\par\penalty-1000\vskip .1 pt\noindent{\bf Proof sketch\/: }}
\newcommand{\QED}{\hfill$\;\;\;\rule[0.1mm]{2mm}{2mm}$\\}

\newcommand{\todo}[1]{{~\\\bf TODO: {#1}}~\\}

%%%%%%%%%%%%%%%%%%%%%%%%%%%%%%%%%%%%%%%%%%%%%%%%%%%%%%%%%%%%%%%%%%
% General
%%%%%%%%%%%%%%%%%%%%%%%%%%%%%%%%%%%%%%%%%%%%%%%%%%%%%%%%%%%%%%%%%%
\newfont{\msym}{msbm10}
\newcommand{\reals}{\mathbb{R}}%Re}%{mbox{\msym R}}
\newcommand{\half}{\frac{1}{2}}
\newcommand{\sign}{{\rm sign}}
\newcommand{\paren}[1]{\left({#1}\right)}
\newcommand{\brackets}[1]{\left[{#1}\right]}
\newcommand{\braces}[1]{\left\{{#1}\right\}}
\newcommand{\ceiling}[1]{\left\lceil{#1}\right\rceil}
\newcommand{\abs}[1]{\left\vert{#1}\right\vert}
\newcommand{\tr}{{\rm Tr}}
\newcommand{\pr}[1]{{\rm Pr}\left[{#1}\right]}
\newcommand{\prp}[2]{{\rm Pr}_{#1}\left[{#2}\right]}
\newcommand{\Exp}[1]{{\rm E}\left[{#1}\right]}
\newcommand{\Expp}[2]{{\rm E}_{#1}\left[{#2}\right]}
\newcommand{\eqdef}{\stackrel{\rm def}{=}}
\newcommand{\comdots}{, \ldots ,}
\newcommand{\true}{\texttt{True}}
\newcommand{\false}{\texttt{False}}
\newcommand{\mcal}[1]{{\mathcal{#1}}}
\newcommand{\argmin}[1]{\underset{#1}{\mathrm{argmin}} \:}
\newcommand{\normt}[1]{\left\Vert {#1} \right\Vert^2}
\newcommand{\step}[1]{\left[#1\right]_+}
\newcommand{\1}[1]{[\![{#1}]\!]}
\newcommand{\diag}{{\textrm{diag}}}
\newcommand{\KL}{{\textrm{D}_{\textrm{KL}}}}
\newcommand{\IS}{{\textrm{D}_{\textrm{IS}}}}
\newcommand{\EU}{{\textrm{D}_{\textrm{EU}}}}
\newcommand{\bigO}[1]{\mathcal{O}\left(#1\right)}

%%%%%%%%%%%%%%%%%%%%%%%%%%%%%%%%%%%%%%%%%%%%%%%%%%%%%%%%%%
% Control symbols
%%%%%%%%%%%%%%%%%%%%%%%%%%%%%%%%%%%%%%%%%%%%%%%%%%%%%%%%%%
\newcommand{\leftmarginpar}[1]{\marginpar[#1]{}}
\newcommand{\figline}{\rule{0.50\textwidth}{0.5pt}}
\newcommand{\pseudocodefont}{\normalsize}
\newcommand{\nolineskips}{
\setlength{\parskip}{0pt}
\setlength{\parsep}{0pt}
\setlength{\topsep}{0pt}
\setlength{\partopsep}{0pt}
\setlength{\itemsep}{0pt}}

%%%%%%%%%%%%%%%%%%%%%%%%%%%%%%%%%%%%%%%%%%%%%%%%%%%%%%%%%%%
% Equations and references
%%%%%%%%%%%%%%%%%%%%%%%%%%%%%%%%%%%%%%%%%%%%%%%%%%%%%%%%%%%
\newcommand{\beq}[1]{\begin{equation}\label{#1}}
\newcommand{\eeq}{\end{equation}}
\newcommand{\beqa}{\begin{eqnarray}}
\newcommand{\eeqa}{\end{eqnarray}}
\newcommand{\secref}[1]{Section~\ref{#1}}
\newcommand{\figref}[1]{Figure~\ref{#1}}
\newcommand{\exmref}[1]{Example~\ref{#1}}
\newcommand{\thmref}[1]{Theorem~\ref{#1}}
\newcommand{\sthmref}[1]{Thm.~\ref{#1}}
\newcommand{\defref}[1]{Definition~\ref{#1}}
\newcommand{\remref}[1]{Remark~\ref{#1}}
\newcommand{\chapref}[1]{Chapter~\ref{#1}}
\newcommand{\appref}[1]{Appendix~\ref{#1}}
\newcommand{\lemref}[1]{Lemma~\ref{#1}}
\newcommand{\propref}[1]{Proposition~\ref{#1}}
\newcommand{\claimref}[1]{Claim~\ref{#1}}
\newcommand{\algoref}[1]{Algorithm~\ref{#1}}

\newcommand{\corref}[1]{Corollary~\ref{#1}}
\newcommand{\scorref}[1]{Cor.~\ref{#1}}
\newcommand{\tabref}[1]{Table~\ref{#1}}
\newcommand{\tran}[1]{{#1}^{\top}}
\newcommand{\norm}{\mcal{N}}
\newcommand{\eqsref}[1]{Eqns.~(\ref{#1})}

%%%%%%%%%%%%%%%%%%%%%%%%%%%%%%%%%%%%%%%%%%%%%%%%%%%%%%%%%%%
% bold, up, down
%%%%%%%%%%%%%%%%%%%%%%%%%%%%%%%%%%%%%%%%%%%%%%%%%%%%%%%%%%%
\newcommand{\mb}[1]{{\boldsymbol{#1}}}
\newcommand{\up}[2]{{#1}^{#2}}
\newcommand{\dn}[2]{{#1}_{#2}}
\newcommand{\du}[3]{{#1}_{#2}^{#3}}
\newcommand{\textl}[2]{{$\textrm{#1}_{\textrm{#2}}$}}

%%%%%%%%%%%%%%%%%%%%%%%%%%%%%%%%%%%%%%%%%%%%%%%%%%%%%%%%%%%
% vectors \va
%%%%%%%%%%%%%%%%%%%%%%%%%%%%%%%%%%%%%%%%%%%%%%%%%%%%%%%%%%%
\newcommand{\vx}{\mathbf{x}}
\newcommand{\vxi}[1]{\vx_{#1}}
\newcommand{\vxii}{\vxi{t}}

\newcommand{\yi}[1]{y_{#1}}
\newcommand{\yii}{\yi{t}}
\newcommand{\hyi}[1]{\hat{y}_{#1}}
\newcommand{\hyii}{\hyi{i}}

\newcommand{\vy}{\mb{y}}
\newcommand{\vyi}[1]{\vy_{#1}}
\newcommand{\vyii}{\vyi{i}}

\newcommand{\vn}{\mb{\nu}}
\newcommand{\vni}[1]{\vn_{#1}}
\newcommand{\vnii}{\vni{i}}

\newcommand{\vmu}{\mb{\mu}}
\newcommand{\vmus}{{\vmu^*}}
\newcommand{\vmuts}{{\vmus}^{\top}}
\newcommand{\vmui}[1]{\vmu_{#1}}
\newcommand{\vmuii}{\vmui{i}}

\newcommand{\vmut}{\vmu^{\top}}
\newcommand{\vmuti}[1]{\vmut_{#1}}
\newcommand{\vmutii}{\vmuti{i}}

\newcommand{\vsigma}{\mb \sigma}
\newcommand{\msigma}{\Sigma}
\newcommand{\msigmas}{{\msigma^*}}
\newcommand{\msigmai}[1]{\msigma_{#1}}
\newcommand{\msigmaii}{\msigmai{t}}

\newcommand{\mups}{\Upsilon}
\newcommand{\mupss}{{\mups^*}}
\newcommand{\mupsi}[1]{\mups_{#1}}
\newcommand{\mupsii}{\mupsi{i}}
\newcommand{\upssl}{\upsilon^*_l}

\newcommand{\vu}{\mathbf{u}}
\newcommand{\vut}{\tran{\vu}}
\newcommand{\vui}[1]{\vu_{#1}}
\newcommand{\vuti}[1]{\vut_{#1}}
\newcommand{\hvu}{\hat{\vu}}
\newcommand{\hvut}{\tran{\hvu}}
\newcommand{\hvur}[1]{\hvu_{#1}}
\newcommand{\hvutr}[1]{\hvut_{#1}}
\newcommand{\vw}{\mb{w}}
\newcommand{\vwi}[1]{\vw_{#1}}
\newcommand{\vwii}{\vwi{t}}
\newcommand{\vwti}[1]{\vwt_{#1}}
\newcommand{\vwt}{\tran{\vw}}

\newcommand{\tvw}{\tilde{\mb{w}}}
\newcommand{\tvwi}[1]{\tvw_{#1}}
\newcommand{\tvwii}{\tvwi{t}}

\newcommand{\vv}{\mb{v}}
\newcommand{\vvt}{\tran{\vv}}

\newcommand{\vvi}[1]{\vv_{#1}}
\newcommand{\vvti}[1]{\vvt_{#1}}
\newcommand{\lambdai}[1]{\lambda_{#1}}
\newcommand{\Lambdai}[1]{\Lambda_{#1}}

\newcommand{\vxt}{\tran{\vx}}
\newcommand{\hvx}{\hat{\vx}}
\newcommand{\hvxi}[1]{\hvx_{#1}}
\newcommand{\hvxii}{\hvxi{i}}
\newcommand{\hvxt}{\tran{\hvx}}
\newcommand{\hvxti}[1]{\hvxt_{#1}}
\newcommand{\hvxtii}{\hvxti{i}}
\newcommand{\vxti}[1]{\vxt_{#1}}
\newcommand{\vxtii}{\vxti{i}}

\newcommand{\vb}{\mb{b}}
\newcommand{\vbt}{\tran{\vb}}
\newcommand{\vbi}[1]{\vb_{#1}}

\newcommand{\hvy}{\hat{\vy}}
\newcommand{\hvyi}[1]{\hvy_{#1}}

%%%%%%%%%%%%%%%%%%%%%%%%%%%%%%%%%%%%%%%%%%%%%%%%%%%%%%%%%%%%%%%%%
% Matrices (\mA)
%%%%%%%%%%%%%%%%%%%%%%%%%%%%%%%%%%%%%%%%%%%%%%%%%%%%%%%%%%%%%%%%%

\renewcommand{\mp}{P}
\newcommand{\mpd}{\mp^{(d)}}
\newcommand{\mpt}{\mp^T}
\newcommand{\tmp}{\tilde{\mp}}
\newcommand{\mpi}[1]{\mp_{#1}}
\newcommand{\mpti}[1]{\mpt_{#1}}
\newcommand{\mptii}{\mpti{i}}
\newcommand{\mpii}{\mpi{i}}
\newcommand{\mps}{Q}
\newcommand{\mpsi}[1]{\mps_{#1}}
\newcommand{\mpsii}{\mpsi{i}}
\newcommand{\tmpt}{\tmp^T}
\newcommand{\mz}{Z}
\newcommand{\mv}{V}
\newcommand{\mvi}[1]{\mv_{#1}}
\newcommand{\mvt}{V^T}
\newcommand{\mvti}[1]{\mvt_{#1}}
\newcommand{\mzt}{\mz^T}
\newcommand{\tmz}{\tilde{\mz}}
\newcommand{\tmzt}{\tmz^T}
\newcommand{\mx}{\mathbf{X}}
\newcommand{\ma}{\mathbf{A}}
\newcommand{\mxs}[1]{\mx_{#1}}

\newcommand{\mai}[1]{\ma_{#1}}
\newcommand{\mat}{\tran{\ma}}
\newcommand{\mati}[1]{\mat_{#1}}

\newcommand{\mc}{{C}}
\newcommand{\mci}[1]{\mc_{#1}}
\newcommand{\mcti}[1]{\mct_{#1}}

\newcommand{\md}{{\mathbf{D}}}
\newcommand{\mdi}[1]{\md_{#1}}

\newcommand{\mxi}[1]{\textrm{diag}^2\paren{\vxi{#1}}}
\newcommand{\mxii}{\mxi{i}}

\newcommand{\hmx}{\hat{\mx}}
\newcommand{\hmxi}[1]{\hmx_{#1}}
\newcommand{\hmxii}{\hmxi{i}}
\newcommand{\hmxt}{\hmx^T}
\newcommand{\mxt}{\mx^\top}
\newcommand{\mi}{\mathbf{I}}
\newcommand{\mq}{Q}
\newcommand{\mqt}{\mq^T}
\newcommand{\mlam}{\Lambda}
%\newcommand{\ma}{A}
%\newcommand{\ms}{S}
%\newcommand{\mt}{T}

%%%%%%%%%%%%%%%%%%%%%%%%%%%%%%%%%%%%%%%%%%%%%%%%%%%%%%%%%%%
% mathcal
%%%%%%%%%%%%%%%%%%%%%%%%%%%%%%%%%%%%%%%%%%%%%%%%%%%%%%%%%%%
\renewcommand{\L}{\mcal{L}}
\newcommand{\R}{\mcal{R}}
\newcommand{\X}{\mcal{X}}
\newcommand{\Y}{\mcal{Y}}
\newcommand{\F}{\mcal{F}}
\newcommand{\nur}[1]{\nu_{#1}}
\newcommand{\lambdar}[1]{\lambda_{#1}}
\newcommand{\gammai}[1]{\gamma_{#1}}
\newcommand{\gammaii}{\gammai{i}}
\newcommand{\alphai}[1]{\alpha_{#1}}
\newcommand{\alphaii}{\alphai{i}}
\newcommand{\lossp}[1]{\ell_{#1}}
\newcommand{\eps}{\epsilon}
\newcommand{\epss}{\eps^*}
\newcommand{\lsep}{\lossp{\eps}}
\newcommand{\lseps}{\lossp{\epss}}
\newcommand{\T}{\mcal{T}}

%%%%%%%%%%%%%%%%%%%%%%%%%%%%%%%%%%%%%%%%%%%%%%%%%%%%%%%%%%%
% Notes
%%%%%%%%%%%%%%%%%%%%%%%%%%%%%%%%%%%%%%%%%%%%%%%%%%%%%%%%%%%
\newcommand{\kc}[1]{\begin{center}\fbox{\parbox{3in}{{\textcolor{green}{KC: #1}}}}\end{center}}
\newcommand{\edward}[1]{\begin{center}\fbox{\parbox{3in}{{\textcolor{red}{EM: #1}}}}\end{center}}
\newcommand{\itay}[1]{\begin{center}\fbox{\parbox{3in}{{\textcolor{blue}{IE: #1}}}}\end{center}}
\newcommand{\revised}[1]{#1}

\newcommand{\newstuffa}[2]{#2}
\newcommand{\newstufffroma}[1]{}
\newcommand{\newstufftoa}{}

\newcommand{\newstuff}[2]{#2}
\newcommand{\newstufffrom}[1]{}
\newcommand{\newstuffto}{}
\newcommand{\oldnote}[2]{}

\newcommand{\commentout}[1]{}
\newcommand{\mypar}[1]{\medskip\noindent{\bf #1}}

%%%%%%%%%%%%%%%%%%%%%%%%%%%%%%%%%%%%%%%%%%%%%%%%%%%%%%%%%%%
% other
%%%%%%%%%%%%%%%%%%%%%%%%%%%%%%%%%%%%%%%%%%%%%%%%%%%%%%%%%%%
% inner products
\newcommand{\inner}[2]{\left< {#1} , {#2} \right>}
\newcommand{\kernel}[2]{K\left({#1},{#2} \right)}
\newcommand{\tprr}{\tilde{p}_{rr}}
\newcommand{\hxr}{\hat{x}_{r}}
\newcommand{\projalg}{{PST }}%{\tt Projection }}
\newcommand{\projealg}[1]{$\textrm{PST}_{#1}~$}%{\tt Projection }}
\newcommand{\gradalg}{{GST }}%\tt Gradient }}

\newcounter {mySubCounter}
\newcommand {\twocoleqn}[4]{
  \setcounter {mySubCounter}{0} %
  \let\OldTheEquation \theequation %
  \renewcommand {\theequation }{\OldTheEquation \alph {mySubCounter}}%
  \noindent \hfill%
  \begin{minipage}{.40\textwidth}
\vspace{-0.6cm}
    \begin{equation}\refstepcounter{mySubCounter}
      #1
    \end {equation}
  \end {minipage}
~~~~~~
%\hfill %
  \addtocounter {equation}{ -1}%
  \begin{minipage}{.40\textwidth}
\vspace{-0.6cm}
    \begin{equation}\refstepcounter{mySubCounter}
      #3
    \end{equation}
  \end{minipage}%
  \let\theequation\OldTheEquation
}

\newcommand{\vzero}{\mb{0}}

\newcommand{\smargin}{\mcal{M}}

\newcommand{\ai}[1]{A_{#1}}
\newcommand{\bi}[1]{B_{#1}}
\newcommand{\aii}{\ai{i}}
\newcommand{\bii}{\bi{i}}
\newcommand{\betai}[1]{\beta_{#1}}
\newcommand{\betaii}{\betai{i}}
\newcommand{\mar}{M}
\newcommand{\mari}[1]{\mar_{#1}}
\newcommand{\marii}{\mari{i}}
\newcommand{\nmari}[1]{m_{#1}}
\newcommand{\nmarii}{\nmari{i}}

\newcommand{\erf}{\Phi}

\newcommand{\var}{V}
\newcommand{\vari}[1]{\var_{#1}}
\newcommand{\varii}{\vari{i}}

\newcommand{\varb}{v}
\newcommand{\varbi}[1]{\varb_{#1}}
\newcommand{\varbii}{\varbi{i}}

\newcommand{\vara}{u}
\newcommand{\varai}[1]{\vara_{#1}}
\newcommand{\varaii}{\varai{i}}

\newcommand{\marb}{m}
\newcommand{\marbi}[1]{\marb_{#1}}
\newcommand{\marbii}{\marbi{i}}

\newcommand{\algname}{{AROW}}
\newcommand{\rlsname}{{RLS}}
\newcommand{\mrlsname}{{MRLS}}

\newcommand{\phia}{\psi}
\newcommand{\phib}{\xi}

\newcommand{\amsigmaii}{\tilde{\msigma}_t}
\newcommand{\amsigmai}[1]{\tilde{\msigma}_{#1}}
\newcommand{\avmuii}{\tilde{\vmu}_i}
\newcommand{\avmui}[1]{\tilde{\vmu}_{#1}}
\newcommand{\amarbii}{\tilde{\marb}_i}
\newcommand{\avarbii}{\tilde{\varb}_i}
\newcommand{\avaraii}{\tilde{\vara}_i}
\newcommand{\aalphaii}{\tilde{\alpha}_i}

\newcommand{\svar}{v}
\newcommand{\smar}{m}
\newcommand{\nsmar}{\bar{m}}

\newcommand{\vnu}{\mb{\nu}}
\newcommand{\vnut}{\vnu^\top}
\newcommand{\vz}{\mb{z}}
\newcommand{\vZ}{\mb{Z}}
\newcommand{\fphi}{f_{\phi}}
\newcommand{\gphi}{g_{\phi}}

%%% Local Variables:
%%% mode: latex
%%% TeX-master: "nips2007"
%%% End:

\newcommand{\vtmui}[1]{\tilde{\vmu}_{#1}}
\newcommand{\vtmuii}{\vtmui{i}}

\newcommand{\zetai}[1]{\zeta_{#1}}
\newcommand{\zetaii}{\zetai{i}}

%%%%%%

\newcommand{\vstate}{\bf{s}}
\newcommand{\vstatet}[1]{\vstate_{#1}}
\newcommand{\vstatett}{\vstatet{t}}

\newcommand{\mtran}{\bf{\Phi}}
\newcommand{\mtrant}[1]{\mtran_{#1}}
\newcommand{\mtrantt}{\mtrant{t}}

\newcommand{\vstatenoise}{\bf{\eta}}
\newcommand{\vstatenoiset}[1]{\vstatenoise_{#1}}
\newcommand{\vstatenoisett}{\vstatenoiset{t}}

\newcommand{\vobser}{\bf{o}}
\newcommand{\vobsert}[1]{\vobser_{#1}}
\newcommand{\vobsertt}{\vobsert{t}}

\newcommand{\mobser}{\bf{H}}
\newcommand{\mobsert}[1]{\mobser_{#1}}
\newcommand{\mobsertt}{\mobsert{t}}

\newcommand{\vobsernoise}{\bf{\nu}}
\newcommand{\vobsernoiset}[1]{\vobsernoise_{#1}}
\newcommand{\vobsernoisett}{\vobsernoiset{t}}

\newcommand{\mstatenoisecov}{\bf{Q}}
\newcommand{\mstatenoisecovt}[1]{\mstatenoisecov_{#1}}
\newcommand{\mstatenoisecovtt}{\mstatenoisecovt{t}}

\newcommand{\mobsernoisecov}{\bf{R}}
\newcommand{\mobsernoisecovt}[1]{\mobsernoisecov_{#1}}
\newcommand{\mobsernoisecovtt}{\mobsernoisecovt{t}}

\newcommand{\vestate}{\bf{\hat{s}}}
\newcommand{\vestatet}[1]{\vestate_{#1}}
\newcommand{\vestatett}{\vestatet{t}}
\newcommand{\vestatept}[1]{\vestatet{#1}^+}
\newcommand{\vestatent}[1]{\vestatet{#1}^-}

\newcommand{\mcovar}{\bf{P}}
\newcommand{\mcovart}[1]{\mcovar_{#1}}
\newcommand{\mcovarpt}[1]{\mcovart{#1}^+}
\newcommand{\mcovarnt}[1]{\mcovart{#1}^-}

\newcommand{\mkalmangain}{\bf{K}}
\newcommand{\mkalmangaint}[1]{\mkalmangain_{#1}}

\newcommand{\vkalmangain}{\bf{\kappa}}
\newcommand{\vkalmangaint}[1]{\vkalmangain_{#1}}

\newcommand{\obsernoise}{{\nu}}
\newcommand{\obsernoiset}[1]{\obsernoise_{#1}}
\newcommand{\obsernoisett}{\obsernoiset{t}}

\newcommand{\obsernoisecov}{r}
\newcommand{\obsernoisecovt}[1]{\obsernoisecov_{#1}}
\newcommand{\obsernoisecovtt}{\obsernoisecov}%t{t}}

\newcommand{\obsnscv}{s}
\newcommand{\obsnscvt}[1]{\obsnscv_{#1}}
\newcommand{\obsnscvtt}{\obsnscvt{t}}

\newcommand{\Psit}[1]{\Psi_{#1}}
\newcommand{\Psitt}{\Psit{t}}

\newcommand{\Omegat}[1]{\Omega_{#1}}
\newcommand{\Omegatt}{\Omegat{t}}

\newcommand{\ellt}[1]{\ell_{#1}}
\newcommand{\gllt}[1]{g_{#1}}

\newcommand{\chit}[1]{\chi_{#1}}

\newcommand{\ms}{\mathcal{M}}
\newcommand{\us}{\mathcal{U}}
\newcommand{\as}{\mathcal{A}}

\newcommand{\mn}{M}
\newcommand{\un}{U}

\newcommand{\seti}[1]{S_{#1}}

\newcommand{\obj}{\mcal{C}}

\newcommand{\dta}[3]{d_{#3}\paren{#1,#2}}

\newcommand{\coa}{a}
\newcommand{\coc}{c}
\newcommand{\cob}{b}
\newcommand{\cor}{r}
\newcommand{\conu}{\nu}

\newcommand{\coat}[1]{\coa_{#1}}
\newcommand{\coct}[1]{\coc_{#1}}
\newcommand{\cobt}[1]{\cob_{#1}}
\newcommand{\cort}[1]{\cor_{#1}}
\newcommand{\conut}[1]{\conu_{#1}}

\newcommand{\coatt}{\coat{t}}
\newcommand{\coctt}{\coct{t}}
\newcommand{\cobtt}{\cobt{t}}
\newcommand{\cortt}{\cort{t}}
\newcommand{\conutt}{\conut{t}}

\newcommand{\rb}{R_B}
\newcommand{\proj}{\textrm{proj}}

\section{Introduction}

Multiclass classification is the task of assigning instances with a
category or class from a finite set.
Its numerous applications range from finding a topic of a news item,
via classifying objects in images,
via spoken words detection, to predicting the next word in a sentence.
Our ability to solve multiclass problems with larger and
larger sets improves with computation power.
Recent research focuses on \textit{extreme classification}
where the number of possible classes $K$ is extremely large.

In such cases, previously developed methods, such as {\em One-vs-One (OVO)}~\cite{Furnkranz:2002:RRC:944790.944803},
\textit{One-vs-Rest (OVR)}~\cite{CortesVa95} and multiclass SVMs ~\cite{WestonWa99,BredensteinerBe99,CrammerSi01a,Mesterharm99},
 that scale linearly in the
number of classes $K$, are not feasible.
These methods maintain too large models, that cannot be stored easily.
Moreover, their training and inference times are at least linear in $K$,
and thus do not scale for extreme classification problems.

Recently, Jasinska and Karampatziakis \cite{Jasinska2016} proposed a Log-Time Log-Space
(LTLS) approach, representing classes as paths on graphs.
LTLS is very efficient,
but has a limited representation,
resulting in an inferior accuracy compared to other methods.
More than a decade earlier,
Allwein et al. \cite{Allwein2001_new} presented a unified view of error
correcting output coding (ECOC) for classification, as well as
the loss-based decoding framework.
They showed its superiority over
Hamming decoding, both theoretically and empirically.

In this work we build on these two works and introduce an
\emph{efficient} (i.e. $O(\log K)$ time and space) loss-based learning and
decoding algorithm for \emph{any} loss function of the binary
learners' margin.
We show that LTLS can be seen as a special case of ECOC.
We also make a more general connection between
loss-based decoding and graph-based representations and inference. % i.e. LTLS
Based on the theoretical framework and analysis derived by \cite{Allwein2001_new}
for loss-based decoding,
we gain insights on how to
improve on the specific graphs proposed in LTLS by using more general
trellis graphs -- which we name \textit{Wide-LTLS (W-LTLS)}.
Our method profits
from the best of both worlds: better accuracy as in loss-based decoding, and
the logarithmic time and space of LTLS.
Our empirical study suggests that by employing coding matrices induced by different trellis graphs,
our method allows tradeoffs between accuracy, model size, and inference time,
especially appealing for extreme classification.
%Additionally, our method allows parallelized training, as opposed to LTLS.

%The rest of the paper is organized as follows: In
%\secref{sec:problem_setting} we formally define the problem. We then briefly review ECOC in \secref{sec:ecoc} and LTLS in \secref{sec:ltls}.
%Then, in \secref{sec:loss_based_decoding} we present a generalization of LTLS
%for efficient loss-based decoding using any loss function of the
%margin, imposing the theoretical framework of \cite{Allwein2001_new} on
%LTLS.
%We then use this generalization to tackle the classification
%error of LTLS by decreasing one factor in the training error bound for
%output coding schemes.
%Our algorithm is described in
%\secref{sec:Wltls}, and the results of our experimental study are
%reported in \secref{sec:experiments}.  After discussing differences
%from related work in \secref{sec:related_work}, we conclude our work
%in \secref{sec:conclusions} and discuss some promising future work
%directions.
%

\section{Problem setting}
\label{sec:problem_setting}

We consider \emph{multiclass classification} with $K$ classes, where $K$ is very large. 
Given a training set of $m$ examples $\left(x_i,y_i\right)$ 
for $x_i\in \mathcal{X}\subseteq\mathbb{R}^d$ and $y_i\in \mathcal{Y}=\{1,...,K\}$ 
our goal is to learn a mapping from $\mathcal{X}$ to $\mathcal{Y}$. 
We focus on the 0/1 loss and evaluate the performance 
of the learned mapping by measuring its accuracy on a test set -- 
i.e. the fraction of instances with a correct prediction. 
Formally, the accuracy of a mapping $h:\mathcal{X}\rightarrow\mathcal{Y}$ 
on a set of $n$ pairs, $\left\{ \left(x_i, y_i\right) \right\}_{i=1}^n$, 
is defined as $\frac{1}{n} \sum_{i=1}^n \mathbf{1}_{h(x_i)=y_i}$,
where $\mathbf{1}_z$ equals $1$ if the predicate $z$ is true, and $0$ otherwise.

\section{Error Correcting Output Coding (ECOC)}
\label{sec:ecoc}

Dietterich and Bakiri \cite{Dietterich1995} employed ideas from coding theory \cite{Lin_2004} to create \textit{Error Correcting Output Coding (ECOC)} -- a reduction from a multiclass classification problem to multiple binary classification subproblems.
In this scheme, each class is assigned with a (distinct) binary codeword of $\ell$ bits (with values in $\{-1,+1\}$).
The $K$ codewords create a matrix $M\in\left\{-1,+1\right\}^{K\times{\ell}}$ whose rows are the codewords
and whose columns induce $\ell$ partitions of the classes into two subsets.
Each of these partitions induces a binary classification subproblem.
We denote by $M_k$ the $k$th row of the matrix, and by $M_{k,j}$ its $(k,j)$ entry.
In the j$th$ partition, class $k$ is assigned with the binary label $M_{k,j}$.

ECOC introduces redundancy in order to acquire
error-correcting capabilities such as a minimum Hamming distance between codewords.
The Hamming distance between two codewords $M_a, M_b$ is defined as $\rho(a,b)\triangleq\sum_{j=1}^{\ell}\frac{1-M_{a,j}M_{b,j}}{2}$,
and the minimum Hamming distance of $M$ is $\rho=\min_{a\neq{b}}\rho(a,b)$.
A high minimum distance of the coding matrix
potentially allows overcoming binary classification errors during inference time.

At training time, this scheme generates $\ell$ binary classification training sets of the form $\left\{x_i, M_{y_i, j}\right\}^m_{i=1}$ for $j=1,\dots,\ell$,
and executes some binary classification learning algorithm that returns $\ell$ classifiers,
each trained on one of these sets.
We assume these classifiers are margin-based, that is,
each classifier is a real-valued function,
$f_j:\mathcal{X}\to\mathbb{R}$,
whose binary prediction for an input $x$ is $\sign\left(f_j\left(x\right)\right)$.
The binary classification learning algorithm defines a margin-based loss $L:\reals\rightarrow\reals_+$, %from the reals to the non-negatives,
and minimizes the average loss over the induced set.
Formally, $f_j = \arg\min_{f\in\mathcal{F}} \frac{1}{m}\sum_{i=1}^m L\left( M_{y_i,j} f\left(x_i\right) \right)$,
where $\mathcal{F}$ is a class of functions,
such as the class of bounded linear functions.
Few well known loss functions are the hinge loss $L(z) \triangleq \max\left(0,1-z\right)$, used by SVM,
its square, the log loss $L\left(z\right)\triangleq\log\left(1+e^{-z}\right)$ used in logistic regression,
and the exponential loss $L\left(z\right)\triangleq e^{-z}$ used in AdaBoost \cite{DBLP:conf/birthday/Schapire13}.

 % binary margin-based learners are employed to train $\ell$ classifiers $\left\{f_j:\mathcal{X}\to\mathbb{R}\right\}_{j=1}^{\ell}$. Meaning, for each $j=1,\dots,\ell$ a function $f_j$ is trained to distinguish between all samples $\left(x_i,y_i\right)$ that hold $M_{{y_i}j}=1$ and those that hold $M_{{y_i}j}=-1$.

Once these classifiers are trained, a straightforward inference is performed.
Given an input $x$, the algorithm first applies the
$\ell$ functions on $x$ and computes a $\{\pm 1\}$-vector of size $\ell$,
that is $\left(\sign\left(f_1\left(x\right)\right) \dots \sign\left(f_\ell\left(x\right)\right)\right)$.
%At inference time, upon receiving an input $x_i$, the code vector $\left\{f_{j}(x_i)\right\}_{j=1}^{\ell}$ is computed.
Then, the class $k$ which is assigned to the codeword closest
in Hamming distance to this vector is returned.
This inference scheme is often called {\em Hamming decoding}.
%The decoder then finds the codeword closest to this vector according to some distance measure and predicts the class corresponding to this codeword.

The Hamming decoding uses only the binary prediction of the binary learners,
ignoring the confidence each learner has in its prediction per input.
Allwein et al. \cite{Allwein2001_new} showed that this margin or confidence
holds valuable information for predicting a class $y\in\mathcal{Y}$,
and proposed the {\em loss-based decoding} framework for ECOC\footnote{
Another contribution of their work, less relevant to our work, is a unifying approach for multiclass classification tasks. %was proposed by \citet{Allwein2001_new},
They showed that many popular approaches are unified into a framework of sparse (ternary) coding schemes with a coding matrix $M\in\left\{-1,0,1\right\}^{K\times{\ell}}$.
For example, One-vs-Rest (OVR) could be thought of
as $K\times{K}$ matrix whose diagonal elements are 1, and the rest are -1.}.
In loss-based decoding, the margin is incorporated via the loss function $L(z)$.
Specifically, the class predicted is the one minimizing the total loss
\begin{align}
\label{loss_based}
k^*=\arg\min_{k} \sum_{j=1}^{\ell} L\left(M_{k,j}f_j\left(x\right)\right)~.
\end{align}
%For convex loss functions, \cite{Allwein2001_new} developed error bounds and showed theoretically and practically that loss-based decoding outperforms Hamming-based decoding.

They \cite{Allwein2001_new} also developed error bounds and showed theoretically and empirically that
loss-based decoding outperforms Hamming decoding.

One drawback of their method is that given a loss function $L$,
loss-based decoding requires an exhaustive evaluation
of the total loss for each codeword $M_k$ (each row of the coding matrix).
This implies a decoding time at least linear in $K$,
making it intractable for extreme classification.
We address this problem below.

% \citet{Allwein2001_new} also derived bounds

 %These works and others like \cite{Escalera2008} employed different decoding schemes, including \textit{hard decoding} schemes (e.g. the Hamming distance) that ignore the outputted margin or confidence of the binary learners, and \textit{soft decoding} methods (e.g. loss-based decoding) which exploit these confidence measures to allow a more informed decoding.

%In this work we employ the error correcting output codes (ECOC) framework \cite{Dietterich1995}, which reduces the multiclass classification problem into a set of $\ell$ binary subproblems. 
\section{LTLS}
\label{sec:ltls}

A recent extreme classification approach,
proposed by Jasinska and Karampatziakis \cite{Jasinska2016},
performs training and inference in time and space logarithmic in $K$,
by embedding the $K$ classes
into $K$ paths of a directed-acyclic trellis graph $T$,
built compactly with $\ell=\bigO{\log{K}}$ edges.
%We denote the sets of vertices and edges by $V$ and $E$ (respectively).
We denote the set of vertices $V$ and set of edges $E$.
A multiclass model is defined using $\ell$ functions from the feature space to the reals,
$w_{j}\left(x\right)$, one function per edge in $E=\{e_j\}_{j=1}^\ell$.
Given an input, the algorithm assigns weights to the edges,
and computes the \textit{heaviest path}
using the Viterbi \cite{Viterbi67} algorithm in $\bigO{\left|E\right|}=\bigO{\log{K}}$ time.
It then outputs the class (from $\mathcal{Y}$) assigned to the heaviest path.

Jasinska and Karampatziakis \cite{Jasinska2016} proposed to train the model in an online manner.
The algorithm maintains $\ell$ functions $f_j(x)$ and works in rounds.
In each training round a specific input-output pair $\left(x_i, y_i\right)$ is considered,
the algorithm performs inference using the $\ell$ functions to predict a class $\hat{y}_i$,
and the functions $f_j(x)$ are modified to improve the overall prediction for $x_i$
according to $y_i,\hat{y}_i$.
 The inference performed during train and test times,
 includes using the obtained functions $f_j(x)$
 to compute the weights $w_j(x)$ of each input,
 by simply setting $w_j(x)=f_j(x)$.
 Specifically, they used margin-based learning algorithms,
 where $f_j(x)$ is the margin of a binary prediction.

% The different edges are binary learners trained as an online structured prediction problem. The weight of each edge $e_{j}$ is a learnt function of the input, i.e. $w_{j}: \mathcal{X}\to\mathbb{R}$. During inference, each edge $e_j$ is assigned with a weight $w_{j}\left(x_{i}\right)$ which equals to the margin $f_{j}\left(x_{i}\right)$ of its corresponding binary learner,

Our first contribution is the observation that the LTLS approach can be thought of as an ECOC scheme,
in which the codewords (rows) represent paths in the trellis graph,
and the columns correspond to edges on the graph.
\figref{fig:path-to-code} illustrates how a codeword corresponds to a path on the graph.

\begin{figure}
    \centering
    \begin{minipage}[t]{.48\textwidth}
      %\centering
      \includegraphics[width=\linewidth]{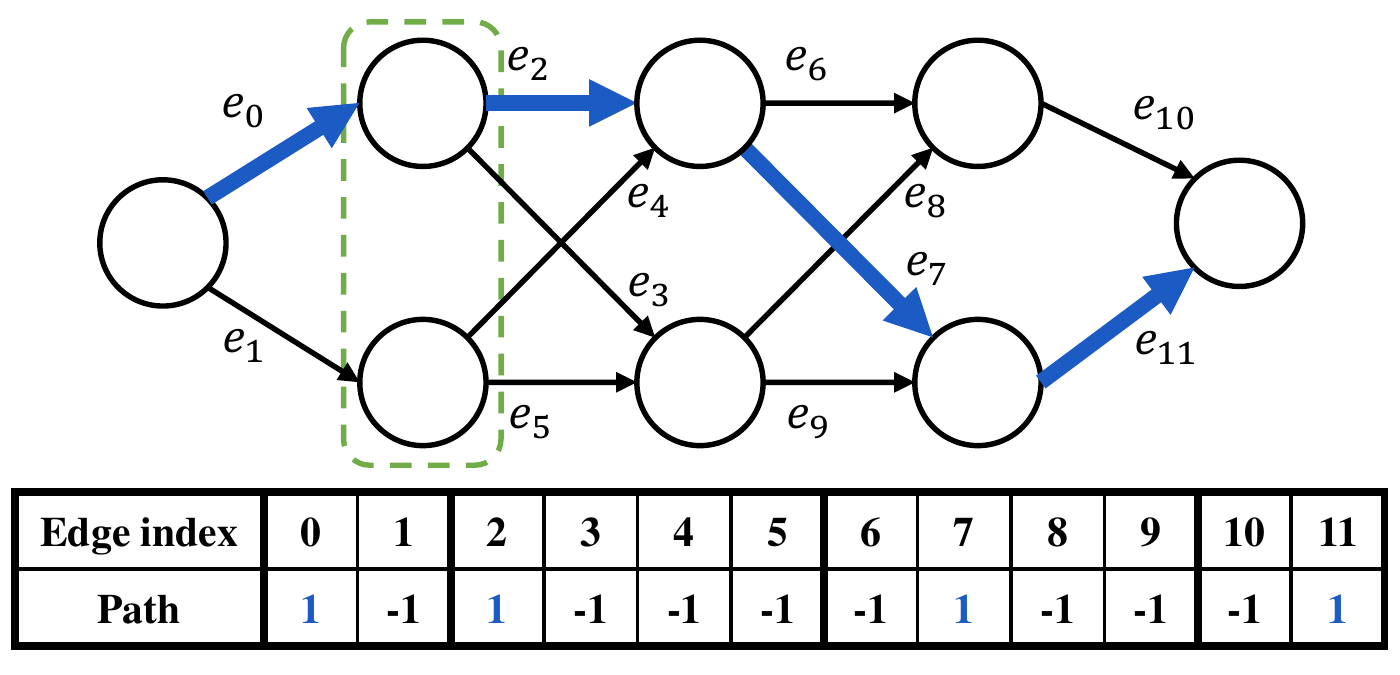}
      \vskip -0.1in
      \caption{Path codeword representation.
      An entry containing 1 means that the corresponding edge is a part of the illustrated bold \textcolor[rgb]{0.00,0.00,0.93}{blue} path.
      The green dashed rectangle shows a vertical slice.}
      \label{fig:path-to-code}
    \end{minipage}
    \hspace{.25cm}
    \begin{minipage}[t]{.48\textwidth}
      %\centering
      \includegraphics[width=\linewidth]{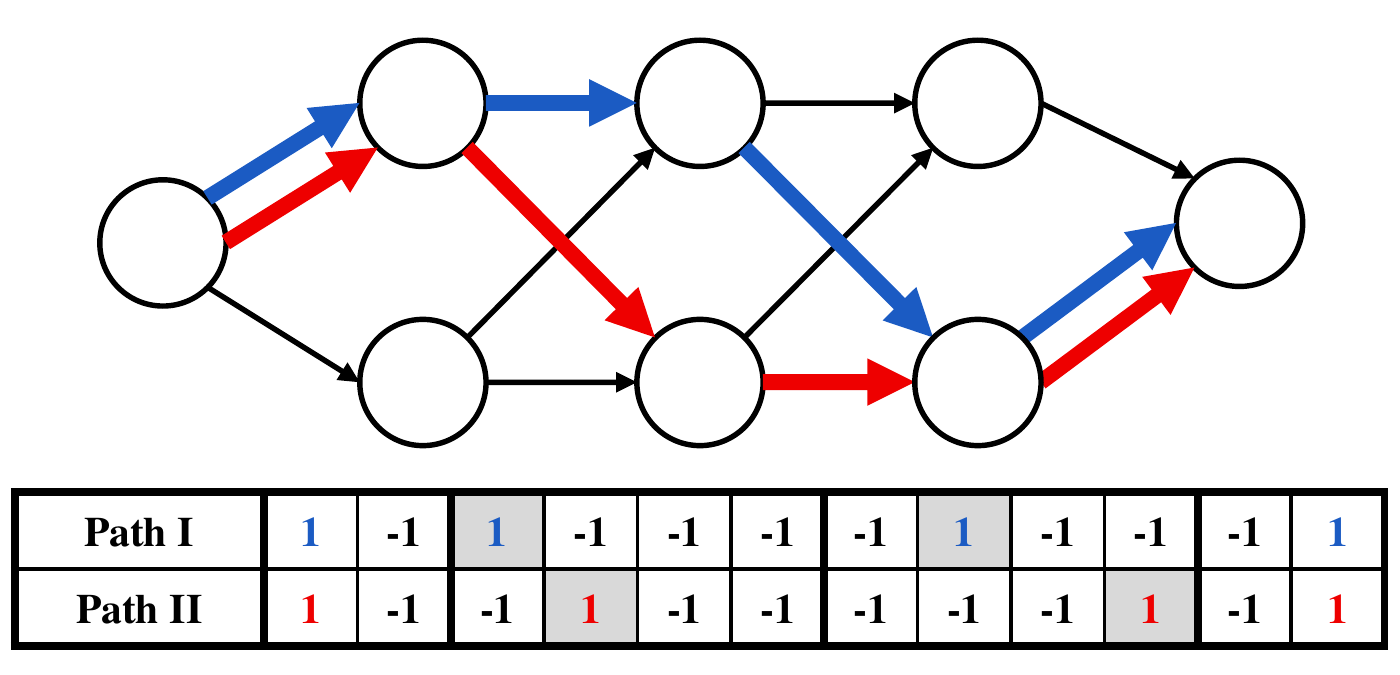}
      \vskip -0.1in
      \caption{Two closest paths. Predicting Path II (\textcolor[rgb]{0.89,0.00,0.00}{red}) instead of I (\textcolor[rgb]{0.00,0.00,0.93}{blue}),
      will result in a prediction error.
      The Hamming distance between the corresponding codewords is 4.
      The highlighted entries correspond to the 4 disagreement edges.}
      \label{fig:path-decode-mistake}
    \end{minipage}
\end{figure}

%\begin{figure}[t]
%\vskip 0.2in
%\begin{center}
%\centerline{\includegraphics[width=4in]{path-to-code.pdf}}
%\caption{Path codeword representation.
%An entry containing 1 means that the corresponding edge is a part of the illustrated bold blue path.
%The green dashed rectangle shows a vertical slice.}
%\label{fig:path-to-code}
%\end{center}
%\vskip -0.2in
%\end{figure}

It might seem like this approach can represent only numbers of classes $K$ which are powers of $2$.
However, in \appref{subsec:graphConstruction} we show how to create trellis graphs with exactly $K$ paths,
for any $K\in\mathbb{N}$.

\subsection{Path assignment}
\label{pathAssignment}
LTLS requires a bijective mapping between paths to classes and vice versa.
It was proposed in \cite{Jasinska2016} to employ a greedy assignment policy suitable for online learning,
where during training, a sample whose class is yet unassigned with a path,
is assigned with the heaviest unassigned path.
One could also consider a naive random assignment between paths and classes.
%Both these policies require at training time an encoding function from class to path,
%and a decoding function from path to class. At test time, they require only a decoding function.
%The mapping holds both time and space complexity implications,
%as discussed in \appref{supp:complexity}.

\subsection{Limitations}
\label{sec:limitations}

The elegant LTLS construction suffers from two limitations:

\begin{enumerate}[leftmargin=*]

    \item \textbf{Difficult induced binary subproblems}:
      The induced binary subproblems are hard, especially when learned with \emph{linear} classifiers.
      Each path uses one of four edges between every two adjacent vertical slices.
      %      Therefore, these edges induce \textit{$\frac{1}{4}K$-vs-$\frac{3}{4}K$} subproblems.
      Therefore, \revised{each edge is used by $\frac{1}{4}$ of the classes,
      inducing a} \textit{$\frac{1}{4}K$-vs-$\frac{3}{4}K$} subproblem.
      Similarly, the edges connected to the source or sink induce
      \textit{$\frac{1}{2}K$-vs-$\frac{1}{2}K$} subproblems.
      In both cases classes are split into two groups,
      almost arbitrarily, with no clear semantic interpretation for that partition.
      For comparison, in 1-vs-Rest (OVR) the induced subproblems are considered much simpler as they require classifying only one class vs the rest\footnote{
        \label{foot:OVR}
        A similar observation is given in Section 6 of Allwein et al. \cite{Allwein2001_new} regarding OVR.
        } (meaning they are much less \emph{balanced}).

    \item \textbf{Low minimum distance}:
        %In the LTLS trellis architecture, every path has another path close to it
        %up to 2 edge replacements,
        \revised{In the LTLS trellis architecture, every path has another (closest) path
        within 2 edge deletions and 2 edge insertions} (see \figref{fig:path-decode-mistake}).
        Thus, the \textit{minimum Hamming distance}
        in the underlying coding matrix is restrictively small: $\rho=4$,
        which might imply a poor error correcting capability.
        The OVR coding matrix also suffers from a small minimum distance ($\rho=2$),
        but as we explained,
        the induced subproblems are very simple,
        allowing a higher classification accuracy in many cases.

\end{enumerate}

We focus on improving the multiclass accuracy by tackling the first limitation,
namely making the underlying binary subproblems easier.
Addressing the second limitation is deferred to future work.

%\begin{figure}[t]
%\vskip 0.2in
%\begin{center}
%\centerline{\includegraphics[width=4in]{path-decode-mistake.pdf}}
%\caption{Two closest paths. Predicting Path II (red) instead of Path I (blue), will result in a prediction error. The Hamming distance between the corresponding codewords is 4. The highlighted entries correspond to the disagreement edges. }
%\label{fig:path-decode-mistake}
%\end{center}
%\vskip -0.2in
%\end{figure}

\section{Efficient loss-based decoding}
\label{sec:loss_based_decoding}
We now introduce another contribution --
a new algorithm performing efficient loss-based decoding (inference)
for \emph{any} loss function by exploiting the structure of trellis graphs.
Similarly to \cite{Jasinska2016}, our decoding algorithm performs inference in two steps.
First, it assigns (per input $x$ to be classified) weights
$\left\{w_{j}\left(x\right)\right\}_{j=1}^\ell$
to the edges $\left\{e_j\right\}_{j=1}^\ell$ of the trellis graph.
Second, it finds the \emph{shortest} path (instead of the heaviest) $P_{k^{*}}$
by an efficient dynamic programming (Viterbi) algorithm and predicts the class $k^{*}$.
Unlike \cite{Jasinska2016},
our strategy for assigning edge weights ensures that for any class $k$,
the weight of the path assigned to this class,
$w\left(P_k\right)\triangleq\sum_{j:e_j\in P_k}w_{j}\left(x\right)$,
equals the total loss $\sum_{j=1}^{\ell} L\left(M_{k,j}f_j\left(x\right)\right)$
for the classified input $x$.
Therefore, finding the shortest path on the graph is equivalent to minimizing the total loss,
which is the aim in loss-based decoding.
In other words, we design a new weighting scheme
that links loss-based decoding
% of ECOC and finding
to the shortest path in a graph.

We now describe our algorithm in more detail
for the case when the number of classes $K$ is a power of 2
(see \appref{subsec:arbitraryLBDGeneralization} for extension to arbitrary $K$).
Consider a directed edge $e_j\in{E}$
and denote by $\left(u_j,v_j\right)$ the two vertices it connects.
Denote by $S\left(e_j\right)$ the set of edges outgoing from the same vertical slice as $e_j$.
Formally, $S\left(e_j\right)=\left\{\left(u,u'\right):
    \delta\left(u\right)=\delta\left(u_j\right)\right\}$,
where $\delta\left(v\right)$ is the shortest distance
from the source vertex to $v$ (in terms of number of edges).
 For example, in \figref{fig:path-to-code},
$S\left(e_0\right)=S\left(e_1\right)=\left\{e_0,e_1\right\}$, $S\left(e_2\right)=S\left(e_3\right)=S\left(e_4\right)=S\left(e_5\right)=\left\{e_2,e_3,e_4,e_5\right\}$.
Given a loss function $L\left(z\right)$ and an input instance $x$, we set the weight $w_j$ for edge $e_j$ as following,
\begin{align}
  w_{j}\left(x\right) = L\left(1\times f_j(x)\right) + \sum_{j{'}:e_{j'}\in S\left(e_j\right)\backslash \left\{e_j\right\}}  L\left((-1)\times f_{j'}(x)\right) ~.
\label{w_i}
\end{align}
For example, in \figref{fig:path-to-code} we have,
\begin{align*}
 w_0\left(x\right) = &~L\left(1\times f_0(x)\right) + L\left(\left(-1\right)\times f_1(x)\right) \\
 w_2\left(x\right) = &~L\left(1\times f_2(x)\right) + L\left(\left(-1\right)\times f_3(x)\right)
 +L\left(\left(-1\right)\times f_4(x)\right) +L\left((-1)\times f_5(x)\right) ~.
\end{align*}

The next theorem states that for our choice of weights,
finding the shortest path in the weighted graph
is equivalent to loss-based decoding.
%there is an equivalence between loss-based decoding and
%finding the shortest path in the induced weighted graph.
Thus, algorithmically we can enjoy fast decoding (i.e. inference),
and statistically we can enjoy better performance by using loss-based decoding.

\begin{theorem}
Let $L\left(z\right)$ be any loss function of the margin.
Let $T$ be a trellis graph with an underlying coding matrix $M$.
Assume that for any $x\in \mathcal{X}$ the edge weights are calculated as in Eq. \eqref{w_i}.
Then, the weight of any path $P_k$ equals to the loss suffered by predicting its corresponding class $k$,
i.e. $w(P_k)=\sum_{j=1}^{\ell} L\left(M_{k,j}f_j(x)\right)$.
\label{thm:loss_based_decoding}
\end{theorem}
The proof appears in \appref{supp:proof_thm_loss_based_decoding}.
%As noted, the extension for the weight reduction and correctness for arbitrary $K$ appears in \appref{subsec:arbitraryLBDGeneralization}.
In the next lemma we claim that LTLS decoding is a special case of loss-based decoding with the squared loss function. See \appref{supp:proof_lemma_ltls_decoding} for proof.
\begin{lemma}
Denote the squared loss function by $L_{sq}(z)\triangleq\left(1-z\right)^2$.
Given a trellis graph represented using a coding matrix $M\in\{-1,+1\}^{K\times \ell}$,
and $\ell$ functions $f_j\left(x\right)$, for $j = 1 \dots \ell$,
the decoding method of LTLS (mentioned in \secref{sec:ltls})
is a special case of loss-based decoding with the squared loss,
that is
$\arg\max_{k}w\left(P_{k}\right) = \arg\min_{k}\left\{ \sum_{j}L_{sq}\left(M_{k,j}f_j\left(x\right)\right) \right\}~.$
\label{lemma:ltls_decoding}
\end{lemma}
%Next we show that the decoding method of LTLS (mentioned in \secref{sec:ltls}) is a special case of loss-based decoding for the square loss function $L_{square}(z)=(1-z)^2$.
%
%%
%\edward{do we want to say that according to this, the LTLS can be interpreted as decoding with regression loss, and thus may not be good for classification (we do not have empirical evidence for this), whereas we can do loss-based decoding with any (classification) loss ? }
%
We next build on the framework of \cite{Allwein2001_new} to design graphs with a better multiclass accuracy. 
\section{Wide-LTLS (W-LTLS)}
\label{sec:Wltls}
Allwein et al. \cite{Allwein2001_new} derived error bounds for loss-based decoding with any convex loss function $L$. They showed that the training multiclass error with loss-based decoding is upper bounded by:
\begin{align}
\label{bound}
\frac{\ell\times\varepsilon}{\rho\times{L(0)}}
\end{align}
 where $\rho$ is the minimum Hamming distance of the code and
 \begin{align}
\label{avgBinaryLoss}
\varepsilon=\frac{1}{m\ell}\sum_{i=1}^{m}{\sum_{j=1}^{\ell}{L\left(M_{y_{i},j}f_{j}{\left(x_{i}\right)}\right)}}
\end{align}
 is the average binary loss on the training set of the learned functions $\{f_{j}\}_{j=1}^\ell$ with respect to a coding matrix $M$ and a loss $L$.
 One approach to reduce the bound, and thus hopefully also the multiclass training error (and under some conditions also the test error) is to reduce the \textit{total error} of the binary problems $\ell\times\varepsilon$. We now show how to achieve this by generalizing the LTLS framework to a more flexible architecture which we call W-LTLS
 \footnote{Code is available online at https://github.com/ievron/wltls/}.

%Motivated by the error bound of \cite{Allwein2001_new}, we propose new graph architectures that generalize the LTLS of \cite{Jasinska2016} and are able to induce less balanced binary subproblems. %In order to make learning easier we propose to make the induced label space partitions more imbalanced.
%In order to keep inference fast, we focus on graphs with $\bigO{\log{K}}$ edges.

\begin{figure}[t]
%\vskip 0.2in
\begin{center}
\centerline{\includegraphics[width=.9\linewidth]{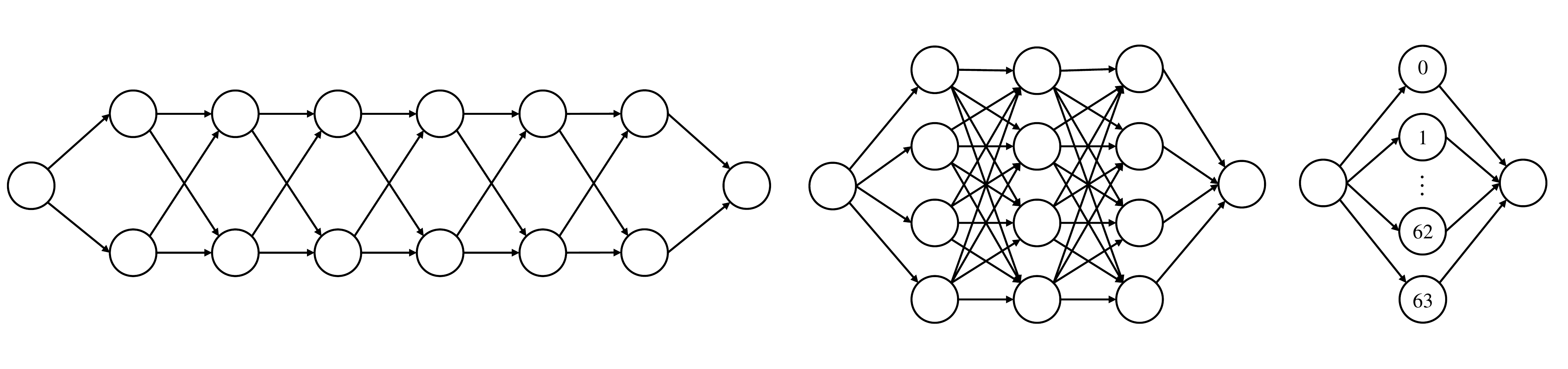}}
\vskip -0.15in
\caption{Different
%trellis
graphs for $K=64$ classes. From left to right:
the LTLS graph with a slice width of $b=2$,
W-LTLS with $b=4$,
and the widest W-LTLS graph with $b=64$, corresponding to OVR.}
\label{fig:slice-width}
\end{center}
\vskip -0.2in
\end{figure}

Motivated by the error bound of \cite{Allwein2001_new},
we propose a generalization of the LTLS model.
By increasing the \emph{slice width} of the trellis graph,
and consequently increasing the number of edges between adjacent vertical slices,
the induced subproblems become less balanced and potentially easier to learn (see \remref{foot:OVR}).
For simplicity we choose a fixed slice width $b\in\left\{2,\dots,K\right\}$ for the entire graph
(e.g. see \figref{fig:slice-width}).
In such a graph, most of the induced subproblems are \textit{$\frac{1}{b^2}K$-vs-rest}
(corresponding to edges between adjacent slices)
and some are \textit{$\frac{1}{b}K$-vs-rest} (the ones connected to the source or to the sink).
As $b$ increases, the graph representation becomes less compact and requires more edges,
\revised{i.e. $\ell$ increases}.
However, the induced subproblems potentially become easier, improving the multiclass accuracy.
This suggests that our model allows an \textit{accuracy vs model size tradeoff}.

In the special case where $b=K$ we get the widest graph containing $2K$ edges
(see \figref{fig:slice-width}).
All the subproblems are now \textit{$1$-vs-rest}:
the $k$th path from the source to the sink contains two edges
(one from the source and one to the sink)
which are not a part of any other path.
Thus, the corresponding two columns in the underlying coding matrix are identical --
having $1$ at their $k$th entry and $\left(-1\right)$ at the rest.
This implies that the distinct columns of the matrix
could be rearranged as the diagonal coding matrix corresponding to OVR,
 making our model when $b=K$ an implementation of OVR.

In \secref{sec:experiments} we show empirically that W-LTLS improves the multiclass accuracy of LTLS.
\revised{In \appref{supp:avgBinaryLoss} we show that the binary subproblems indeed become easier,
i.e. we observe a decrease in the average binary loss $\varepsilon$,
lowering the bound in \eqref{bound}}.
Note that the denominator $\rho\times L\left(0\right)$ is left untouched --
the minimum distance of the coding matrices corresponding to different architectures of W-LTLS is still 4,
like in the original LTLS model (see \secref{sec:limitations}).

\subsection{Time and space complexity analysis}
\label{complexityAnalysis}

W-LTLS requires training and storing a binary learner for every edge.
For most linear classifiers (with $d$ parameters each)
we get\footnote{
Clearly, when $b\approx \sqrt{K}$ our method cannot be regarded as sublinear in $K$ anymore.
However, our empirical study shows that high accuracy
can be achieved using much smaller values of $b$.
} a total \emph{model size complexity}
and an \emph{inference time complexity}
of $\bigO{d\left|E\right|}=\bigO{d\frac{b^2}{\log{b}}\log{K}}$
(see \appref{supp:complexity} for further details).
Moreover, many extreme classification datasets are sparse --
the average number of non-zero features in a sample is $d_e\ll d$.
The inference time complexity thus decreases to $\bigO{d_e\frac{b^2}{\log{b}}\log{K}}$.

This is a significant advantage:
while inference with loss-based decoding for general matrices requires $\bigO{d_e\ell+K\ell}$ time,
our model performs it in only $\bigO{d_e\ell+\ell}=\bigO{d_e\ell}$.

\revised{
Since training requires learning $\ell$ binary subproblems,
the \emph{training time complexity} is also sublinear in $K$.
These subproblems can be learned separately on $\ell$ cores,
leading to major speedups.
}

\subsection{\revised{Wider graphs induce sparse models}}
\label{sec:sparsity}

The high sparsity typical to extreme classification datasets
(e.g. the {\tt Dmoz} dataset has $d=833,484$ features,
but on average only $d_e=174$ of them are non-zero),
is heavily exploited by previous works
such as PD-Sparse \cite{En-HsuYen2016},
PPDSparse \cite{Yen:2017:PPP:3097983.3098083},
and DiSMEC \cite{Babbar2017}, which all learn sparse models.

Indeed, we find that for sparse datasets,
our algorithm typically learns a model with a low percentage of non-zero weights.
Moreover, the percentage of non-zero decreases significantly as the slice width $b$ is increased
(see \appref{app:sparsityExperimental}).
This allows us to employ a simple post-pruning of the learned weights.
For some threshold value $\lambda$,
we set to zero all learned weights in $\left[-\lambda, \lambda\right]$,
yielding a sparse model.
Similar approaches were taken by \cite{Babbar2017, Jasinska2016, En-HsuYen2016}
either explicitly or implicitly.

In \secref{sec:sparsityExperiment} we show that
the above scheme successfully yields highly sparse models.

\section{Experiments}
\label{sec:experiments}

We test our algorithms on 5 extreme multiclass datasets previously used in
\cite{En-HsuYen2016}, having approximately $10^2$, $10^3$, and $10^4$
classes (see \tabref{datasets_table} in \appref{supp:datasets}).
We use AROW \cite{Crammer2009} to train the binary functions $\left\{f_j\right\}_{j=1}^{\ell}$ of W-LTLS.
Its online updates are based on the squared hinge loss $L_{SH}\left(z\right)\triangleq\left(\max\left(0,1-z\right)\right)^2$.
%(we used the constant hyperparameter $r=1$).
For each dataset, we build wide graphs with multiple slice widths.
%(depending on the number of classes of the specific dataset).
For each configuration (dataset and graph) we perform five runs using
random sample shuffling on every epoch,
and a random path assignment (as explained in \secref{pathAssignment},
unlike the greedy policy used in \cite{Jasinska2016}),
and report averages over these five runs.
%The accuracies and total binary
%losses reported below are averaged over these five runs.
%
%
Unlike \cite{Jasinska2016}, we train the $\ell$ binary learners
{\em independently} rather than in a \revised{joint (structured)} manner.
This allows parallel independent training,
as common for training binary learners for ECOC,
with no need to perform full multiclass inference during training.% time.

\subsection{Loss-based decoding}
\label{lbd-simulations}

\begin{figure*}[t]
\vskip 0.2in
\begin{center}

\includegraphics[width=.2\linewidth]{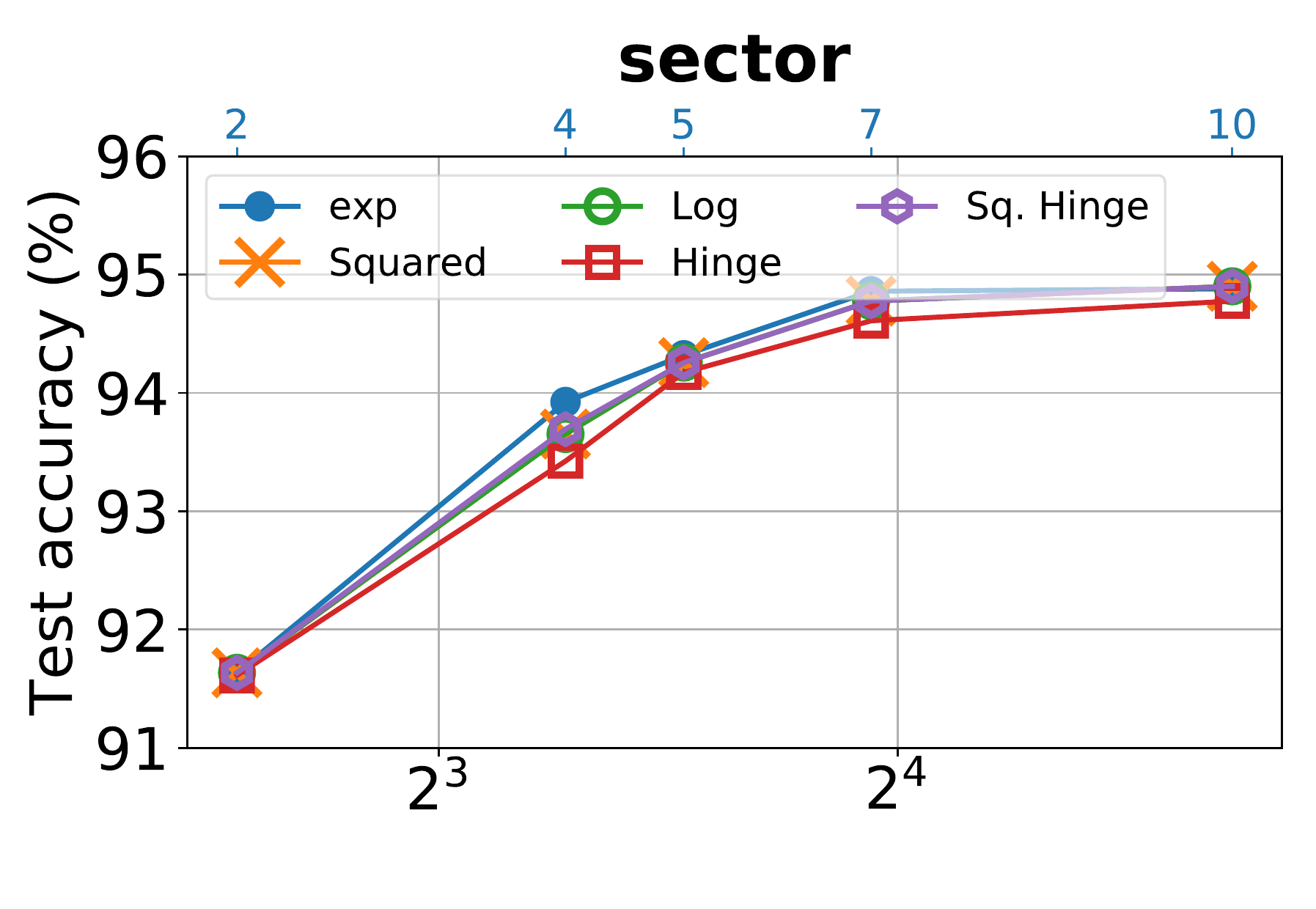}
\includegraphics[width=.1905\linewidth]{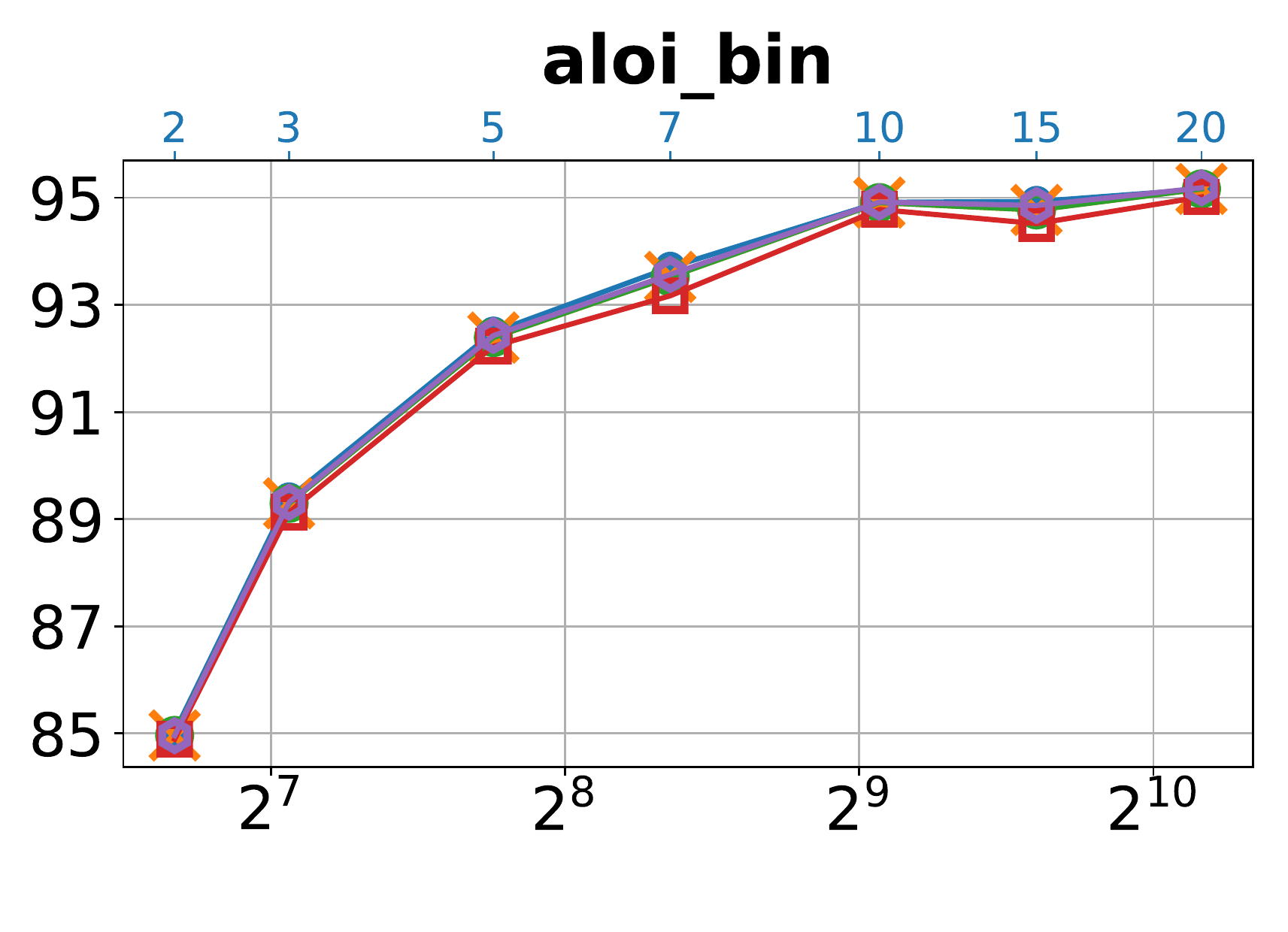}
\includegraphics[width=.1905\linewidth]{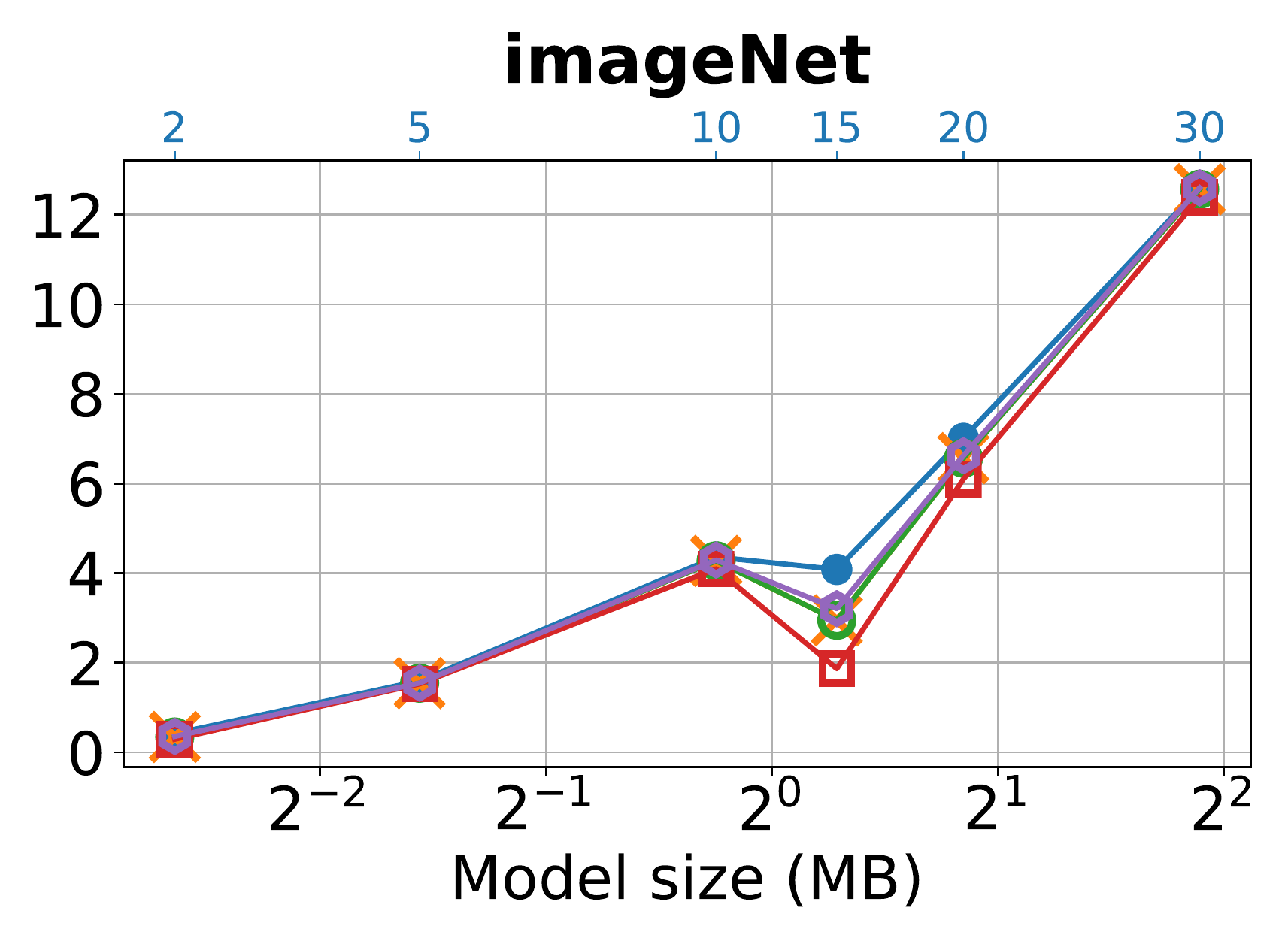}
\includegraphics[width=.1905\linewidth]{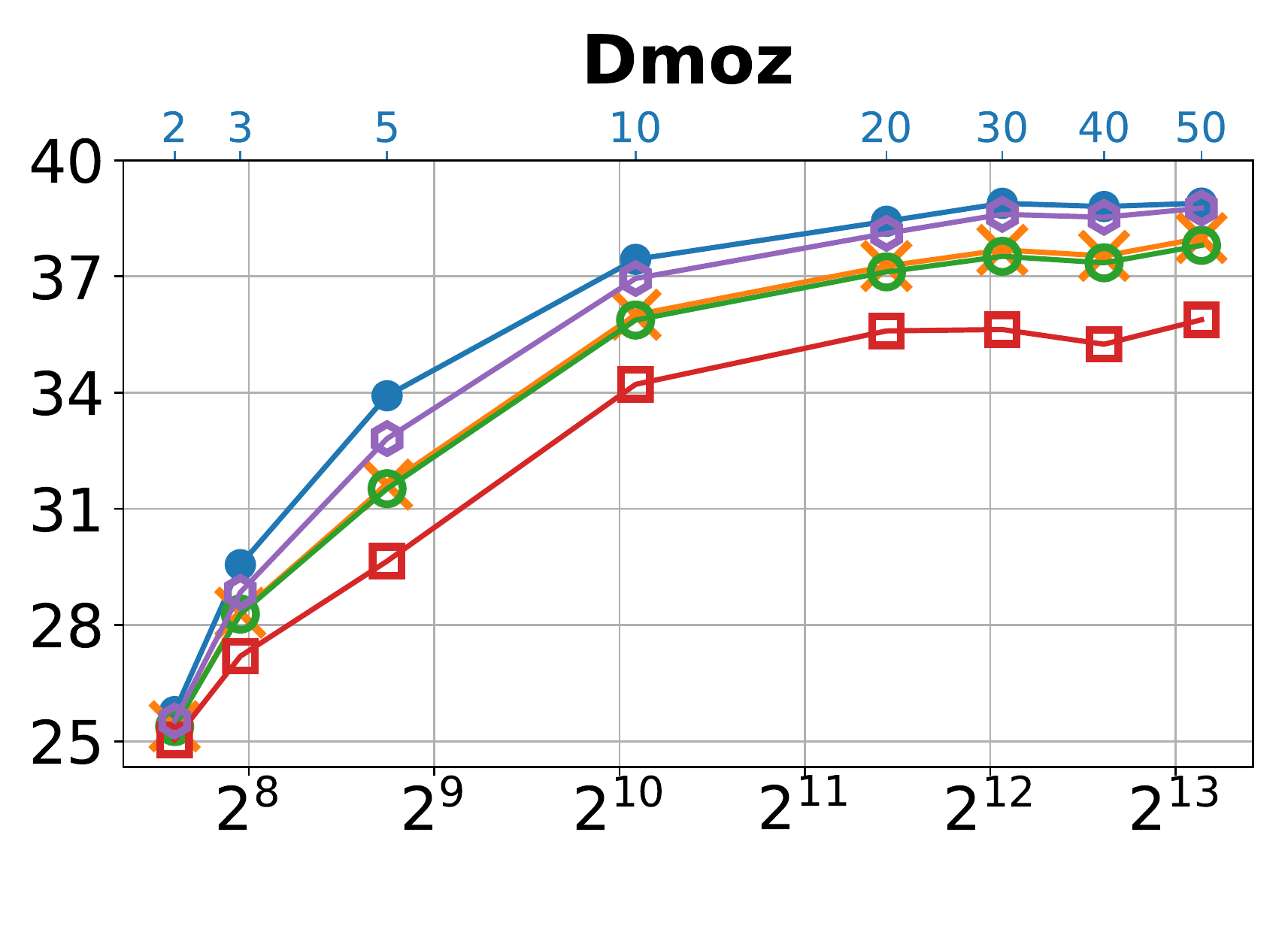}
\includegraphics[width=.1905\linewidth]{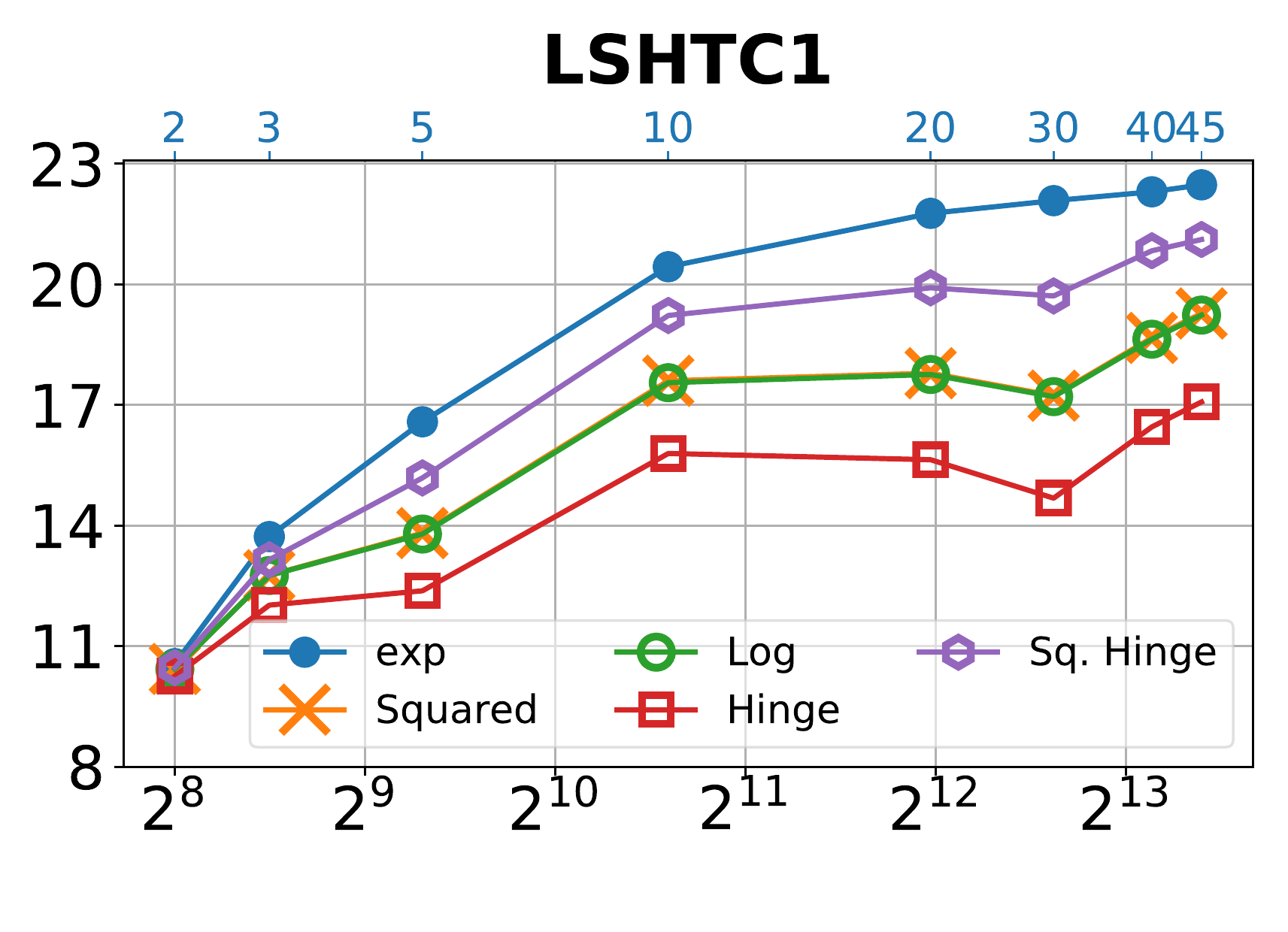}

\vskip -0.03in

\includegraphics[width=.2\linewidth]{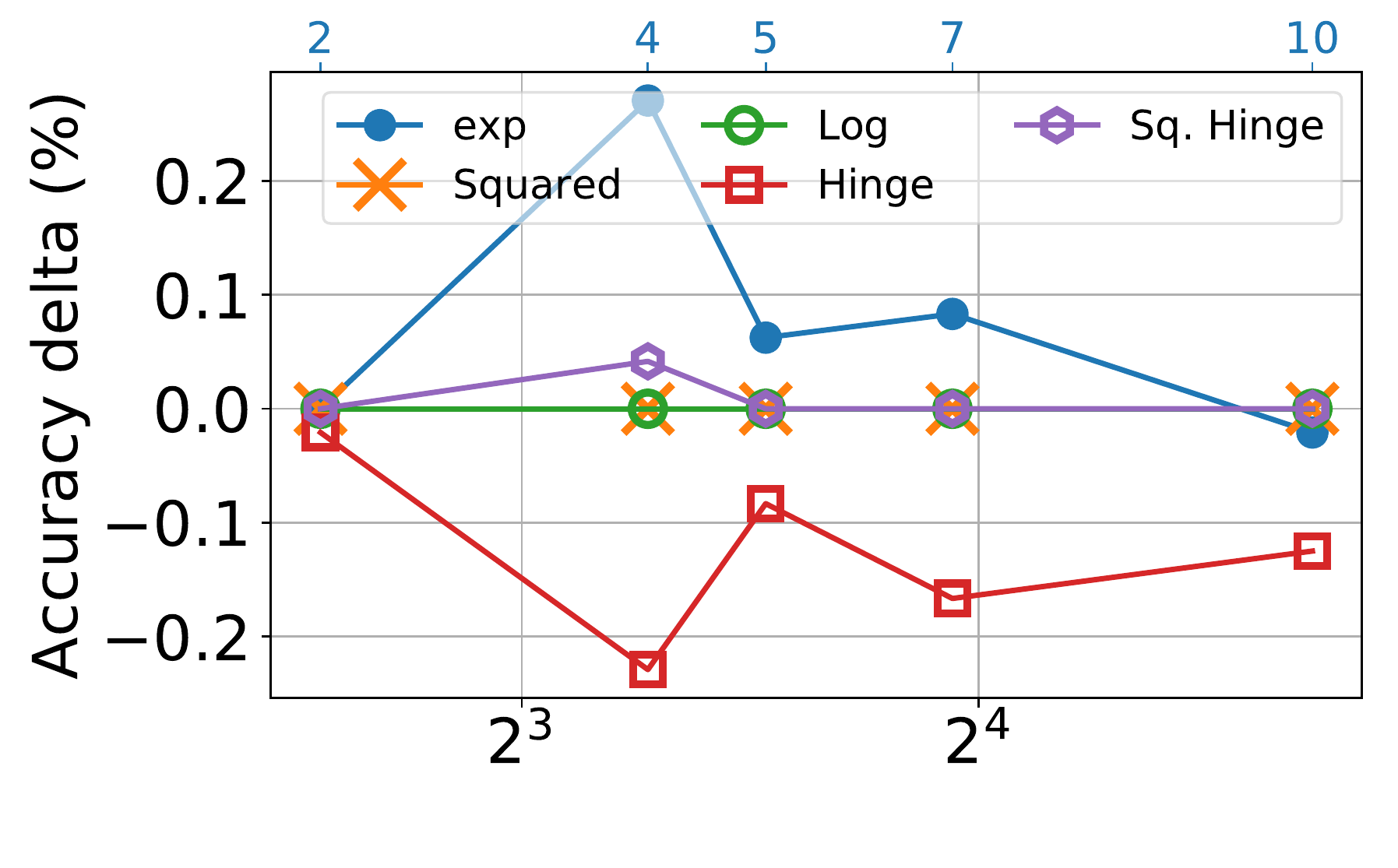}
\includegraphics[width=.1905\linewidth]{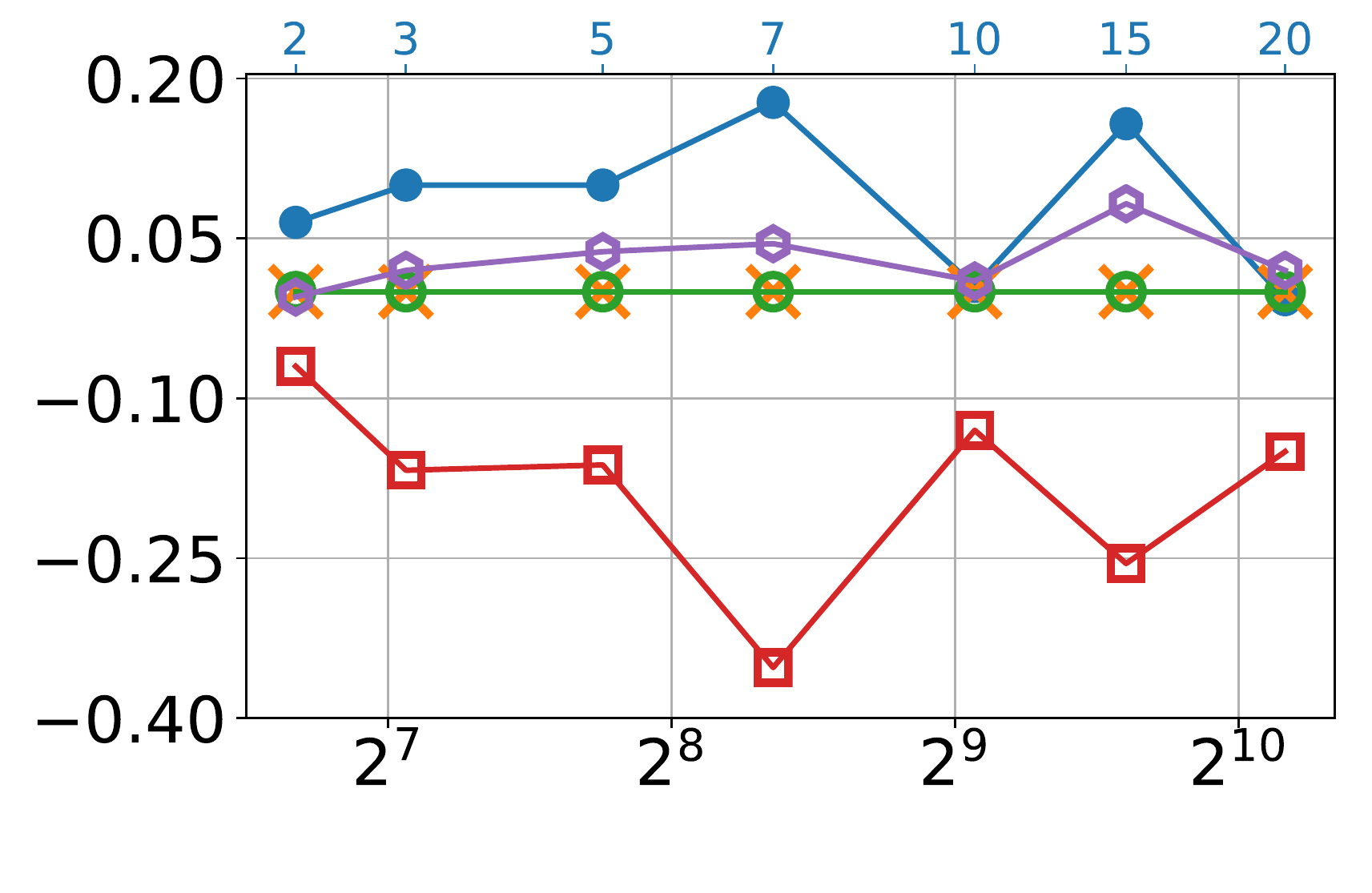}
\includegraphics[width=.1905\linewidth]{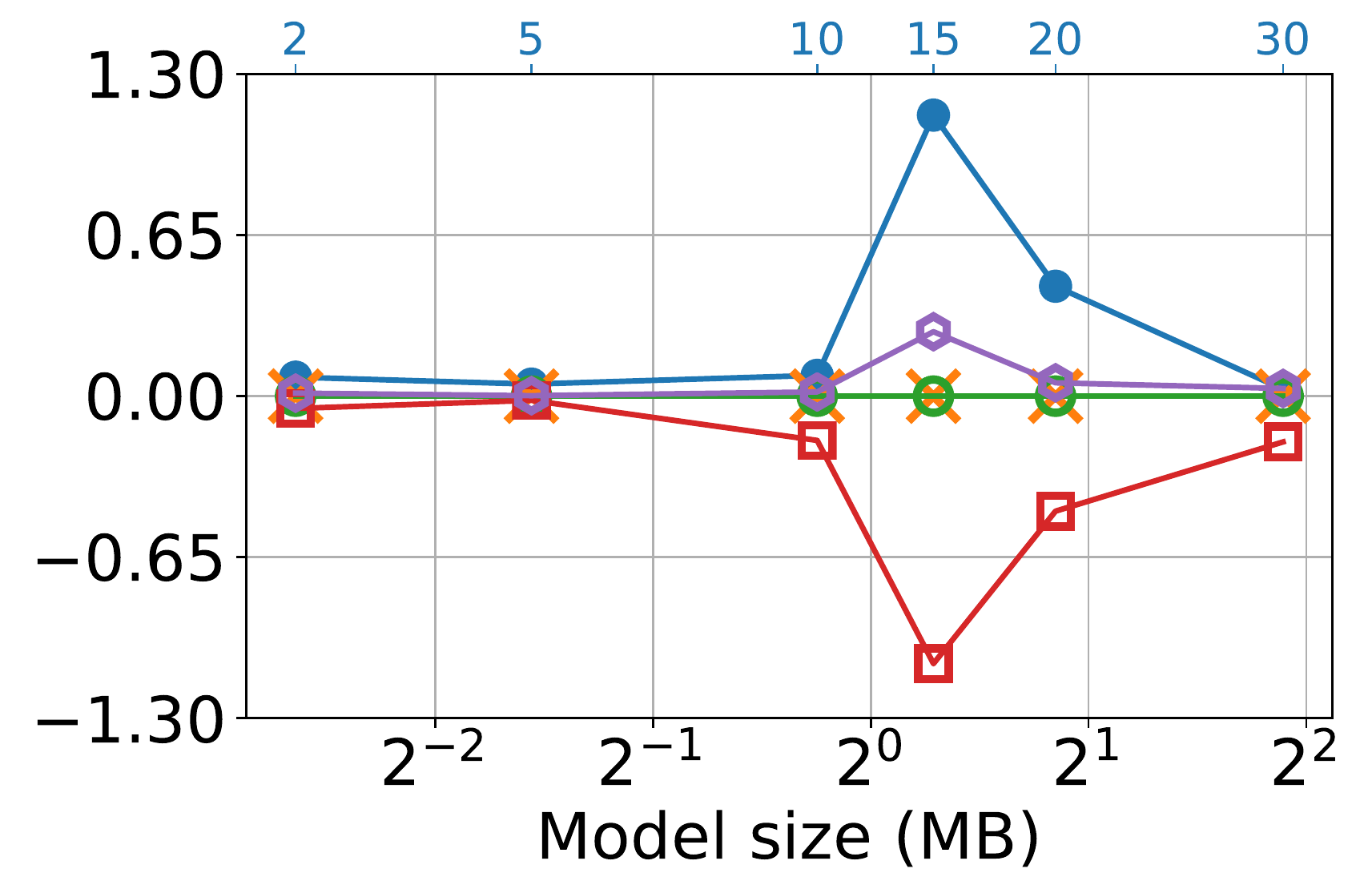}
\includegraphics[width=.1905\linewidth]{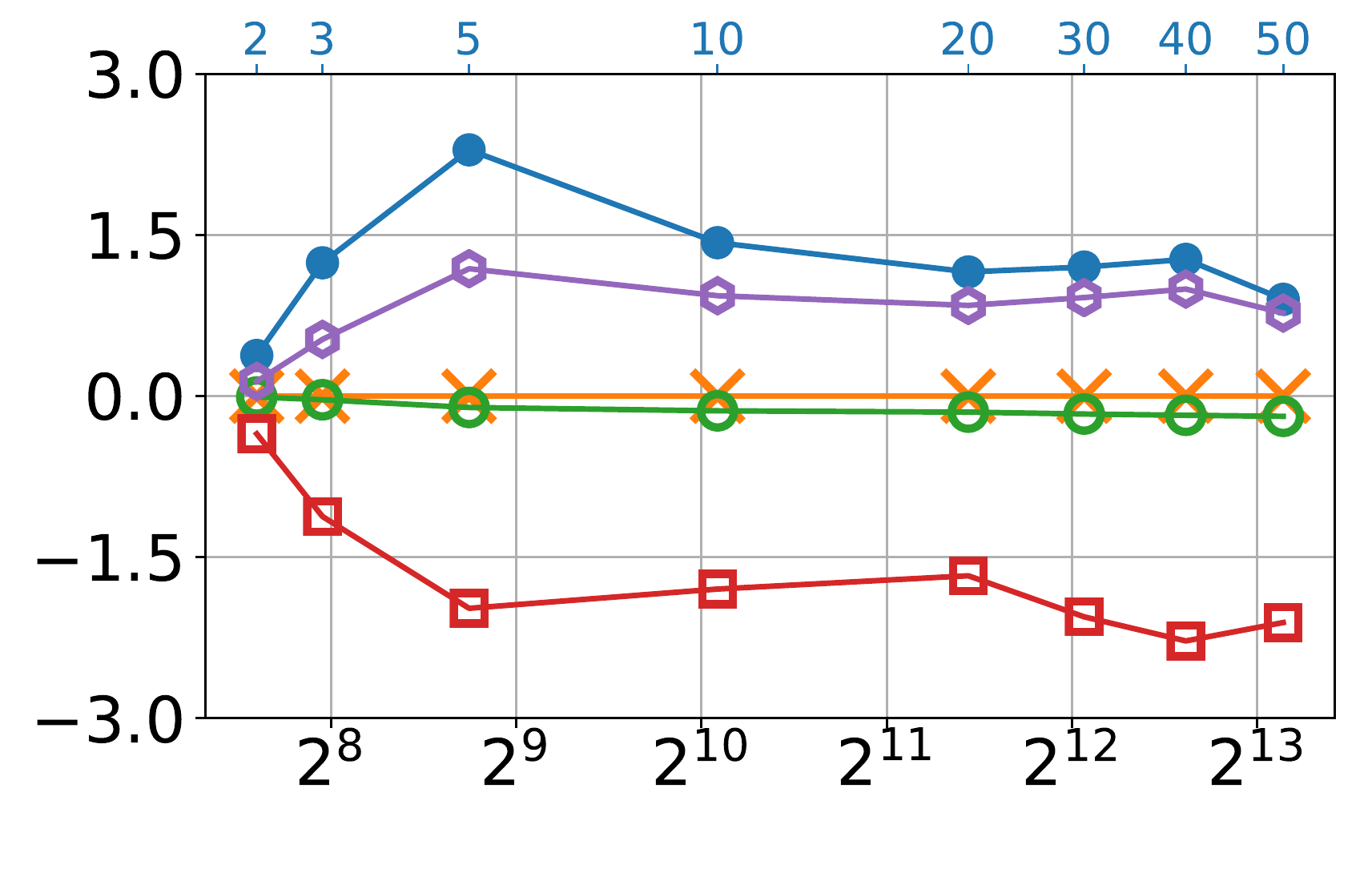}
\includegraphics[width=.1905\linewidth]{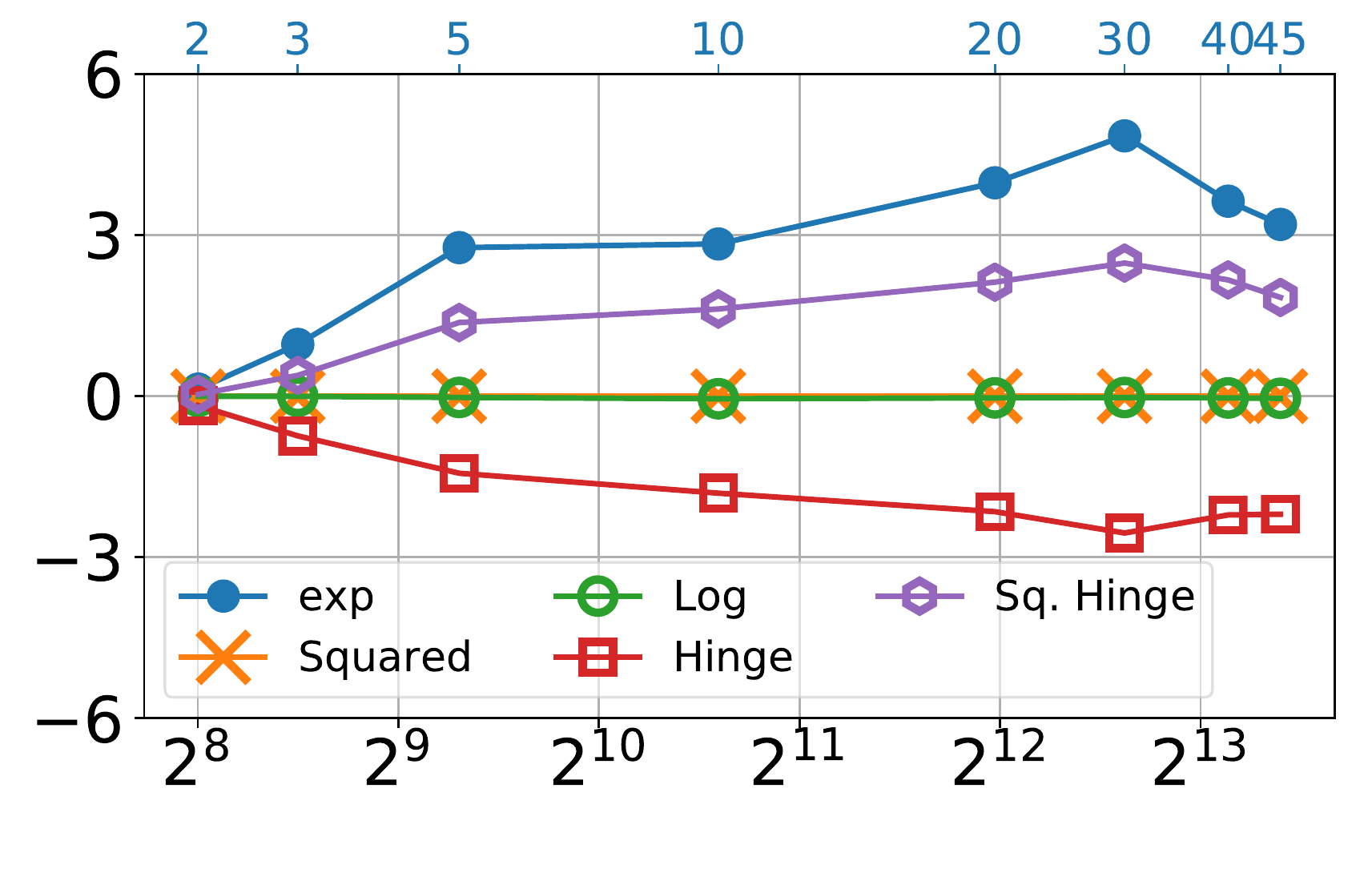}

\vskip -0.15in

\caption{
First row:
Multiclass test accuracy as a function of the model size (MBytes)
for loss-based decoding with different loss functions.
Second row:
Relative increase in multiclass test accuracy
compared to decoding with the squared loss used implicitly in LTLS.
The secondary x-axes (top axes, \textcolor[rgb]{0.17,0.50,0.72}{blue}) indicate the slice widths ($b$)
used for the W-LTLS trellis graphs.
}
\label{fig:loss_based_decoding_losses}
\end{center}
\vskip -0.1in
\end{figure*}

We run W-LTLS with different loss functions for loss-based decoding:
the exponential loss, the squared loss (used by LTLS, see \lemref{lemma:ltls_decoding}),
the log loss, the hinge loss, and the squared hinge loss.

The results appear in \figref{fig:loss_based_decoding_losses}.
We observe that decoding with the exponential loss works the best on all five datasets.
For the two largest datasets ({\tt Dmoz} and {\tt LSHTC1})
we report significant accuracy improvement when using the exponential loss
for decoding in graphs with large slice widths ($b$),
over the squared loss used implicitly by LTLS.
Indeed, for these larger values of $b$,
the subproblems are easier (see \appref{supp:avgBinaryLoss} for detailed analysis).
This should result in larger prediction margins $\left|f_j\left(x\right)\right|$,
as we indeed observe empirically (shown in \appref{averageMargin}).
The various loss functions $L\left(z\right)$ differ significantly for $z\ll0$,
potentially explaining why we find larger accuracy differences as $b$ increases
when decoding with different loss functions.

\subsection{Multiclass test accuracy}
\label{WLTLS-simulations}

\begin{figure*}[t]
%\vskip 0.2in
\begin{center}

\includegraphics[width=.2\linewidth]{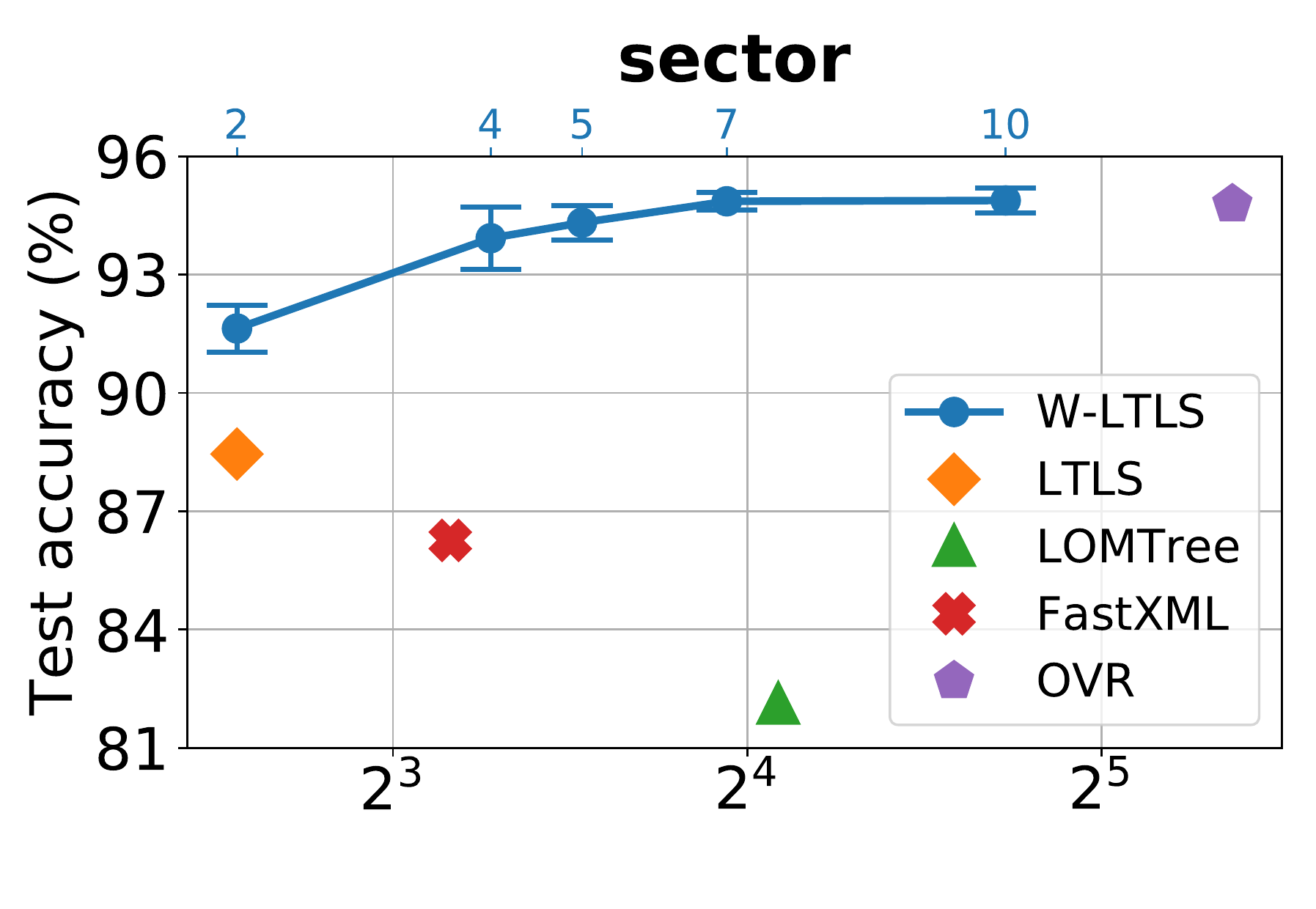}
\includegraphics[width=.1905\linewidth]{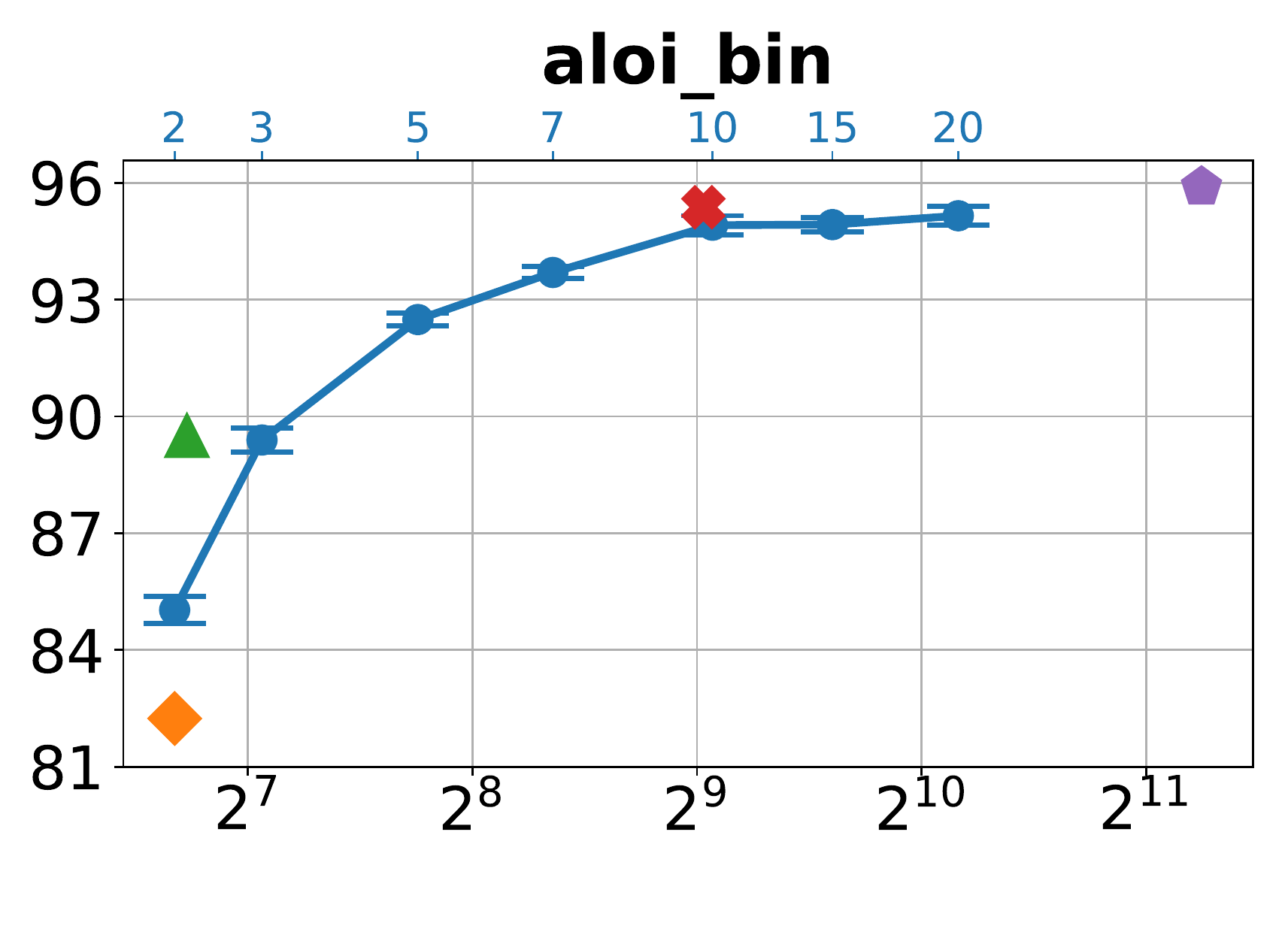}
\includegraphics[width=.1905\linewidth]{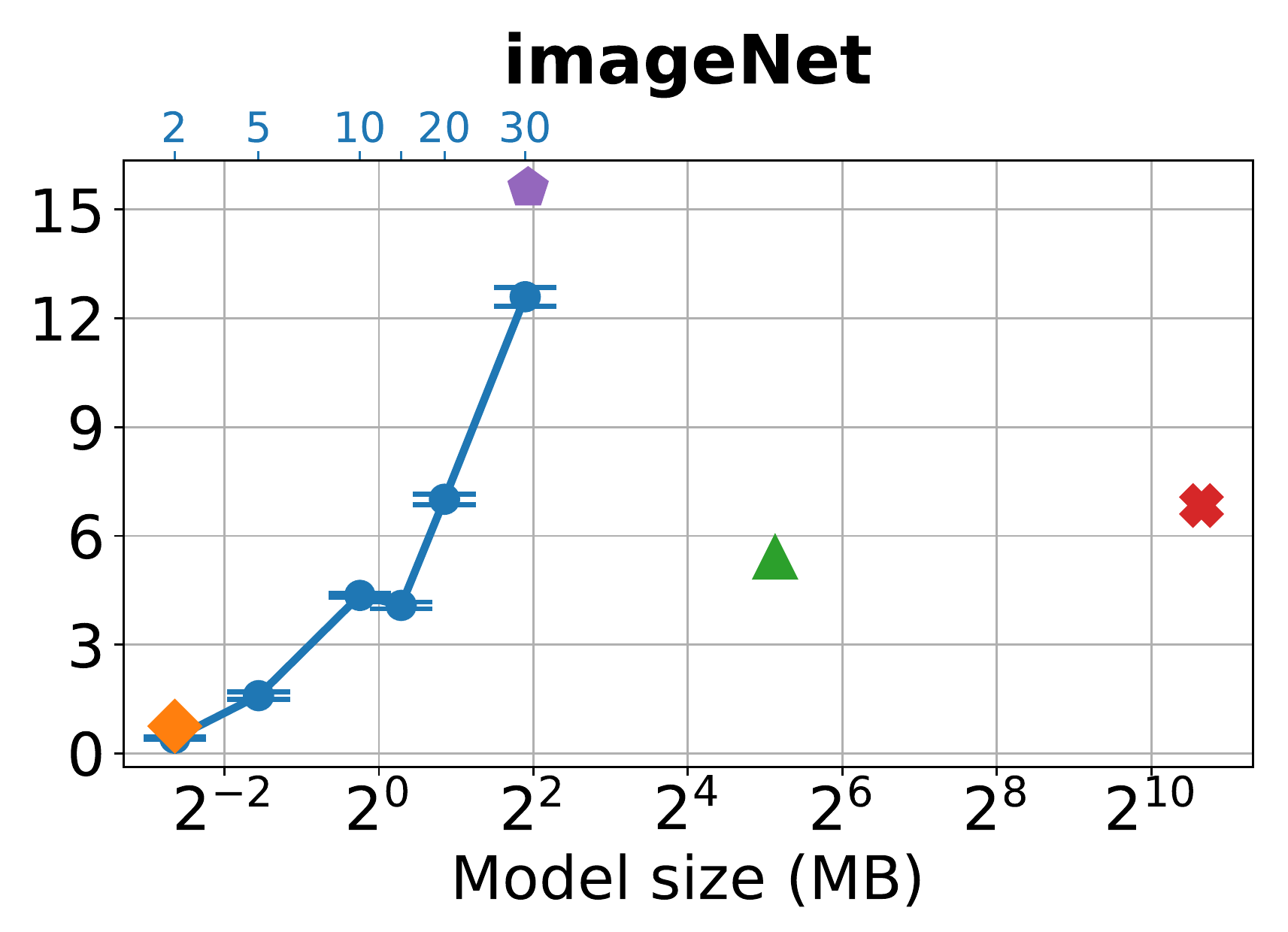}
\includegraphics[width=.1905\linewidth]{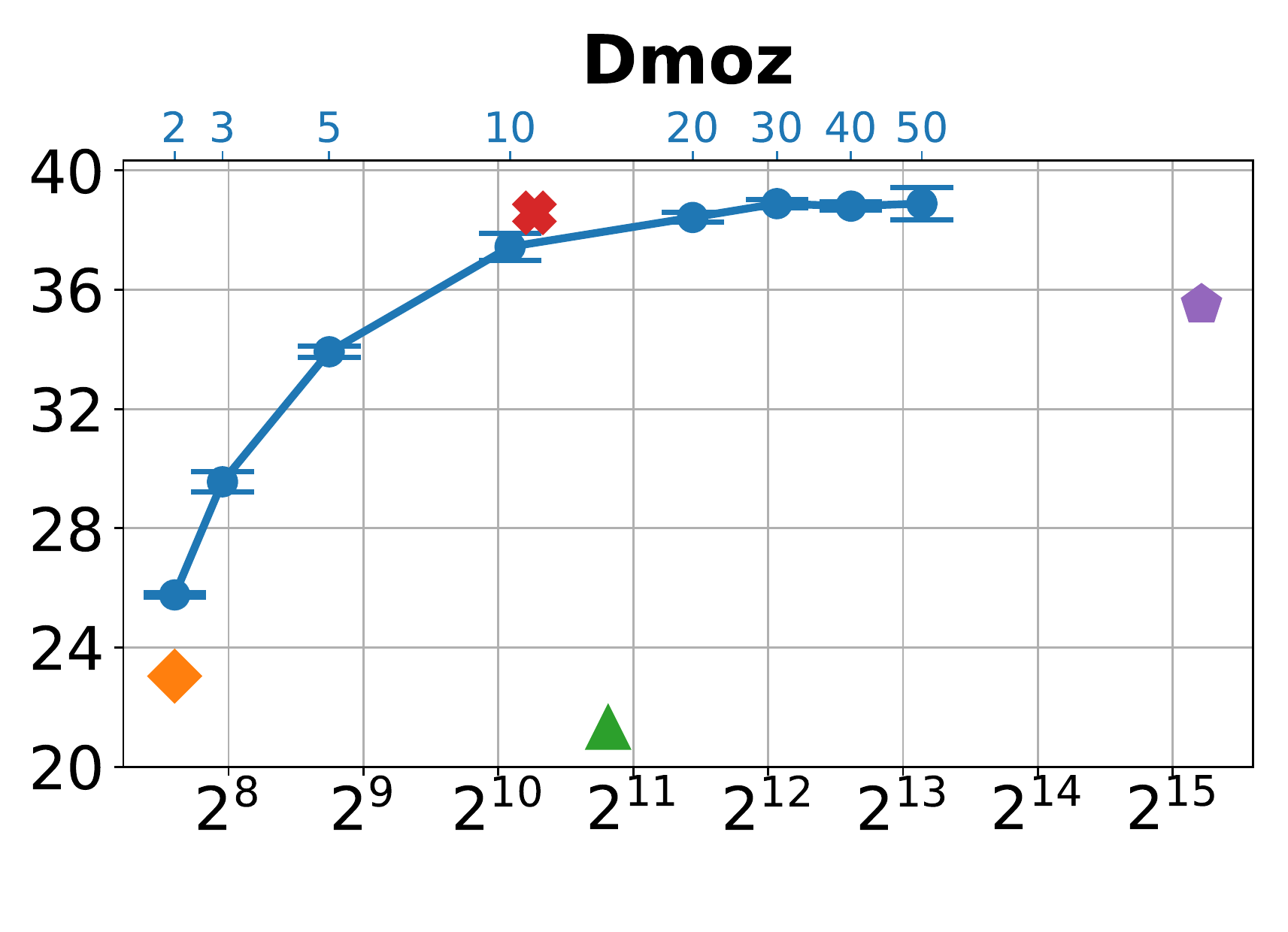}
\includegraphics[width=.1905\linewidth]{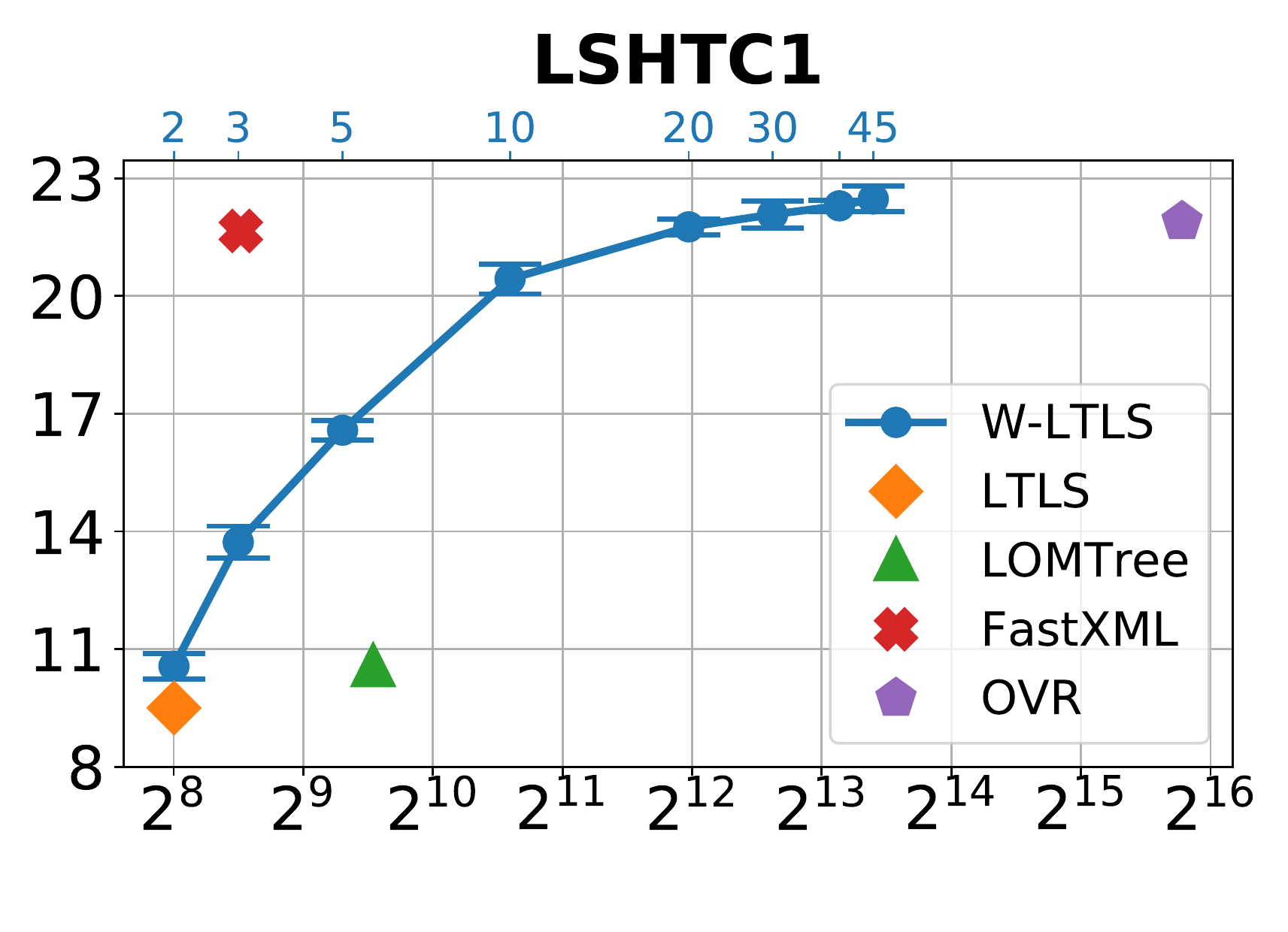}

\vskip -0.03in

\includegraphics[width=.2\linewidth]{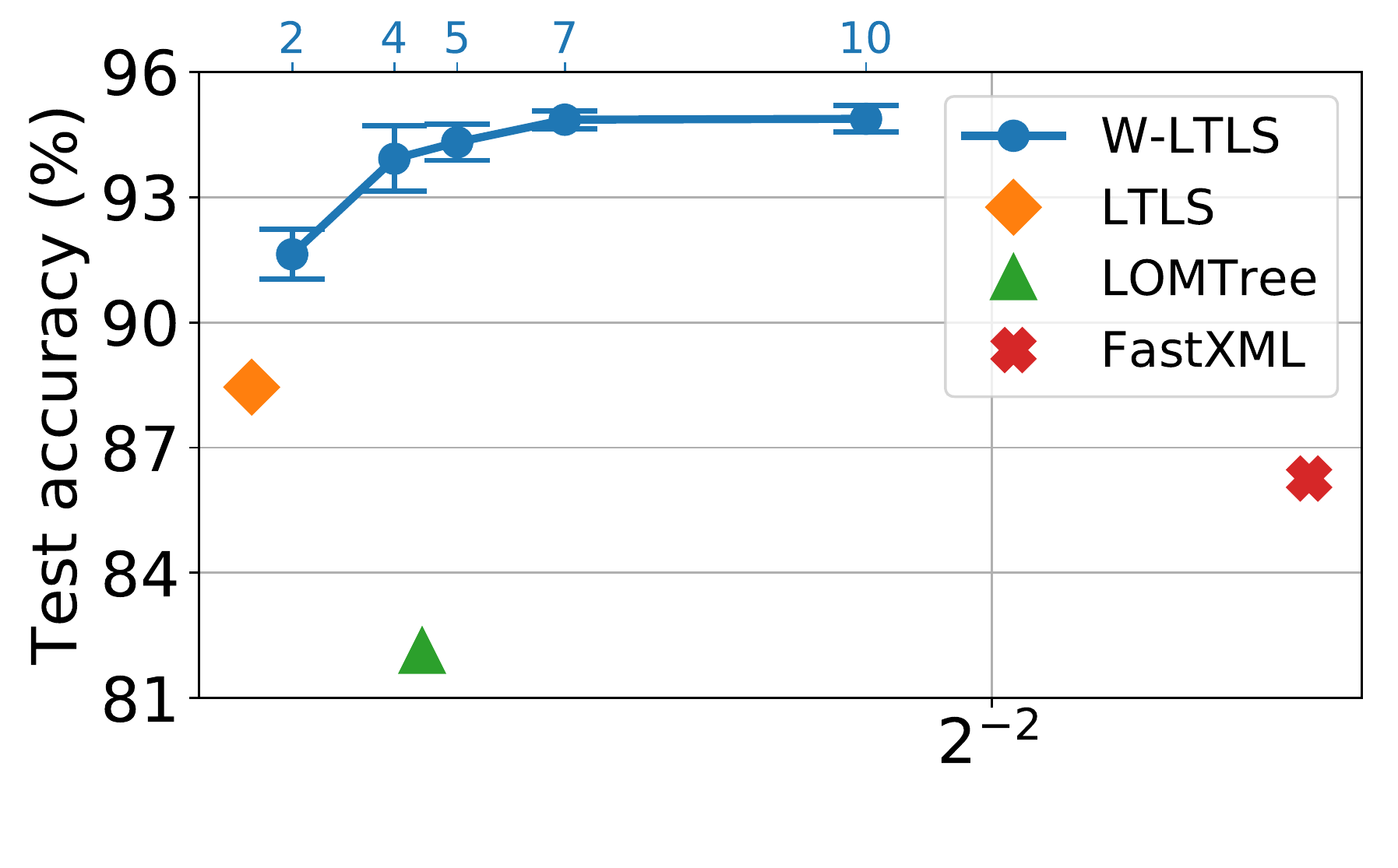}
\includegraphics[width=.1905\linewidth]{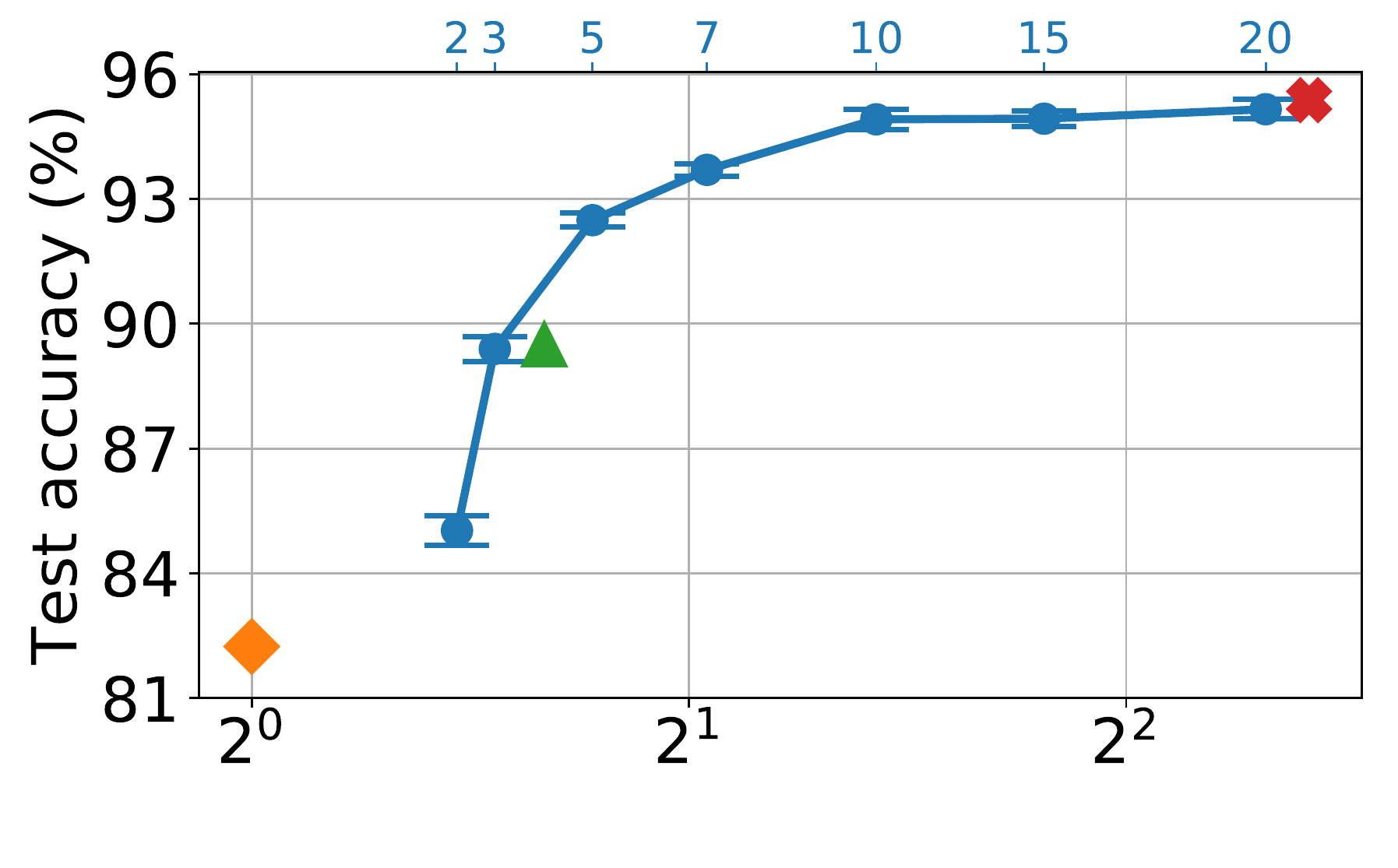}
\includegraphics[width=.1905\linewidth]{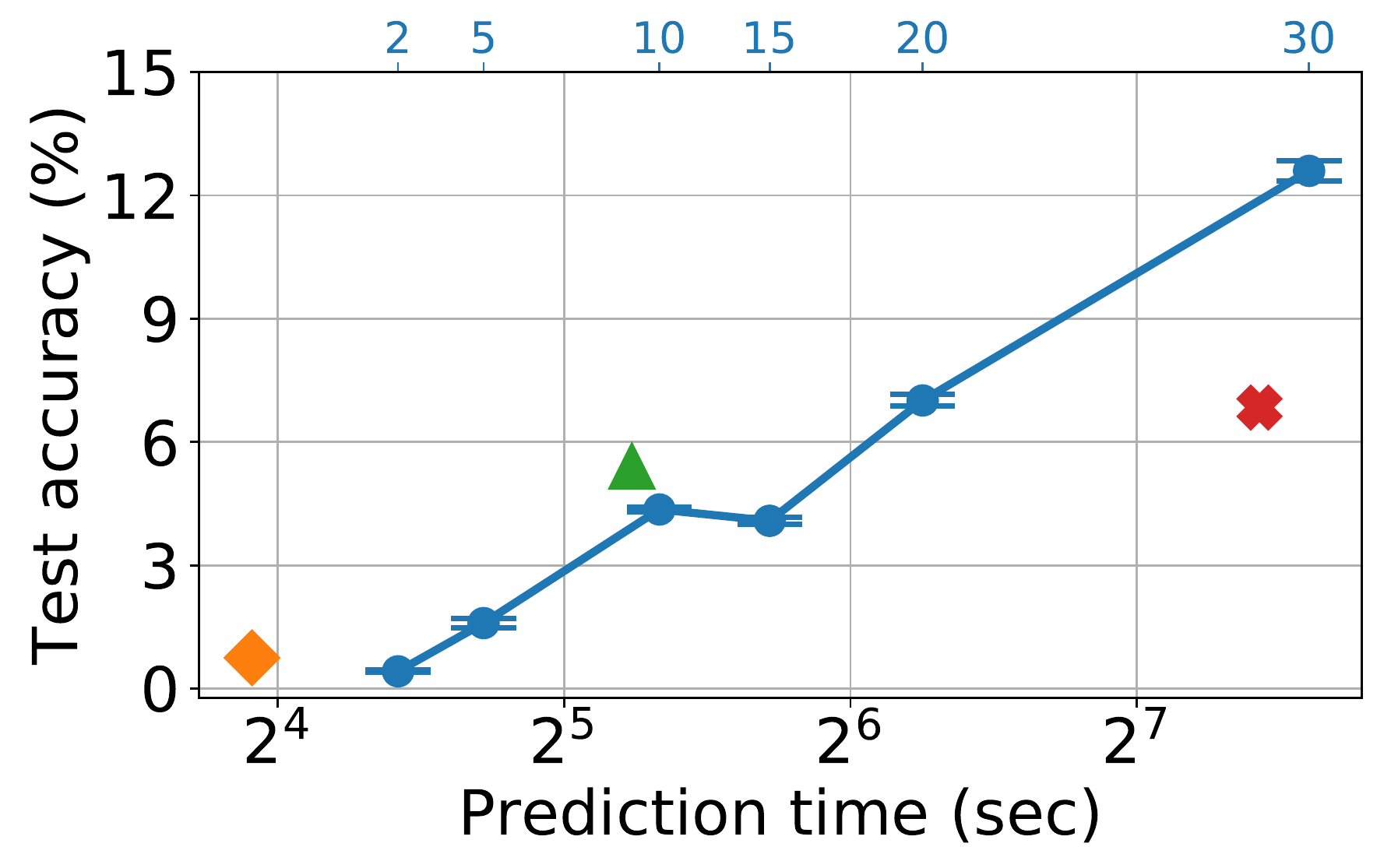}
\includegraphics[width=.1905\linewidth]{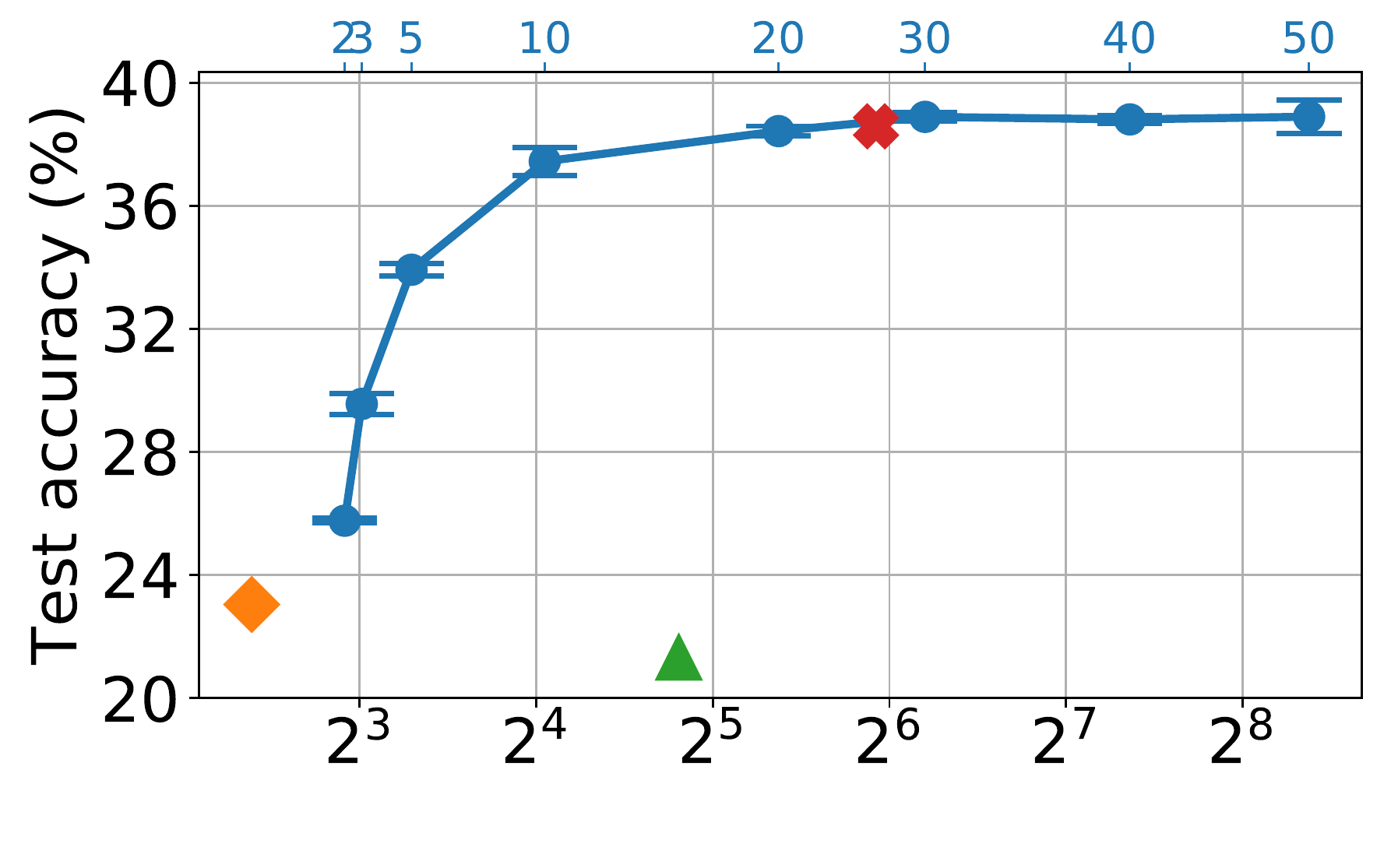}
\includegraphics[width=.1905\linewidth]{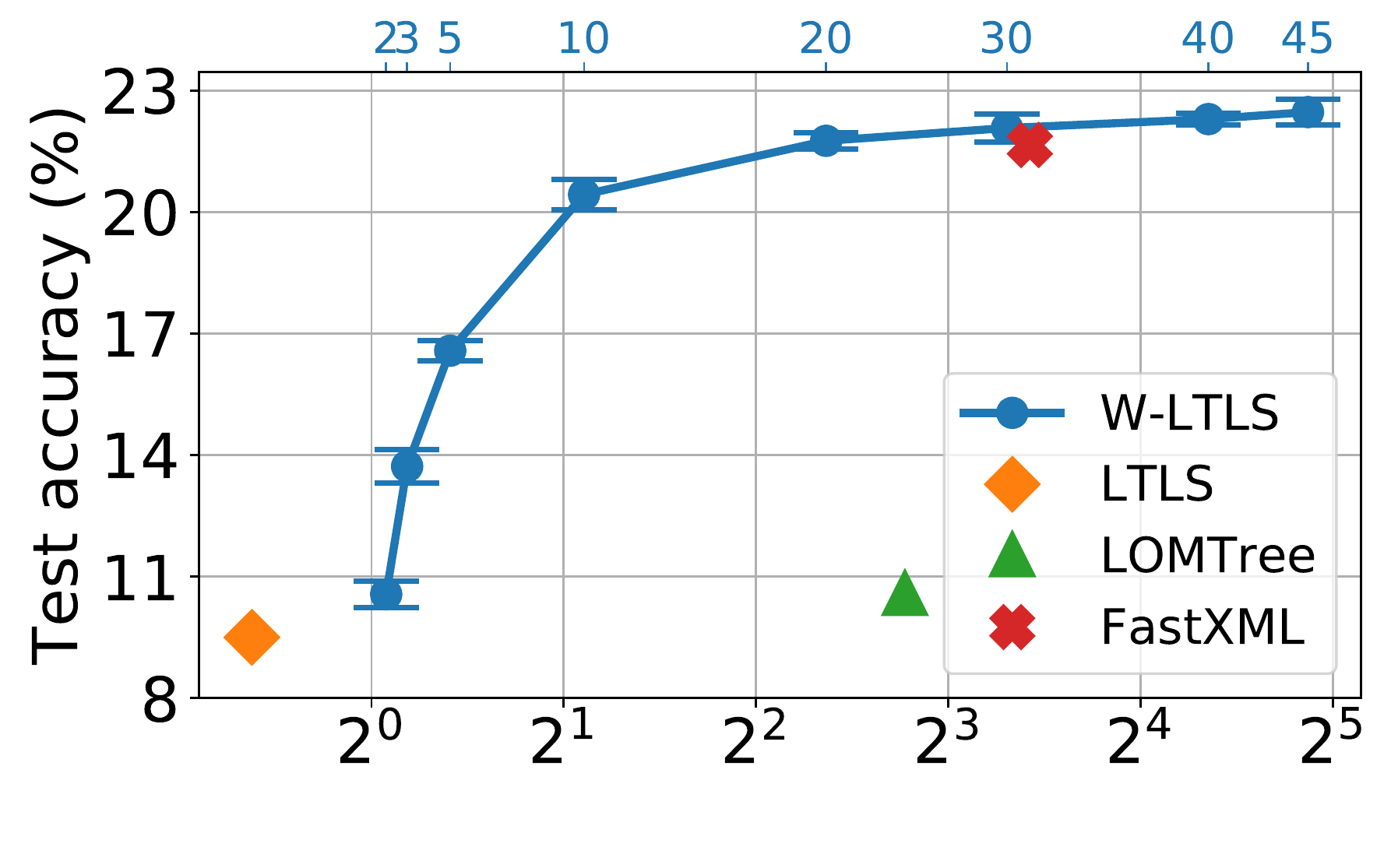}
\vskip -0.1in
\caption{
First row: Multiclass test accuracy vs model size.
Second row: Multiclass test accuracy vs prediction time.
A 95\% confidence interval is shown for the results of W-LTLS.
}
\label{fig:wltls}
\end{center}
\vskip -0.15in
\end{figure*}

We compare the multiclass test accuracy of W-LTLS
(using the exponential loss for decoding) to the same baselines presented in \cite{Jasinska2016}.
Namely we compare to
LTLS \cite{Jasinska2016}, LOMTree
\cite{Choromanska2014} (results quoted from \cite{Jasinska2016}),
%results quoted from \cite{Jasinska2016,En-HsuYen2016}
FastXML \cite{Prabhu:2014:FFA:2623330.2623651}
(\revised{run with the default parameters on the same computer as our model}),
and OVR (binary learners trained using AROW).
\revised{For convenience, the results are also presented in a tabular form in \appref{sec:experimentalDenseResults}.}

%\revised{Below in \secref{sparsity},
% we compare a sparse variant of W-LTLS
% to other extreme classification competitors which employ sparsity.}

%Note that some other competitors,
%like PD-Sparse \cite{En-HsuYen2016},
%PPDSparse \cite{Yen:2017:PPP:3097983.3098083},
%and DiSMEC \cite{Babbar2017} employ sparse models
%(e.g. by $L1$ regularization)
%to reduce the model size and in some cases also the prediction time.
%%
%This line of research is orthogonal to the ECOC approach,
%which can also employ sparse models to reduce the model size.
%%Therefore we do not compare our results to these methods.
%\revised{Indeed, below in \secref{sparsity} we compare a sparse variant of W-LTLS to other competitors which employ sparsity.}

\subsubsection{Accuracy vs Model size}

The first row of \figref{fig:wltls} (best seen in color) summarizes the
multiclass accuracies vs model size.
%We use the accuracy measure (which is one minus the multiclass error)
%to align with results we quote from previous papers.
%We compare W-LTLS (with squared hinge loss-based decoding) to the same baselines presented in \cite{Jasinska2016}, namely we compare to
%%algorithms whose inference time is also logarithmic in $K$:
%LTLS \cite{Jasinska2016}, LOMTree
%\cite{Choromanska2014} (results quoted from \cite{Jasinska2016}) and FastXML \cite{Prabhu:2014:FFA:2623330.2623651} (results quoted from \cite{Jasinska2016,En-HsuYen2016}),
%and to the OVR. % approach whose inference time is linear in $K$
%(red marker, results quoted from earlier paper \cite{En-HsuYen2016} using SVM to learn the binary models).
%(trained using AROW for the binary learners).
%Note that some other competitors, like PD-Sparse \cite{En-HsuYen2016}, PPDSparse \cite{Yen:2017:PPP:3097983.3098083} and DiSMEC \cite{Babbar2017} employ sparse models (e.g. by $L1$ regularization) to reduce the model size and in some cases also the prediction time.
%This line of research is orthogonal to the ECOC approach, which can also employ sparse models to reduce the model size.
%Therefore we do not compare to these methods in our results.
%Error bars indicate $95\%$ confidence intervals.

Among the four competitors, LTLS enjoys the smallest model size,
%(except in {\tt aloi\_bin})
LOMTree and FastXML have larger model sizes, and OVR is the largest.
LTLS achieves lower accuracies than LOMTree on two datasets,
and higher ones on the other two.
OVR enjoys the best accuracy, yet with a price of model size.
For example, in {\tt Dmoz},
LTLS achieves $23\%$ accuracy vs
\revised{$35.5\%$ of OVR, though the model size of the latter is $\times200$ larger than of the former.}

In all five datasets,
 an increase in the slice width of W-LTLS (and consequently in the model size)
 translates almost always to an increase in accuracy.
%By exploiting the tradeoff between accuracy and model size,
%we report statistically significant major improvements in accuracies in comparison to the
%other three algorithms also having logarithmic inference time
%complexity (LTLS, LOMTree, FastXML).
%For larger but still logarithmic in $K$ model sizes, our
%model's accuracy is often competitive with the accuracy of OVR.
Our model is often better or competitive with the other algorithms
that have logarithmic inference time complexity (LTLS, LOMTree, FastXML),
and also competitive with OVR in terms of accuracy,
while we still enjoy much smaller model sizes.

For the smallest model sizes of W-LTLS (corresponding to $b=2$), our
trellis graph falls back to the one of LTLS. The accuracies gaps
%in the second row of \figref{fig:wltls}
 between these two models may be explained by the
different binary learners the experiments were run with -- LTLS
used averaged Perceptron as the binary learner whilst we used
AROW. Also, LTLS was trained in a structured manner with a
greedy path assignment policy while we trained every
binary function independently with a random path assignment policy (see \secref{complexityAnalysis}).
In our runs we observed that independent training
achieves accuracy competitive with to structured online
training, while usually converging much faster.
It is interesting to note that for the {\tt imageNet} dataset the LTLS model cannot fit the data,
i.e the training error is close to 1 and the test accuracy is close to 0.
The reason is that the binary subproblems are very hard, as was also noted by \cite{Jasinska2016}.
By increasing the slice width ($b$), the W-LTLS model mitigates this underfitting problem,
still with logarithmic time and space complexity.

We also observe in the first row of \figref{fig:wltls}
%comparing the first row of \figref{fig:wltls} to the second row of
%\figref{fig:wltls}, again focusing on {\tt LSHTC1}, it can be observed
that there is a point where the multiclass test accuracy of W-LTLS
starts to saturate (except for {\tt imageNet}).
Our experiments show that this point can be found by looking at the training error and its bound only.
We thus have an
effective way to choose the optimal model size for the dataset and
space/time budget at hand by performing model selection (width of the
graph in our case) using the training error bound only
(see detailed analysis in \appref{supp:avgBinaryLoss} and \appref{averageBinaryLossVsBalance}).

% and at the same time the training error (and its
%bound) stops decreasing or even increases. A similar behavior occurs
%on the other datasets as well. This observation suggests a simple and
%effective way to choose the optimal model size for the dataset and
%space/time budget at hand by performing model selection (width of the
%graph in our case) using the training error bound only.
%\vskip -0.5in
\subsubsection{Accuracy vs Prediction time}
\label{WLTLS-time}
In the second row of \figref{fig:wltls} we compare prediction (inference) time
 of W-LTLS to other methods.
 LTLS enjoys the fastest prediction time, but suffers from low accuracy.
 LOMTree runs slower than LTLS, but sometimes achieves better accuracy.
 \revised{Despite being implemented in Python,
 W-LTLS is competitive with FastXML, which is implemented in C++.}

 %%%%%%%%%%%%%%%%%%%%%%%%%%%%%%%%%%%%%%%%%%%%%%%%%%%%%%%%%%%%

\subsection{\revised{Exploiting the sparsity of the datasets}}
\label{sec:sparsityExperiment}

We now demonstrate that the post-pruning proposed in \secref{sec:sparsity},
which zeroes the weights in $\left[-\lambda, \lambda\right]$,
is highly beneficial.
Since {\tt imageNet} is not sparse at all,
we do not consider it in this section.

We tune the threshold $\lambda$
so that the degradation in the multiclass validation accuracy is at most $1\%$
(tuning the threshold is done after the cumbersome learning of the weights,
and does not require much time).

In \figref{fig:sparse} we plot the multiclass test accuracy versus model size
for the non-sparse W-LTLS,
as well as the sparse W-LTLS after pruning the weights as explained above.
We compare ourselves to the aforementioned sparse competitors:
DiSMEC, PD-Sparse, and PPDSparse
(all results quoted from \cite{Yen:2017:PPP:3097983.3098083}).
Since the aforementioned FastXML \cite{Prabhu:2014:FFA:2623330.2623651}
also exploits sparsity to reduce the size of the learned trees,
%(decisions on lower nodes are based only on few features),
we consider it here as well
(we run the code supplied by the authors for various numbers of trees).
For convenience,
all the results are also presented in a tabular form in \appref{app:sparsityExperimental}.

%In \figref{fig:sparse} we plot the multiclass test accuracy versus model size
%of the non-sparse W-LTLS
%and the sparse W-LTLS after pruning the weights as explained above.
%We compare ourselves to the aforementioned (sparse) competitors:
%FastXML \cite{Prabhu:2014:FFA:2623330.2623651}
%(we run the code supplied by the authors for different numbers of trees),
%DiSMEC, PD-Sparse, and PPDSparse
%(all results quoted from \cite{Yen:2017:PPP:3097983.3098083}).
%%
%For convenience,
%the results are also presented in a tabular form in \appref{app:sparsityExperimental}.

We observe that our method can induce very sparse binary learners with a small degradation in accuracy.
In addition, as expected,
the wider the graphs (larger $b$),
the more beneficial is the pruning.
Interestingly, while the number of parameters increases as the graphs become wider,
the actual storage space for the pruned sparse models may even decrease.
This phenomenon is observed for the \texttt{sector} and \texttt{aloi.bin} datasets.

Finally, we note that although PD-Sparse \cite{En-HsuYen2016} and DiSMEC \cite{Babbar2017}
perform better on some model size regions of the datasets,
their worse case space requirement during training is linear in the number of classes $K$,
whereas our approach guarantees (adjustable) logarithmic space for training.

\begin{figure*}[t]
\begin{center}
\includegraphics[width=.25\linewidth]{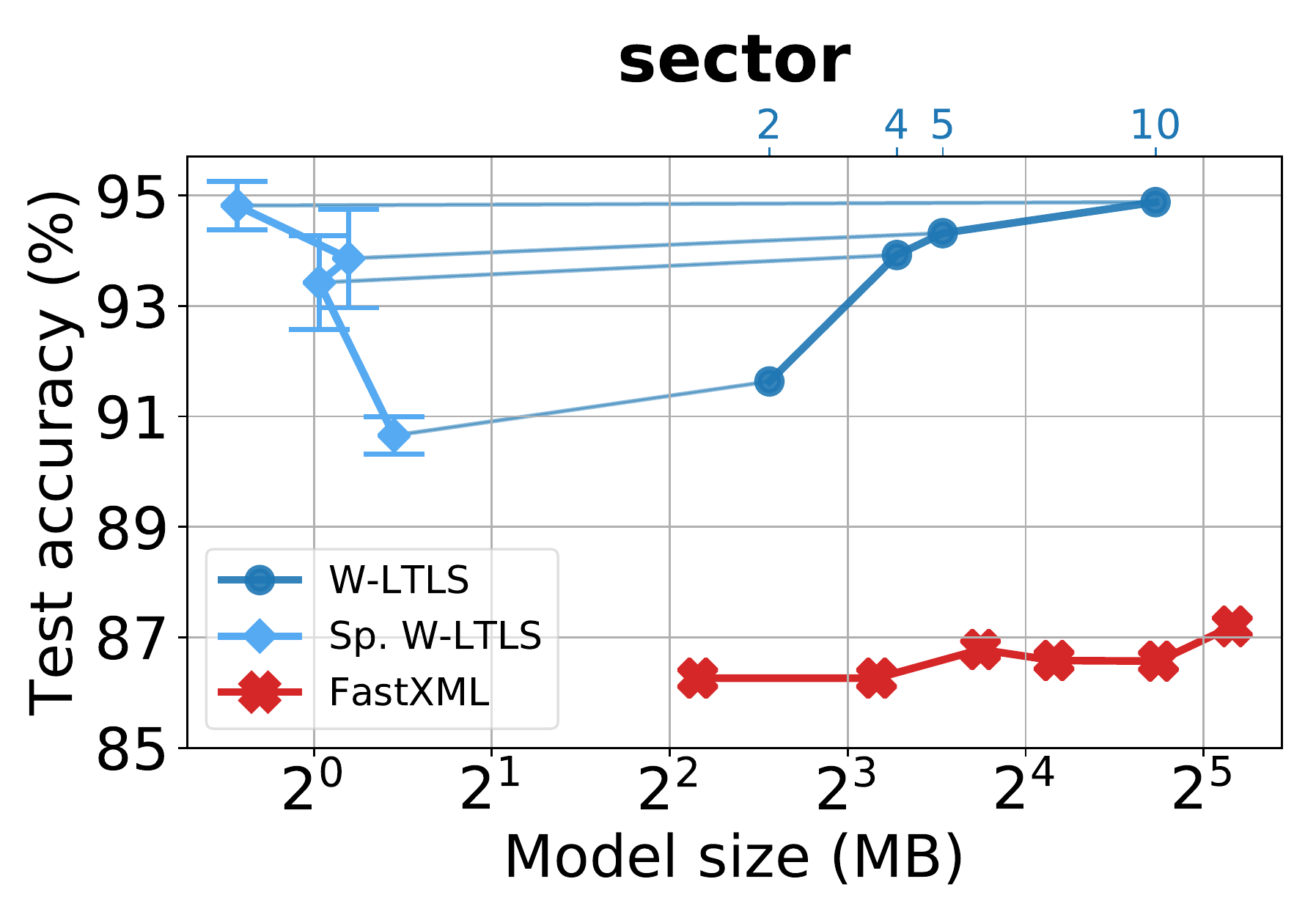}
\includegraphics[width=.24\linewidth]{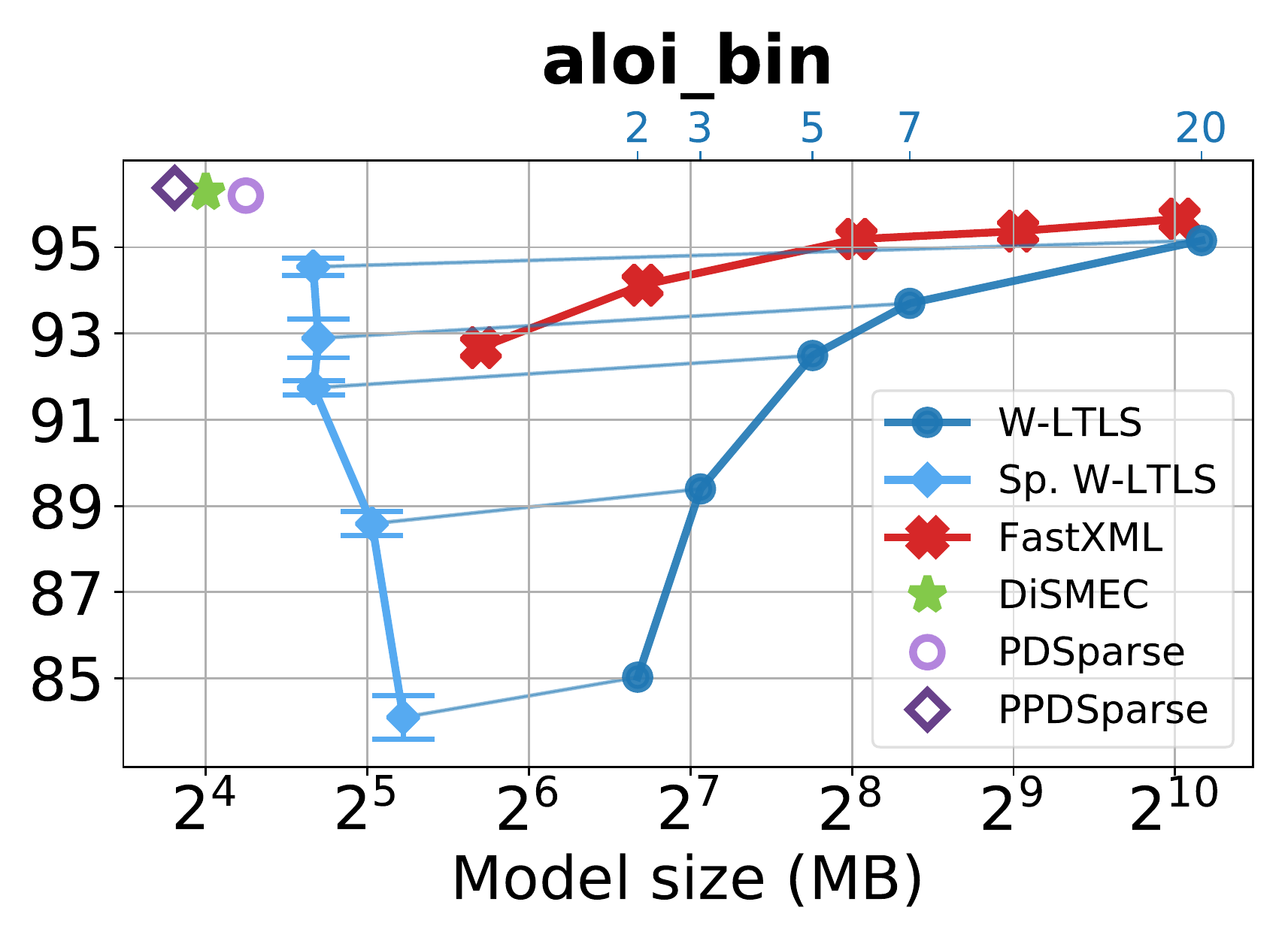}
\includegraphics[width=.24\linewidth]{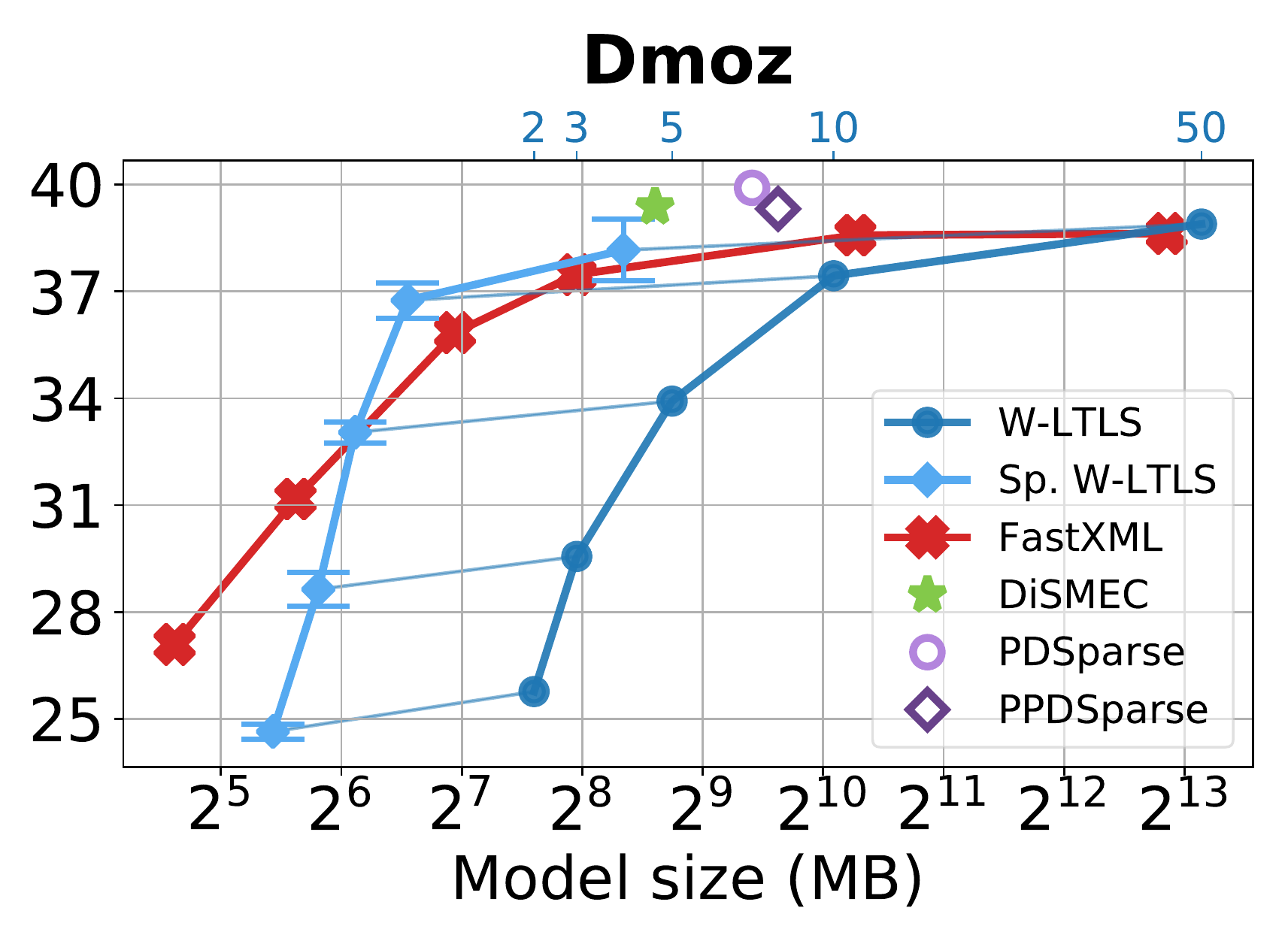}
\includegraphics[width=.24\linewidth]{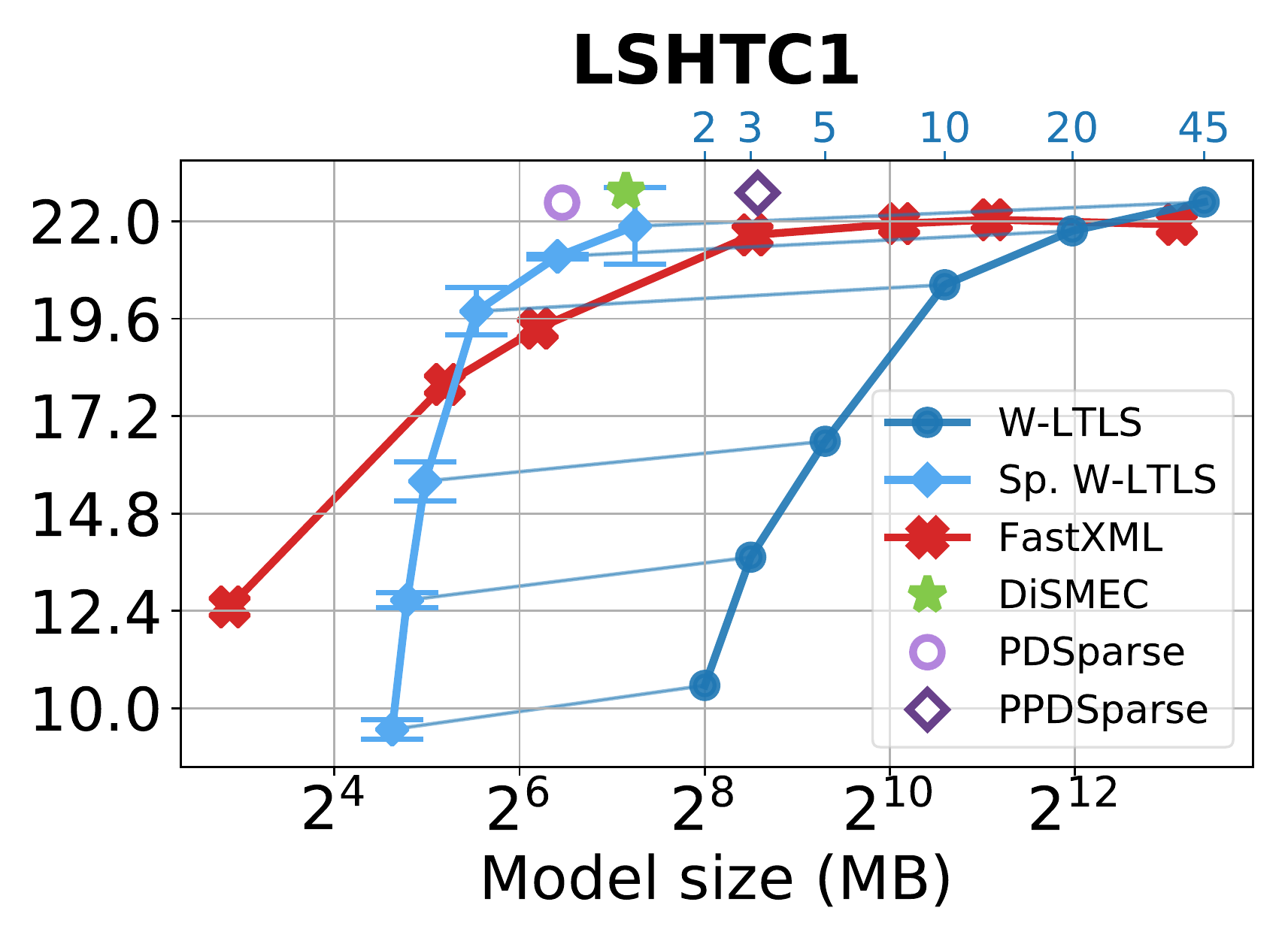}
\caption{
Multiclass test accuracy vs model size for sparse models.
Lines between two W-LTLS plots
connect the same models before and after the pruning.
The secondary x-axes (top axes, \textcolor[rgb]{0.17,0.50,0.72}{blue}) indicate the slice widths ($b$)
used for the (unpruned) W-LTLS trellis graphs.
}
\label{fig:sparse}
\end{center}
\end{figure*}

\section{Related work}
\label{sec:related_work}

Extreme classification was studied extensively in the past decade.
It faces unique challenges, amongst which is the model size of its designated learning algorithms.
An extremely large model size often implies long training and test times,
as well as excessive space requirements.
Also, when the number of classes $K$ is extremely large,
the inference time complexity should be sublinear in $K$ for the classifier to be useful.

The \textit{Error Correcting Output Coding (ECOC)} (see \secref{sec:ecoc}) approach
 seems promising for extreme classification,
 as it potentially allows a very compact representation
 of the label space with $K$ codewords of length $\ell=\bigO{\log{K}}$.
Indeed, many works concentrated on utilizing ECOC for extreme classification.
Some formulate dedicated optimization problems to find ECOC matrices suitable for extreme classification \cite{Cisse2012a} and others focus on learning better binary learners \cite{Liu2016}.

However, very little attention has been given to the decoding time complexity.
In the multiclass regime where only one class is the correct class,
many of these works are forced to use exact (i.e. not approximated)
decoding algorithms which often require $\bigO{K\ell}$ time \cite{Kibriya2007} in the worst-case.
Norouzi et al. \cite{Norouzi2014} proposed a fast exact search nearest neighbor algorithm in the Hamming space,
which for coding matrices suitable for extreme classification can achieve $o\left(K\right)$ time complexity,
but not $\bigO{\log{K}}$.
These algorithms are often limited to binary (dense) matrices and hard decoding.
Some approaches \cite{DLiA15} utilize graphical processing units in order to find the nearest neighbor in Euclidean space, which can be useful for soft decoding,
but might be too demanding for weaker devices.
In our work we keep the time complexity of any loss-based decoding logarithmic in $K$.
%Additionally, while LTLS requires doing inference during training, our approach of W-LTLS does not, for its binary learners are trained separately, thus yielding faster training, that can also be parallelized.

\revised{
Moreover, most existing ECOC methods employ coding matrices with higher minimum distance $\rho$,
but with balanced binary subproblems. % (TODO: cite).
In \secref{sec:Wltls} we explain how our ability of inducing less balanced subproblems
is beneficial both for the learnability of these subproblems, 
and for the post-pruning of learned weights to create sparse models.
}

It is also worth mentioning that many of the ECOC-based works
(like randomized or learned codes \cite{Cisse2012a,Zhao2013})
require storing the entire coding matrix even during inference time.
Hence, the additional space complexity needed only for decoding during inference is $\bigO{K\log{K}}$,
rather than $\bigO{K}$ as in LTLS and W-LTLS 
\revised{which do not directly use the coding matrix for decoding the binary predictions
and only require a mapping from code to label (e.g. a binary tree)}.
% or $\bigO{\log{K}}$ in some block codes employing efficient hard decoding (e.g the Berlekamp-Massey decoding algorithm for BCH codes \cite{Massey69,Berlekamp}).

% TODO : Maybe be more specific?
%In the context of extreme \emph{multilabel} classification several approaches were recently proposed.
\revised{Naturally, hierarchical classification approaches are very popular for extreme classification tasks.
Many of these approaches employ tree based models \cite{Bengio2010,Prabhu:2014:FFA:2623330.2623651,Prabhu:2018:PPL:3178876.3185998,Jain:2016:EML:2939672.2939756,DBLP:conf/icml/JerniteCS17,Choromanska2014,Daume2016,Beygelzimer:2009:CPT:1795114.1795121,Morin+al-2005,10.1007/978-3-319-46227-1_32}.
}
Such models can be seen as decision trees allowing inference time complexity linear in the tree height,
that is $\bigO{\log{K}}$ if the tree is (approximately) balanced.
A few models even achieve logarithmic training time, e.g. \cite{DBLP:conf/icml/JerniteCS17}.
Despite having a sublinear time complexity,
these models require storing $\bigO{K}$ classifiers.

Another line of research focused on label-embedding methods \cite{DBLP:conf/nips/BhatiaJKVJ15, Tagami:2017:AAN:3097983.3097987,Weston2011,pmlr-v32-yu14}.
These methods try to exploit label correlations and project the labels onto a low-dimensional space,
reducing training and prediction time.
However, the low-rank assumption usually leads to an accuracy degradation.

Linear methods were also
the focus of some recent works \cite{Babbar2017,En-HsuYen2016,Yen:2017:PPP:3097983.3098083}.
They learn a linear classifier per label and incorporate sparsity assumptions 
or \revised{perform distributed computations}.
However, 
\revised{the training and prediction complexities of} 
these methods do not scale gracefully to datasets 
with a very large number of labels.
\revised{Using a similar post-pruning approach 
and independent (i.e. not joint) learning of the subproblems,
W-LTLS is also capable of exploiting sparsity and learn in parallel.}

\section{Conclusions and Future work}
\label{sec:conclusions}

We propose a new efficient loss-based decoding algorithm that works for any loss function.
Motivated by a general error bound for loss-based decoding \cite{Allwein2001_new},
 we show how to build on the log-time log-space (LTLS) framework \cite{Jasinska2016}
 and employ a more general type of trellis graph architectures.
 Our method offers a tradeoff between multiclass accuracy, model size and prediction time,
 and achieves better multiclass accuracies under logarithmic time and space guarantees.

Many intriguing directions remain uncovered, suggesting a variety of possible future work.
%First, extending our framework to multilabel setting.
One could try to improve the restrictively low minimum code distance of W-LTLS
discussed in \secref{sec:limitations}
\revised{Regularization terms could also be introduced, 
to try and further improve the learned sparse models.}
Moreover, it may be interesting to consider weighing every entry of the coding matrix
(in the spirit of Escalera et al. \cite{Escalera2008}) in the context of trellis graphs.
Finally,
%we believe that
many ideas in this paper
can be extended for other types of graphs and graph algorithms.
%, e.g. shortest paths on other types of DAGs. 

\subsection*{Acknowledgements}
We would like to thank Eyal Bairey for the fruitful discussions.
This research was supported in part by The Israel Science Foundation, grant No. 2030/16. 

\bibliography{library,custom}
\bibliographystyle{plain}

\newpage
\appendix
\section*{SUPPLEMENTARY MATERIAL}
\label{sec:supp_material}

\section{Proof of \thmref{thm:loss_based_decoding}}
\label{supp:proof_thm_loss_based_decoding}
\begin{proof}
For any class $k$ we have,
\begin{align*}
 w\left(P_k\right)  &=  \sum_{j:e_j\in P_k}w_{j}\left(x\right)
   =  \sum_{j:e_j\in P_k}\left[L\left(1\times f_j\left(x\right)\right) + \sum_{j':e_{j'}\in S\left(e_j\right)\backslash \left\{e_j\right\}}L\left((-1)\times f_{j'}(x)\right)\right] \\
  & = \sum_{j:e_j\in P_k}\left[\sum_{j':e_{j'}\in S(e_j)}L\left(M_{k,{j'}}\times f_{j'}(x)\right)\right]
   = \sum_{j=1}^{\ell} L\left(M_{k,j}f_j(x)\right) ~.
\end{align*}
The third equality in the proof follows from the path codeword representation, see for example \figref{fig:path-to-code}.
\QED
\end{proof}

\section{Proof of \lemref{lemma:ltls_decoding}}
\label{supp:proof_lemma_ltls_decoding}
\begin{proof}
Indeed, for LTLS we have,
\begin{align*}
& \arg\max_{k}w\left(P_{k}\right) = \arg\max_{k}\left\{ \sum_{j:e_j\in P_k}w_j\left(x\right)\right\}
= \arg\max_{k}\left\{ \sum_{j:M_{k,j}=1}f_{j}\left(x\right)\right\} \\
= &\arg\max_{k}\left\{ 2\sum_{j:M_{k,j}=1}f_{j}\left(x\right)-\underbrace{\sum_{j}f_{j}\left(x\right)}_{\text{constant in }k}\right\}
=\arg\max_{k}\left\{ {\sum_{j:M_{k,j}=1}f_{j}\left(x\right)-\sum_{j:M_{k,j}=-1}f_{j}\left(x\right)} \right\} \\
=&\arg\max_{k}\left\{ \sum_{j}M_{k,j}f_{j}\left(x\right)\right\}
 = \arg\max_{k}\left\{ \underbrace{\cancel{\sum_{j}M_{k,j}^2}}_{=\ell} + \underbrace{\cancel{\sum_{j}f_{j}^2(x)}}_{\text{constant in }k} - \sum_{j}\left(M_{k,j}-f_j\left(x\right)\right)^2 \right\} \\
 = &\arg\min_{k}\left\{ \sum_{j}\left(M_{k,j}-f_j\left(x\right)\right)^2\underbrace{\left(M_{k,j}\right)^2}_{=1} \right\}
 = \arg\min_{k}\left\{ \sum_{j}\left(1-M_{k,j}f_j\left(x\right)\right)^2 \right\} \\
 = &\arg\min_{k}\left\{ \sum_{j}L_{sq}\left(M_{k,j}f_j\left(x\right)\right) \right\} ~.
\end{align*}
Note that we used the fact that $M_{k,j}\in \{-1,+1\}$.
\QED
\end{proof}

\newpage
\section{Extension to arbitrary K}

\subsection{Graph construction}
\label{subsec:graphConstruction}

The following algorithm construct graphs with any number of paths,
and is not restricted to powers of $2$. See \figref{fig:arbitraryK_fig} for examples.

\begin{algorithm}
    \SetKwInOut{Input}{Input}
    \SetKwInOut{Output}{Output}

    \Input{Number of labels $K$ and slice width $b$.}
    Convert $K$ to a base-$b$ representation.

    Store the \emph{reverse} representation in an array $A$ of size $\left\lfloor {\log_b{K}}\right\rfloor + 1$
    (i.e. the least significant $b$-ary \emph{digit} is in $A\left[0\right]$).

    Build a trellis graph with $\left\lfloor {\log_b{K}}\right\rfloor+1$ inner slices, each with $b$ vertices.

    Add a source vertex $s$ and connect it to the vertices of the first inner slice.

    Add a sink vertex $t$.

    For every inner slice $i=0 \dots \left\lfloor {\log_b{K}}\right\rfloor$,
    connect $A\left[i\right]$ vertices of the slice to the sink.

    Delete any vertices from which the sink is unreachable.

    \caption{Graph construction with an arbitrary $K$}
    \label{alg:graphConstruction}
\end{algorithm}

\begin{figure*}[h]
\vskip -0.2in
\begin{center}
\includegraphics[width=\linewidth]{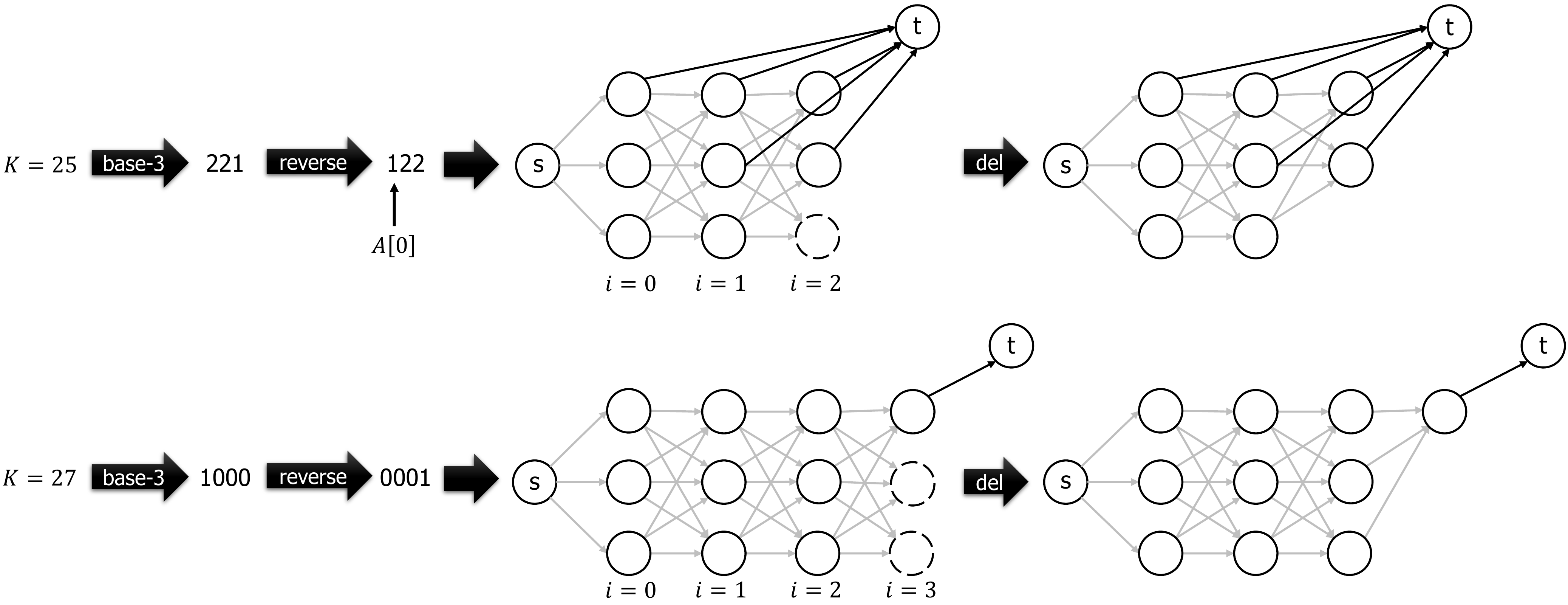}
\caption{An illustration of the graph construction algorithm for two different values of $K$, using $b=3$. The darker edges are the ones created on stage 6 of \algoref{alg:graphConstruction}.
Note that in the upper construction,
the minimum distance between any two path is actually $\rho=3$
and not $\rho=4$ like previously noted.
This sometimes holds minor accuracy implications.
}
\label{fig:arbitraryK_fig}
\end{center}
%\vskip -0.2in
\end{figure*}

Below we show that this construction indeed produces a graph with exactly $K$ paths from source to sink. We start with a technical lemma.

\begin{lemma}
\label{lemma:pathsToSlice}
Let $v\in V\setminus{\left\{t\right\}}$ be a vertex in the $i$th inner slice
(where $i\in\left\{0 \dots \left\lfloor {\log_b{K}}\right\rfloor \right\}$),
and let $\mathcal{N}\left(v\right)$ be the number of paths from the source to $v$.
Then $\mathcal{N}\left(v\right)=b^i$.
\end{lemma}

\begin{proof}
We show this by induction on the slice index $i=0 \dots \left\lfloor {\log_b{K}}\right\rfloor$.
In the base case $i=0$ and $v$ is in the first inner slice (see \figref{fig:arbitraryK_fig}).
It has only one incoming edge from the source $s$,
hence the statement holds, $\mathcal{N}\left(v\right)=b^0=1$.

Next, denote $in\left(v\right)\triangleq \left\{ u: \exists\left(u,v\right)\in E \right\}$ and
suppose the theorem holds for all slices until the $k$th slice. Let $v$ be a vertex in the $\left(k+1\right)$th slice.
Following the graph construction, $v$ has $b$ incoming edges from the vertices in $in\left(v\right)$, all of which are in the $k$th slice.
According to the inductive hypothesis, $\forall{u\in in\left(v\right)}: \mathcal{N}\left(u\right)=b^k$.
Every path from $s$ to $u\in in\left(v\right)$,
could be extended to a path from $s$ to $v$ by concatenating the corresponding edge,
and so we get that
$\mathcal{N}\left(v\right)=\sum_{u\in in\left(v\right)}{\mathcal{N}\left(u\right)}
=\sum_{u\in in\left(v\right)}{b^k}
=b\cdot{b^k}=b^{k+1}$.
\QED
\end{proof}

\newpage

\begin{theorem}
The number of paths from the source to the sink is $K$, i.e. $\mathcal{N}\left(t\right)=K$.
\end{theorem}

\begin{proof}
  Let $A\in{\left\{{0 \dots b-1}\right\}^{\left\lfloor {\log_b{K}}\right\rfloor+1}}$
  be the array of the reverse base-$b$ representation of $K$.
  Using the decomposition of the $b$-ary representation,
  i.e. $K=\sum_{i=0}^{\left\lfloor {\log_b{K}}\right\rfloor}{A\left[i\right]\cdot b^i}$, and \lemref{lemma:pathsToSlice}, we get:

  \begin{align*}
      \mathcal{N}\left(t\right) & =
      \sum_{v\in in\left(t\right)}{\mathcal{N}\left(v\right)} =
      \sum_{i=0}^{\left\lfloor {\log_b{K}}\right\rfloor} {
        \left[\sum_{
            v\in\text{slice}_i \cap in\left(t\right)
        }{\underbrace{\mathcal{N}\left(v\right)}_{=b^i}}\right]
      } =
      \sum_{i=0}^{\left\lfloor {\log_b{K}}\right\rfloor} {
        \underbrace{\left|{\text{slice}_i\cap in\left(t\right)}\right|}_{=A\left[i\right]}\cdot{b^i}
      }\\
      & =
      \sum_{i=0}^{\left\lfloor {\log_b{K}}\right\rfloor} {
        A\left[i\right]\cdot{b^i}
      } = K ~.
  \end{align*}
\QED
\end{proof}

\subsection{Loss-based decoding generalization}
\label{subsec:arbitraryLBDGeneralization}

We now show how to adjust the generalization of loss-based decoding
for graphs with an arbitrary number of paths $K$,
constructed using \algoref{alg:graphConstruction}.

The idea of the reduction in \eqref{w_i} is that going through an edge $e_j$ inflicts the loss of turning its corresponding bit on (i.e. $L\left(1\times f_j\left(x\right)\right)$),
but also the loss of turning off the bits corresponding to other edges between its slices
(i.e. $\sum_{j':e_{j'}\in S\left(e_j\right)\setminus\left\{e_j\right\}}{L\left(\left(-1\right)\times f_{j'}\left(x\right)\right)}$),
which cannot coappear with $e_j$ in the same path (i.e. a codeword).

The only change in the general case is that
an edge $e_j=\left(u_j,t\right)$ that is connected to the sink $t$
cannot coappear with \emph{any}
other edge outgoing from a vertex in the same vertical slice as $u_j$,
\textbf{or} that is reachable from $u_j$.

Let $\delta\left(v\right)$ be the shortest distance from the source to $v$ (in terms of number of edges).

We define an updated $S$ function (which is a generalization of the one defined in \secref{sec:loss_based_decoding}),
where for every edge ${e_j=\left(u_j,v_j\right) \in E}$ we set:
\begin{align}
S\left(e_j\right)=
\begin{cases}
    \left\{\left(u,u'\right): \delta\left(u\right)=\delta\left(u_j\right)\right\} & v_j \neq t\\
    \left\{\left(u,u'\right): \delta\left(u\right)\ge\delta\left(u_j\right)\right\}  & v_j = t
    %\cup
    %\left\{{
    %    \left(u,t\right): \left(u,t\right)\in E
    %}\right\} & v_j = t
\end{cases}~,
\label{new_s_function}
\end{align}

and the weights are set as in \eqref{w_i} but with the new $S$ function defined in \eqref{new_s_function},
\begin{align}
w_j\left(x\right)= L\left(1\times f_j(x)\right) +
\sum_{j{'}:e_{j'}\in S\left(e_j\right)\backslash \left\{e_j\right\}} L\left((-1)\times f_{j'}(x)\right)~.
\label{general_w}
\end{align}

For example, in the following figure we have,

\begin{minipage}[t]{.3\textwidth}
\vskip -0.2in
\begin{align*}
    S\left(e_6\right)&=\left\{e_6,e_7,e_8,e_9\right\}\\
    S\left(e_2\right)&=\left\{e_2,e_3,e_4,e_5,e_{13}\right\}\\
\end{align*}
\end{minipage}
\begin{minipage}[t]{.3\textwidth}
\vskip -0.2in
\begin{align*}
    S\left(e_{11}\right)&=\left\{e_{10},e_{11}\right\}\\
    S\left(e_{12}\right)&=\left\{e_{12}\right\}\\
\end{align*}
\end{minipage}
\begin{minipage}[t]{.3\textwidth}
\vskip -0.2in
\begin{align*}
    S\left(e_{13}\right)&=\left\{e_2,\dots,e_{13}\right\}~.
\end{align*}
\end{minipage}

\begin{figure*}[b]
\vskip -0.3in
\begin{center}
\label{fig:lbdExplanation}
\includegraphics[width=.5\linewidth]{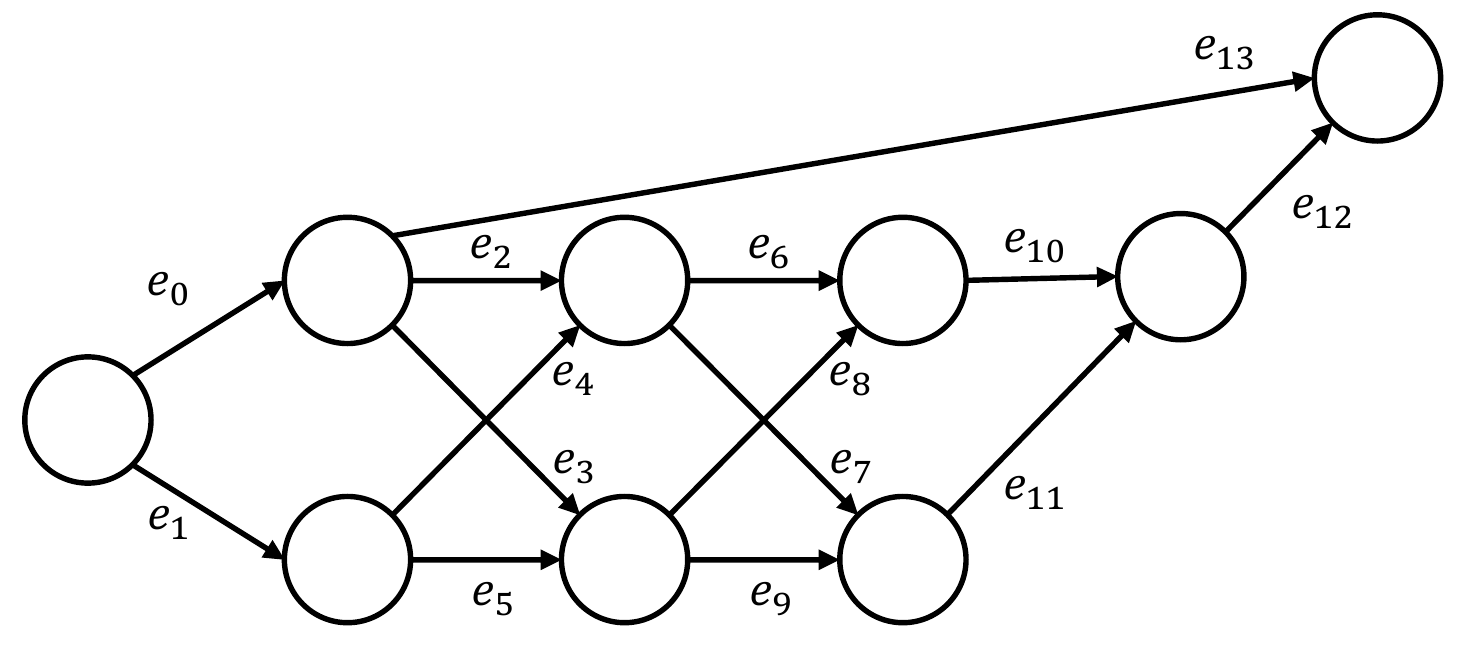}
\caption{An illustration of a graph with $K=9$ and $b=2$.}
\end{center}
\end{figure*}

\newpage

Below we show that \thmref{thm:loss_based_decoding} is correct for any $K$.
Let $P$ be a path on the trellis graph from the source $s$ to the sink $t$.
We start will the following lemma.
\begin{lemma}
\label{lem:incrementalDeltas}
    Let $e_q=\left(u_q,t\right)$ be the last edge in $P$.
    For every edge $e_j=\left(u_j,v_j\right)\in P\setminus\left\{e_q\right\}$ we have $\delta\left(v_j\right)= \delta\left(u_j\right)+1$.
\end{lemma}

\begin{proof}
 Following immediately from the graph construction --
 other than edges to the sink,
 there are only edges between adjacent slices (without cycles).
 Therefore, any vertex $v$ in a vertical slice $i$
 (where the leftmost slice containing $s$ is the first one, i.e. $i=0$)
 holds $\delta\left(v\right)=i$,
 and every edge $e_j=\left(u_j,v_j\right)\in P\setminus\left\{e_q\right\}$
 holds $\delta\left(v_j\right)= \delta\left(u_j\right)+1$.
  \QED
\end{proof}

The next corollaries follow immediately.
\begin{corollary}
    \label{cor:uniqueVerticesAlongPath}
    For every two different vertices $u,v\in V\setminus\left\{t\right\}$ along $P$
    we have $\delta\left(u\right) \neq \delta\left(v\right)$.
\end{corollary}
\begin{corollary}
    \label{cor:uqIsMaximal}
    Let $e_q=\left(u_q,t\right)$ be the last edge in $P$.
    For every edge $e_j=\left(u_j,v_j\right)\in P\setminus\left\{e_q\right\}$ we have $\delta\left(u_j\right)<\delta\left(u_q\right)$.
\end{corollary}

Clearly, since the graph is directed and acyclic, we get the next lemma.

\begin{lemma}
\label{lem:vertexAppearsOnce}
Every vertex in $V$ can have at most one incoming edge and one outgoing edge in $P$.
\end{lemma}

By using \lemref{lem:vertexAppearsOnce}, \corref{cor:uniqueVerticesAlongPath} and \lemref{lem:incrementalDeltas} we have the next corollary.
\begin{corollary}
\label{cor:edgesOfSlice}
Set an edge $e_j=\left(u_j,v_j\right)$ along the path $P$.
Then, $\forall{e_{j{'}} \in S\left(e_j\right)\setminus\left\{e_j\right\}}$ we have
$e_{j{'}}\notin P$.
\end{corollary}
\begin{proof}
  Let $e_{j{'}}=\left(u_{j{'}},v_{j{'}}\right)$ be any edge in $S\left(e_j\right)\setminus\left\{e_j\right\}$,
  and assume $e_{j{'}}\in P$.
  We consider two cases:
  \begin{enumerate}[leftmargin=*]
    \item $v_j \neq t$:\\
  By the definition of $S\left(e_j\right)$, we get
  $\delta\left(u_{j{'}}\right) = \delta\left(u_{j}\right)$.
  By negating \corref{cor:uniqueVerticesAlongPath} we get $u_{j{'}}=u_j$.
  Therefore $e_j,e_{j{'}}$ are two edges in $P$ outgoing from the same vertex $u_j$,
  in contradiction to \lemref{lem:vertexAppearsOnce}.
    \item $v_j=t$:\\
    By definition, $e_j\neq e_{j{'}}$.
    By \lemref{lem:vertexAppearsOnce} we get that $u_j\neq u_{j'}$.
    Also, since $t$ is a sink it has no outgoing edges, thus $u_{j{'}}\neq t$.

    By the definition of $S\left(e_j\right)$, we get
    $\delta\left(u_{j{'}}\right) \ge \delta\left(u_{j}\right)$.
    Since $u_{j{'}}\neq u_{j}$, we can use \corref{cor:uniqueVerticesAlongPath}
    to rule out equivalence and get
    $\delta\left(u_{j{'}}\right) > \delta\left(u_{j}\right)$.
      Following \lemref{lem:incrementalDeltas}, $u_{j{'}}$ must appear later than $u_j$ in $P$.
      However, $t$ is the only vertex in $P$ to appear after $u_j$ ($t$ is a sink),
      and since $u_{j{'}}\neq t$, we get a contradiction.
    \end{enumerate}  \QED
\end{proof}

We conclude with the main result of this section, which states that \thmref{thm:loss_based_decoding} is correct for any $K$.
\begin{theorem}
Following the notations of \thmref{thm:loss_based_decoding}, assume the weights of the edges are calculated as in Eq. \eqref{general_w} with the $S$ function defined in \eqref{new_s_function}.
Then, the weight of any path $P_k$ corresponding to class $k$ equals to the loss suffered by predicting class $k$,
i.e. $w(P_k)=\sum_{j=1}^{\ell} L\left(M_{k,j}f_j(x)\right)$.
\end{theorem}

\begin{proof}

For any class $k$, we denote the last edge in $P_k$ by $e_q=\left(u_{q},t\right)$.
    We have,
\begin{align*}
 w\left(P_k\right)  &=
    \sum_{j:e_j\in P_k}{ w_{j}\left(x\right) }
   \\
   & =
   \sum_{j:e_j\in P_k}\left[L\left(\underbrace{1}_{=M_{k,j}}\times f_j\left(x\right)\right) + \sum_{j':e_{j'}\in S\left(e_j\right)\backslash \left\{e_j\right\}}L\left(\underbrace{(-1)}_{
    \substack{
        =M_{k,j{'}} \\
        \text{(\corref{cor:edgesOfSlice})}
    }
   }\times f_{j'}(x)\right)\right] \\
   & =
   \sum_{j:e_j\in P_k}\left[\sum_{j':e_{j'}\in S\left(e_j\right)}L\left(M_{k,j{'}}\times f_{j'}(x)\right)\right] \\
   & =
   \sum_{j:e_{j}\in S(e_q)}L\left(M_{k,{j}}\times f_{j}(x)\right) +
   \sum_{j:e_j\in P_k\setminus\left\{e_{q}\right\}}
        \left[\sum_{j':e_{j'}\in S(e_j)}L\left(M_{k,{j'}}\times f_{j'}(x)\right)\right] \\
   & =
   \sum_{\substack{
    j: e_{j}=\left(u_{j},v_{j}\right), \\
            \delta\left(u_{j}\right)\ge\delta\left(u_q\right)
   }}
        L\left(M_{k,j}\times f_{j}(x)\right)
    +
   \underbrace{\sum_{j:e_j=\left(u_{j},v_{j}\right)\in P_k\setminus\left\{e_{q}\right\}}
        \left[\sum_{
            \substack{
                j{'}: e_{j{'}}=\left(u_{j{'}},v_{j{'}}\right), \\
                \delta\left(u_{j{'}}\right)=\delta\left(u_j\right)
            }
        }
        L\left(M_{k,{j'}}\times f_{j'}(x)\right)\right]}_{
            \text{ (\corref{cor:uqIsMaximal}) ~~~~}
                =\sum_{
                    \substack{
                        j: e_{j}=\left(u_{j},v_{j}\right), \\
                        \delta\left(u_{j}\right)<\delta\left(u_q\right)
                    }
                }
                L\left(M_{k,j}\times f_j(x)\right)
        } \\
    &= \sum_{\substack{
            j: e_{j}=\left(u_{j},v_{j}\right), \\
                    \delta\left(u_{j}\right)\ge\delta\left(u_q\right)
           }
        }
        L\left(M_{k,j}\times f_{j}(x)\right) +
        \sum_{
            \substack{
            j: e_{j}=\left(u_{j},v_{j}\right), \\
                    \delta\left(u_{j}\right)<\delta\left(u_q\right)
           }
        }
        L\left(M_{k,{j}}\times f_{j}(x)\right)
         \\
     &= \sum_{j: e_{j}\in E} L\left(M_{k,{j}}\times f_{j}(x)\right) \\
     & = \sum_{j=1}^{\ell} L\left(M_{k,j}f_j(x)\right) ~.\\
\end{align*}
\QED
\end{proof}

\newpage
\section{Details for complexity analysis in \secref{complexityAnalysis}}
\label{supp:complexity}

W-LTLS requires training and storing a binary function or model for every edge.
Hence, we first turn to analyze the number of edges.

\begin{lemma}
\vskip 0.1in
\label{lemma:innerSlicesVertices}
The number of vertices in the inner slices is at most $\left({\left\lfloor {\log_b{K}}\right\rfloor+1}\right)b$.
\end{lemma}

\begin{proof}
Follows immediately from the construction in \appref{subsec:graphConstruction}.
\QED
\end{proof}

\begin{corollary}
\label{edgesNumberBound}
The number of edges is upper bounded:

$$\left|E\right| \le
\left(b+1\right)\left({\left\lfloor {\log_b{K}}\right\rfloor+1}\right)b + b =
\bigO{\frac{b^2}{\log{b}}\log{K}}~.$$
\end{corollary}

\begin{proof}
Each vertex in the inner slices can have at most $b+1$ outgoing edges.
We use \lemref{lemma:innerSlicesVertices} and count also the $b$ edges outgoing from the source.
\QED
\end{proof}

%\noindent\rule{\textwidth}{1pt}

For most linear classifiers (with $d$ parameters each) we get a total \emph{model size complexity} of $\bigO{d\frac{b^2}{\log{b}}\log{K}}$.

Inference consists of four steps:
\begin{enumerate}
  \item Computing the value (margin) of all binary functions on the input $x$. This requires $\bigO{d\left|E\right|}=\bigO{d\frac{b^2}{\log{b}}\log{K}}$ time.
  \item Computing the edge weights $\left\{w_{i}\left(x\right)\right\}_{i=1}^{\ell}$ as explained in \secref{sec:loss_based_decoding}.
    This can be performed in $\bigO{\left|V\right|+\left|E\right|}=\bigO{\frac{b^2}{\log{b}}\log{K}}$ time using a simple dynamic programming algorithm (e.g. implementing back recursion).
  \item Finding the shortest path in the trellis graph with respect to $\left\{w_i\left(x\right) \right\}_{i=1}^{\ell}$ using the Viterbi algorithm in $\bigO{\left|V\right|+\left|E\right|}=\bigO{\frac{b^2}{\log{b}}\log{K}}$ time.
  \item Decoding the shortest path to a class.
        As explained in \secref{pathAssignment}, 
        the inference requires a mapping function from path to code. 
        Using data structures such as a binary tree, this can be performed in a $\bigO{\left|E\right|}=\bigO{\frac{b^2}{\log{b}}\log{K}}$ time complexity.

\end{enumerate}

We get that the total \emph{inference time complexity} is $\bigO{d\left|E\right|}=\bigO{d\frac{b^2}{\log{b}}\log{K}}$.

\newpage

\section{Experiments appendix}

\subsection{Datasets used in experiments}
\label{supp:datasets}

\begin{table}[h]
%\vskip 0.15in
\begin{center}
\begin{small}
\begin{sc}
\begin{tabular}{lrrr}
\toprule
Dataset     & Classes & Features    & Train Samples   \\
            & $K$     & $d$        & $m$        \\
\midrule
sector      & 105       & 55,197    & 7,793     \\
aloi\_bin   & 1,000     & 636,911   & 90,000    \\
imageNet    & 1,000     & 1,000     & 1,125,264 \\
Dmoz        & 11,947    & 833,484   & 335,068   \\
LSHTC1      & 12,294    & 1,199,856 & 83,805    \\
\bottomrule
\end{tabular}
\end{sc}
\end{small}
\end{center}
\caption{Datasets used in the experiments.}
%\vskip -0.1in
\label{datasets_table}
\end{table}

\subsection{Average binary training loss}
\label{supp:avgBinaryLoss}

We validate our hypothesis that wider graphs lead to easier
binary problems. In the first row of \figref{fig:avgBinaryLossVsPredictorsNumber} we
plot the average binary training loss $\varepsilon$ as a function of model size.
The average is both over the induced binary
subproblems and over the five runs.
%The error bars indicate one standard deviation.

In all datasets we observe a decrease of the average error as the slice width $b$ grows.
 The decrease is sharp for low values of $b$ and then
practically almost converges (to zero).
% We also observe that the standard deviation also decreases as a function of $\ell$,
%implying that the induced binary subproblems become easier jointly --
%not only in their expectation, but also as a collection.
 These plots validate our claim -- as the
subproblems become more unbalanced they also become easier.
%We next show that in turn, it induces an improvement in the multiclass training error.

\begin{figure*}[h]
\vskip 0.2in
\begin{center}
\includegraphics[width=.196\linewidth]{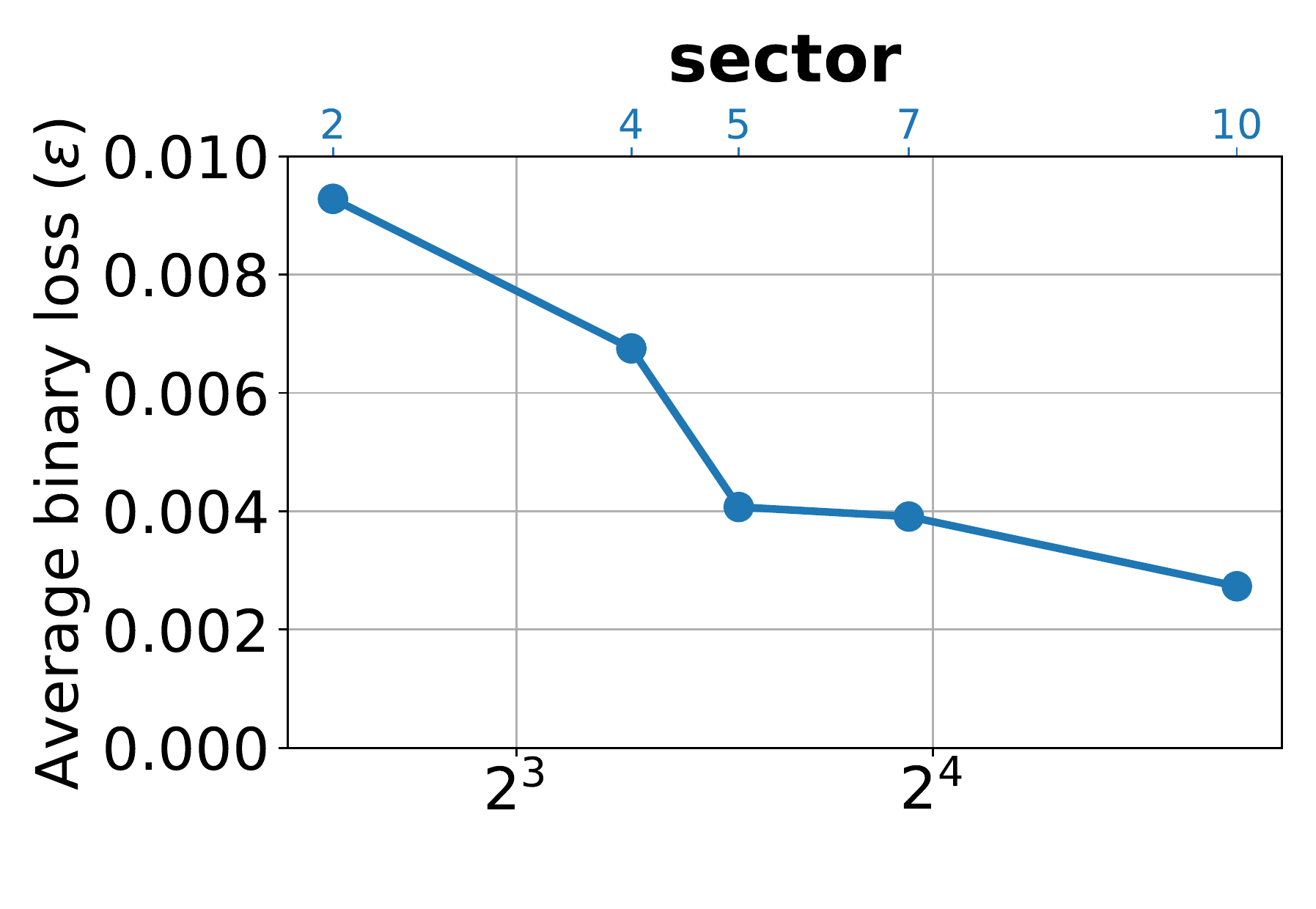}
\includegraphics[width=.191\linewidth]{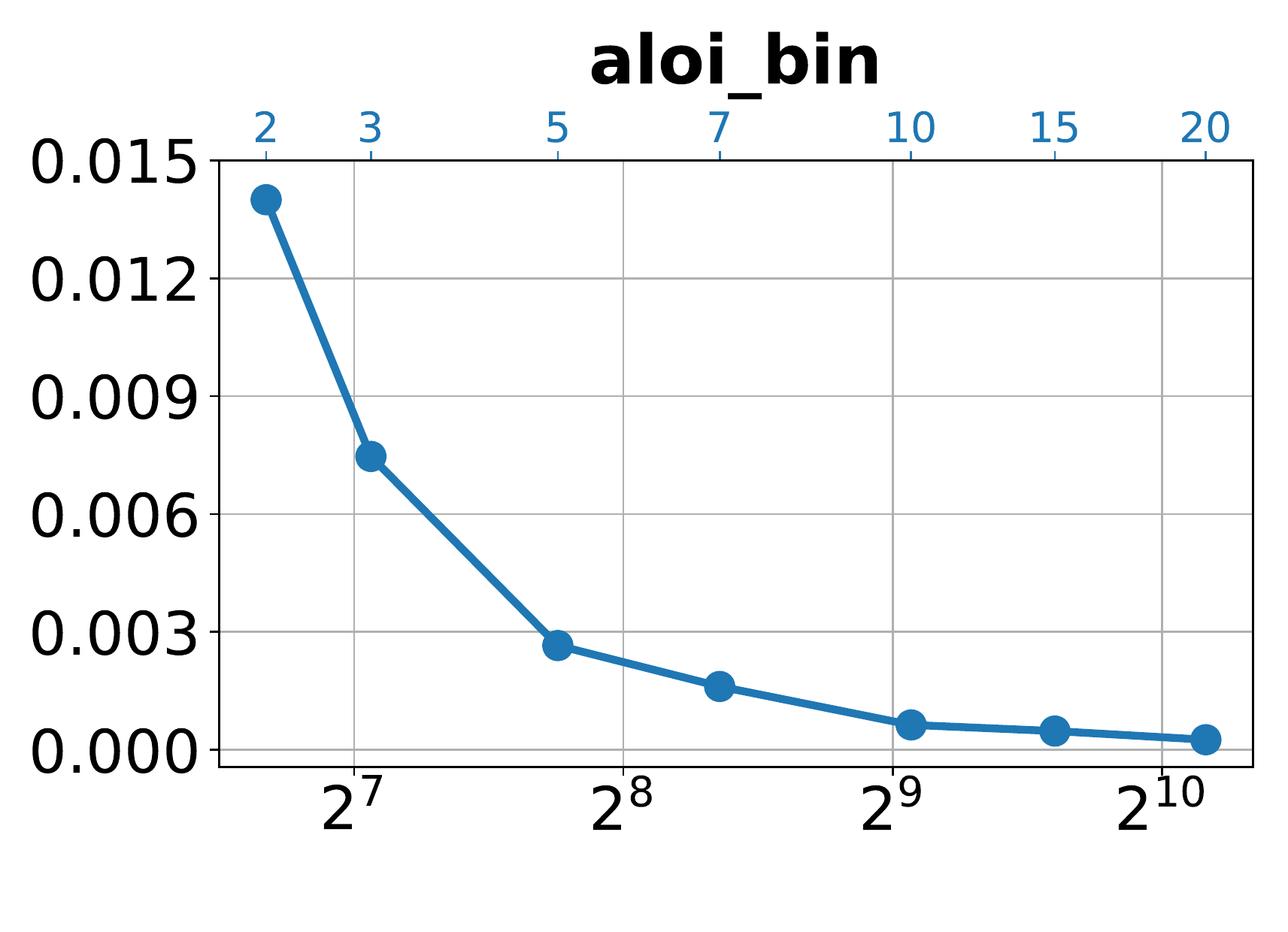}
\includegraphics[width=.191\linewidth]{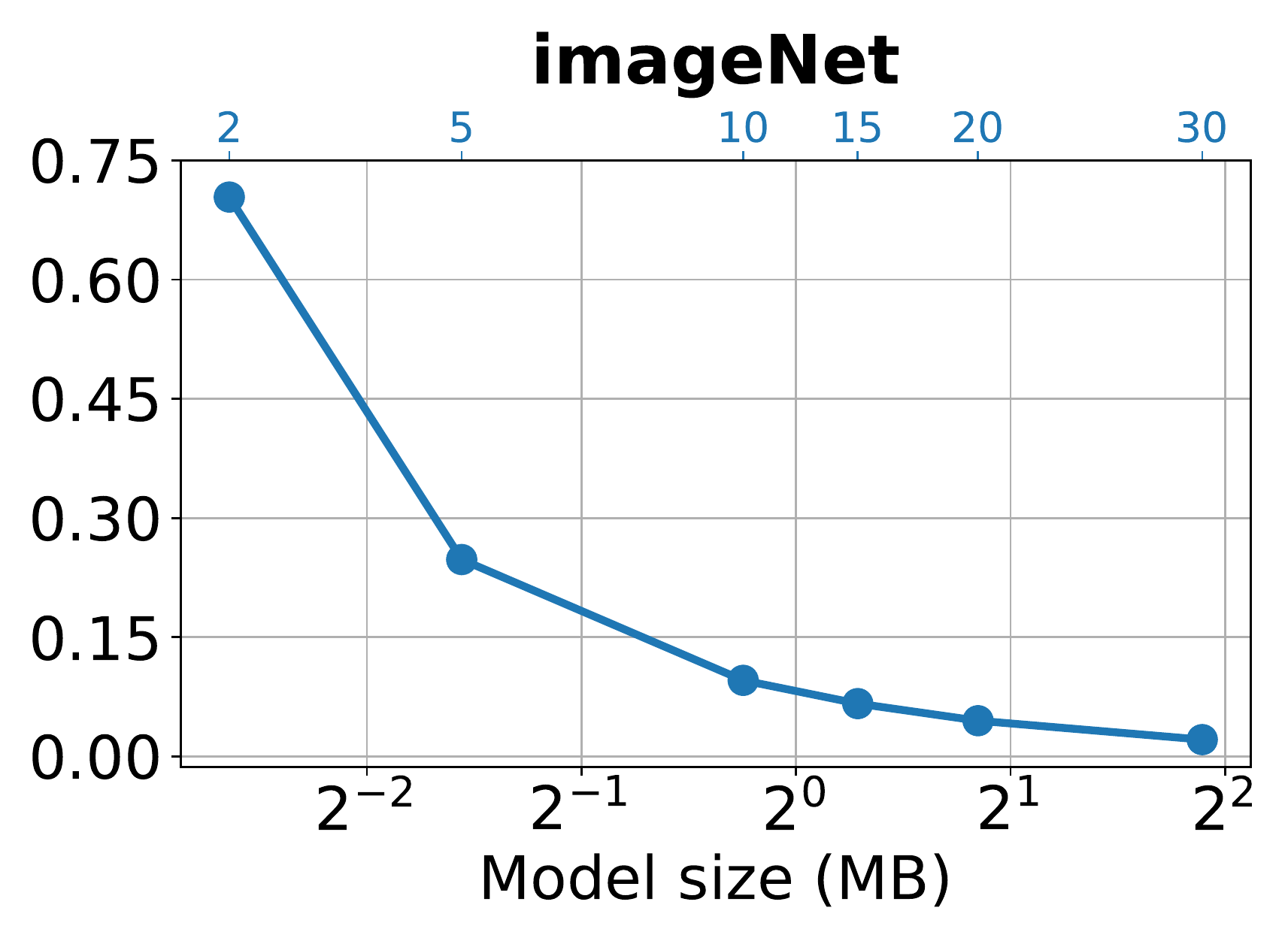}
\includegraphics[width=.191\linewidth]{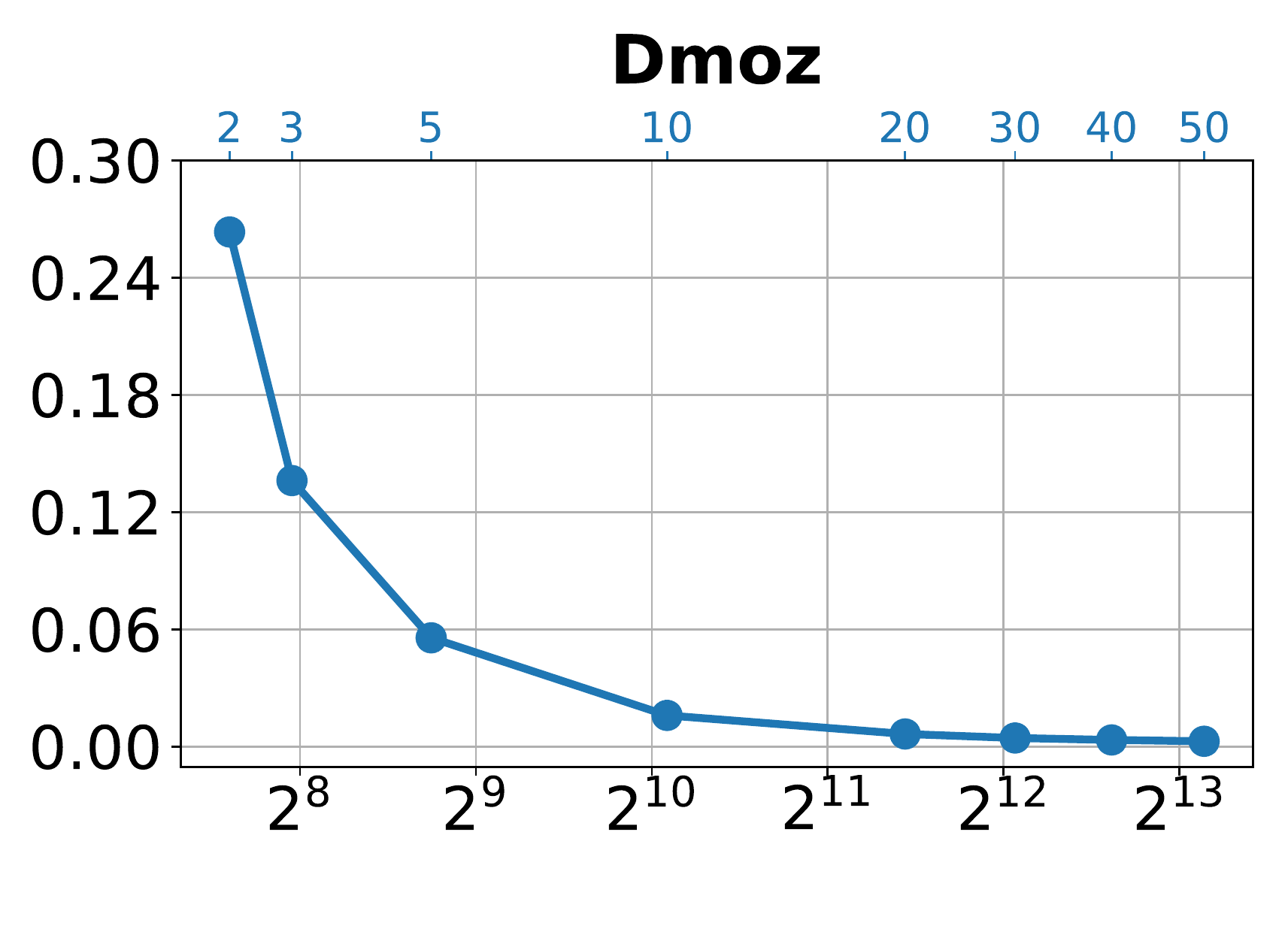}
\includegraphics[width=.191\linewidth]{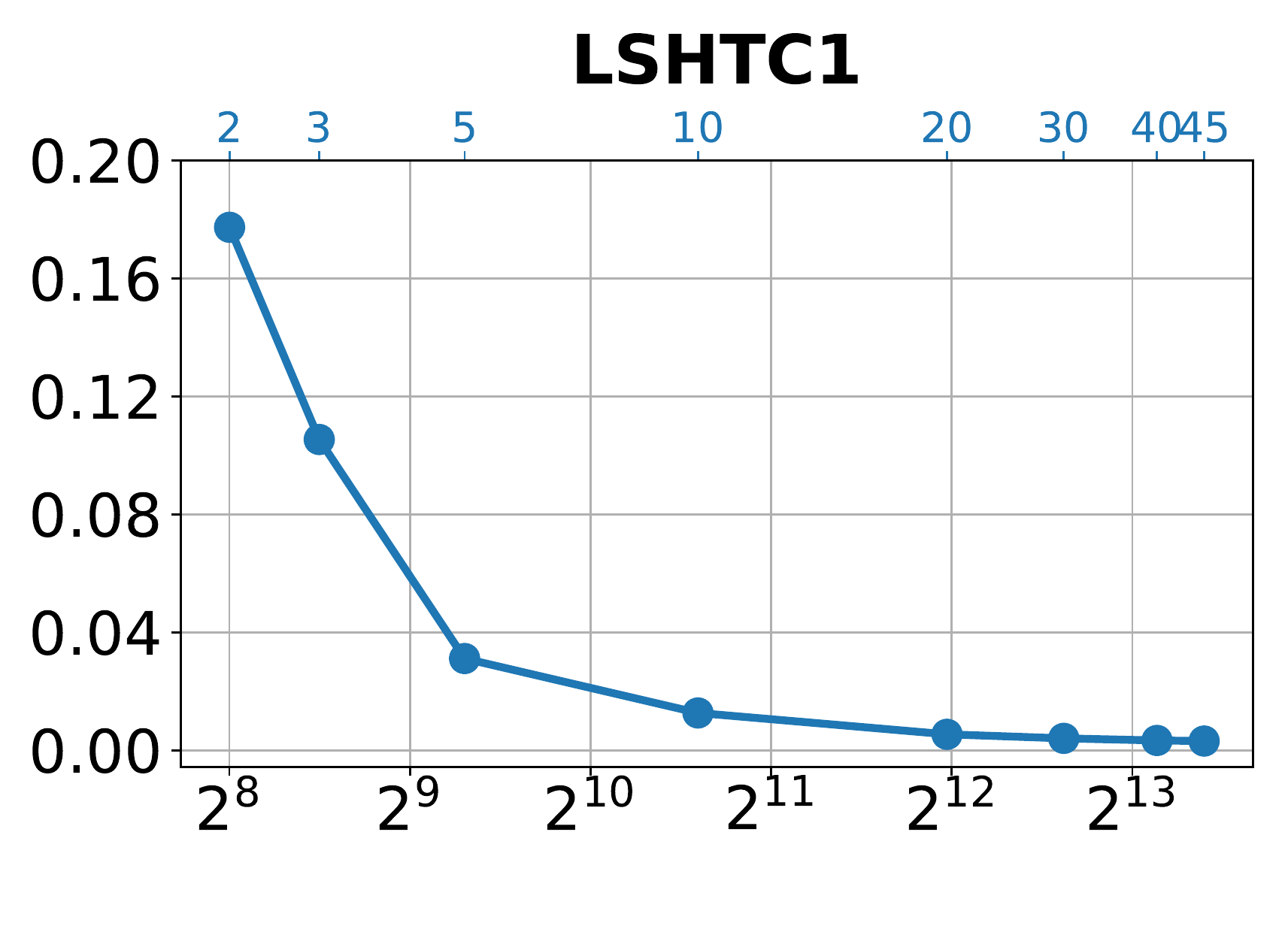}

\vskip -0.03in

\includegraphics[width=.196\linewidth]{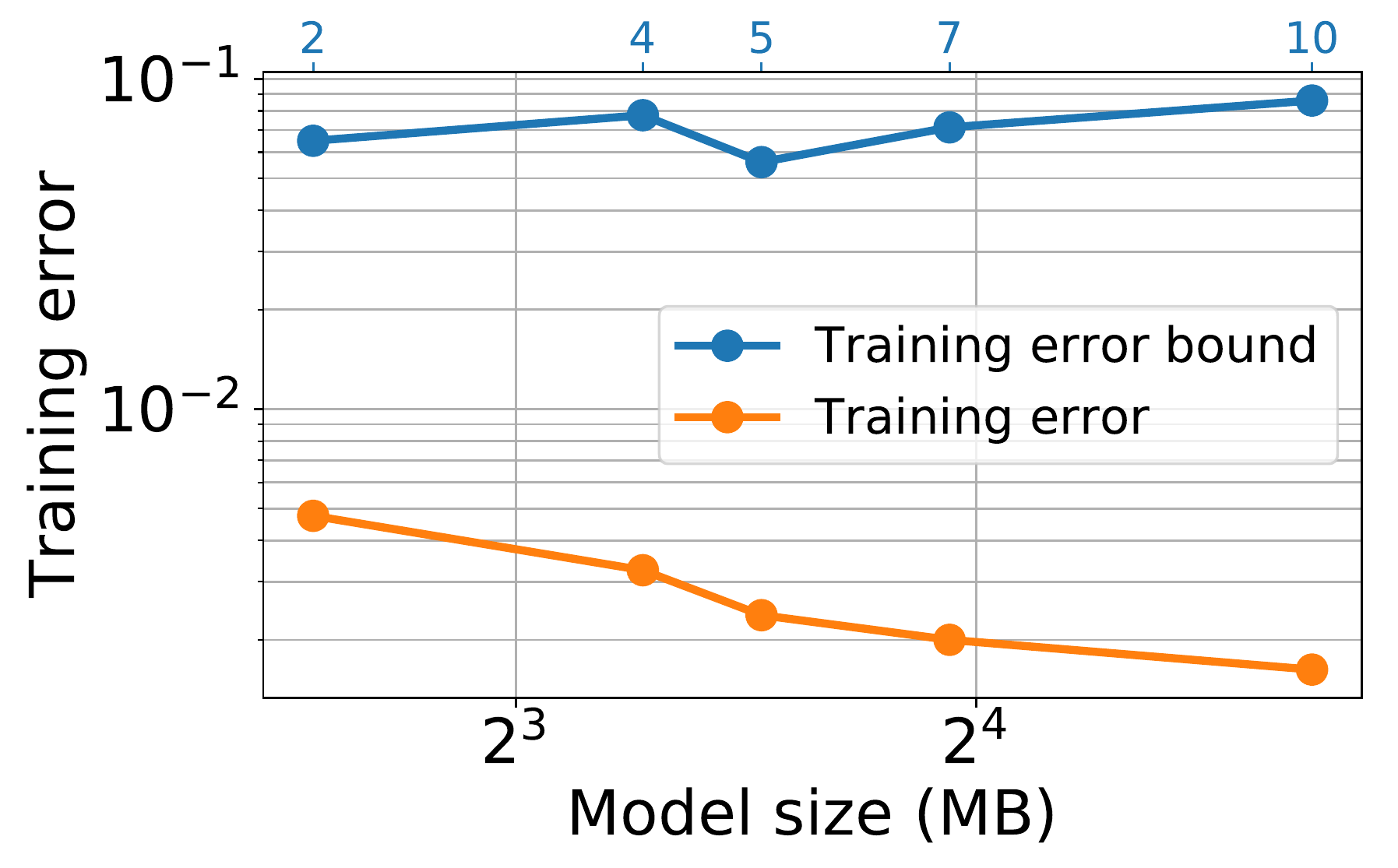}
\includegraphics[width=.191\linewidth]{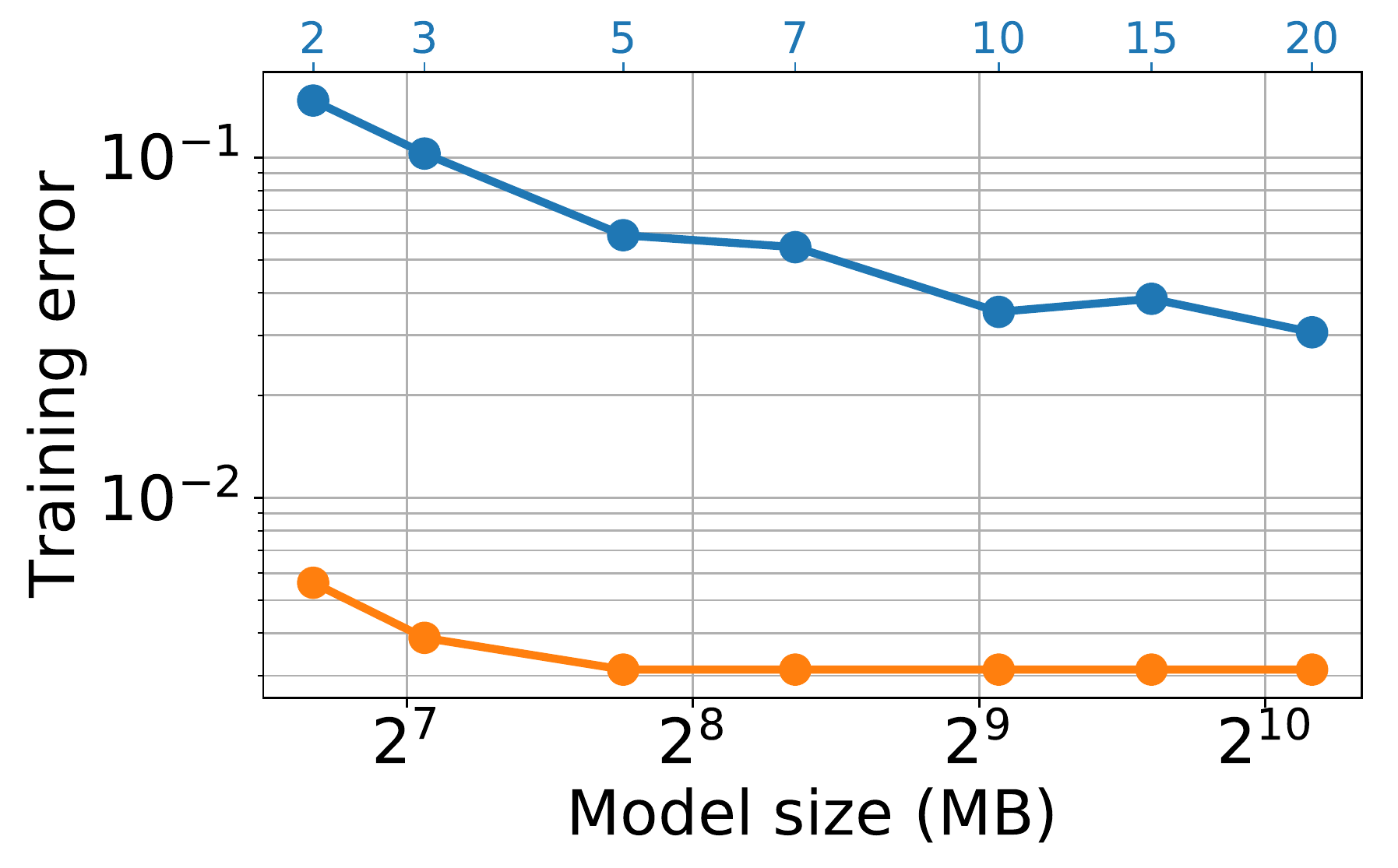}
\includegraphics[width=.191\linewidth]{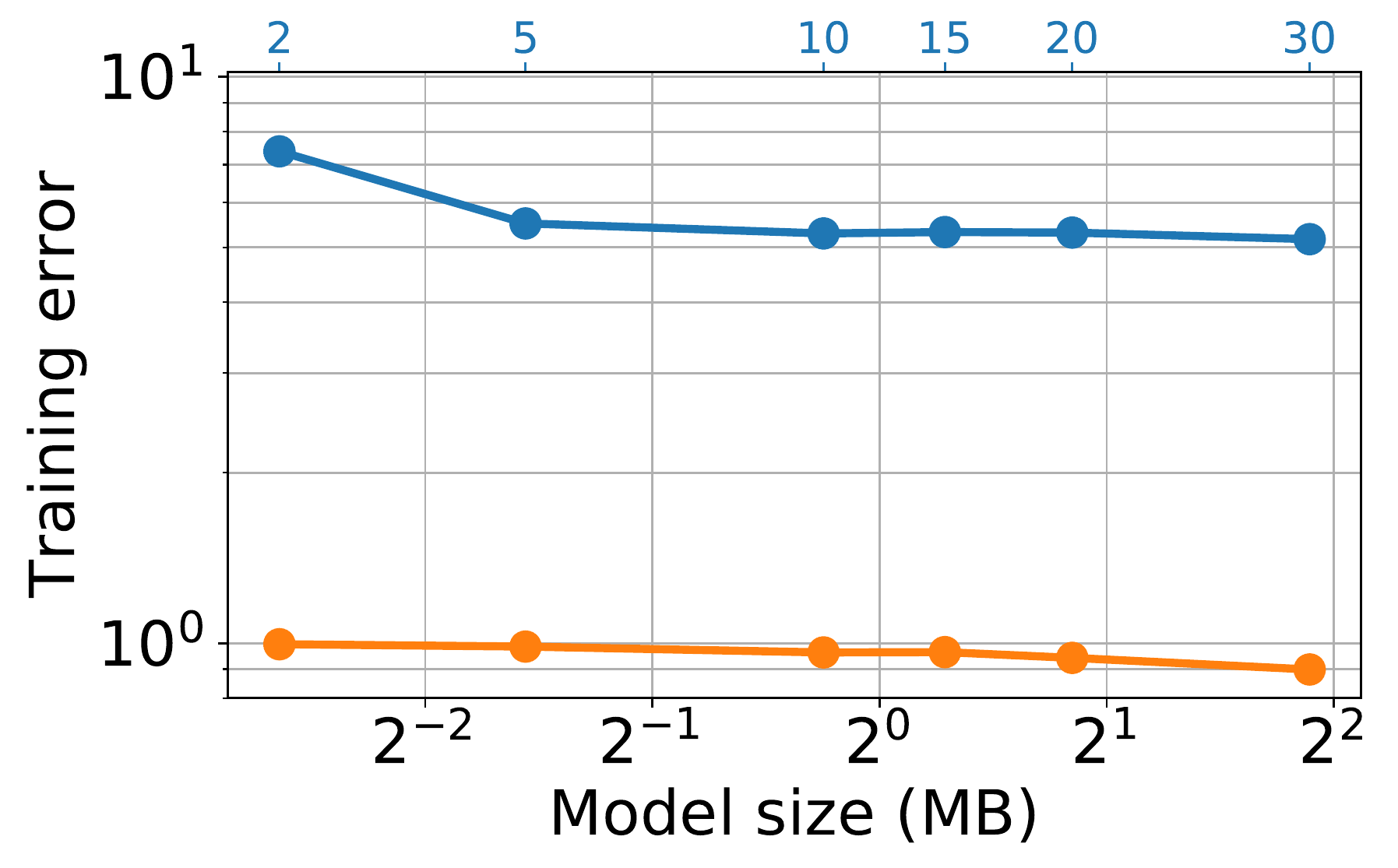}
\includegraphics[width=.191\linewidth]{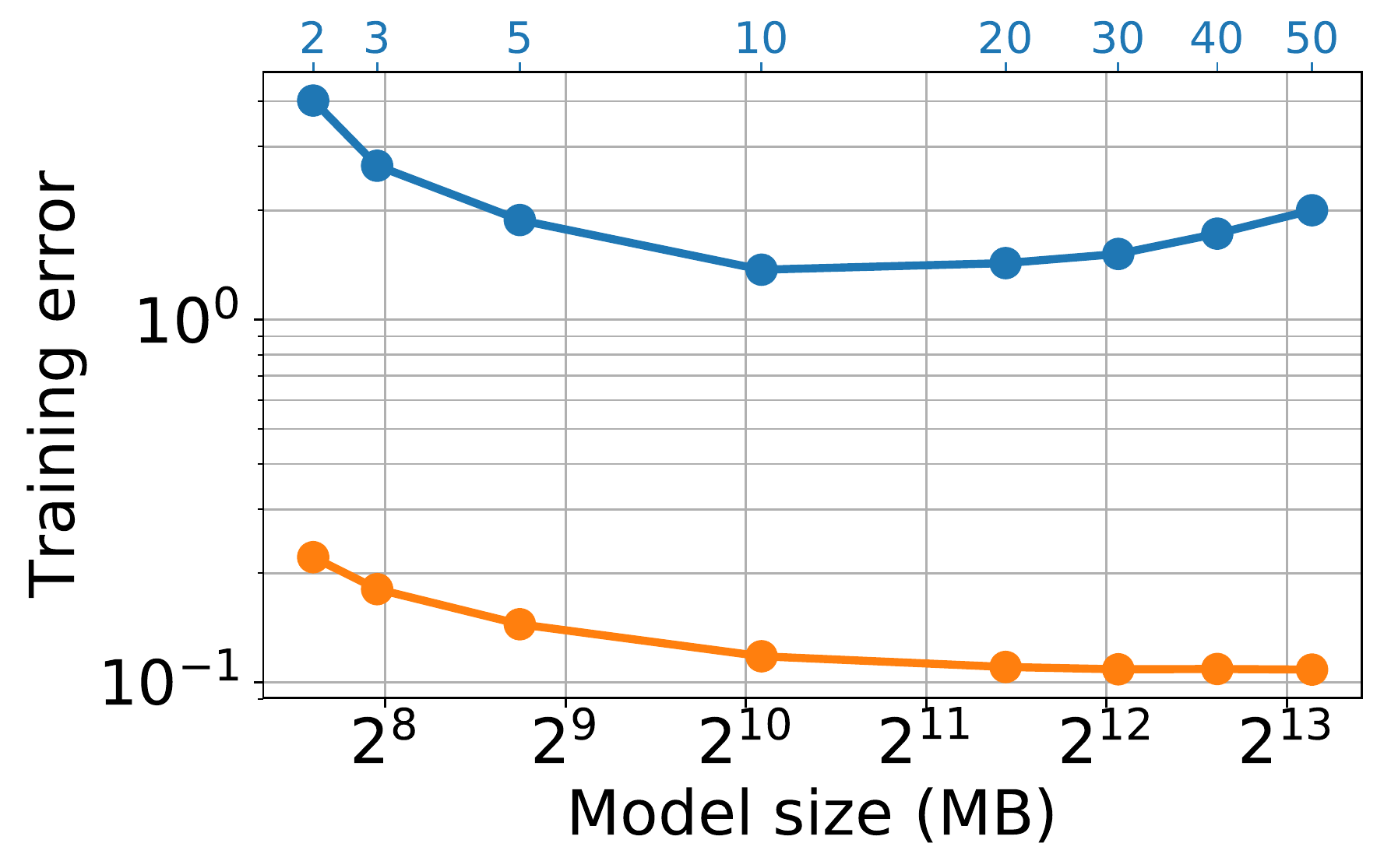}
\includegraphics[width=.191\linewidth]{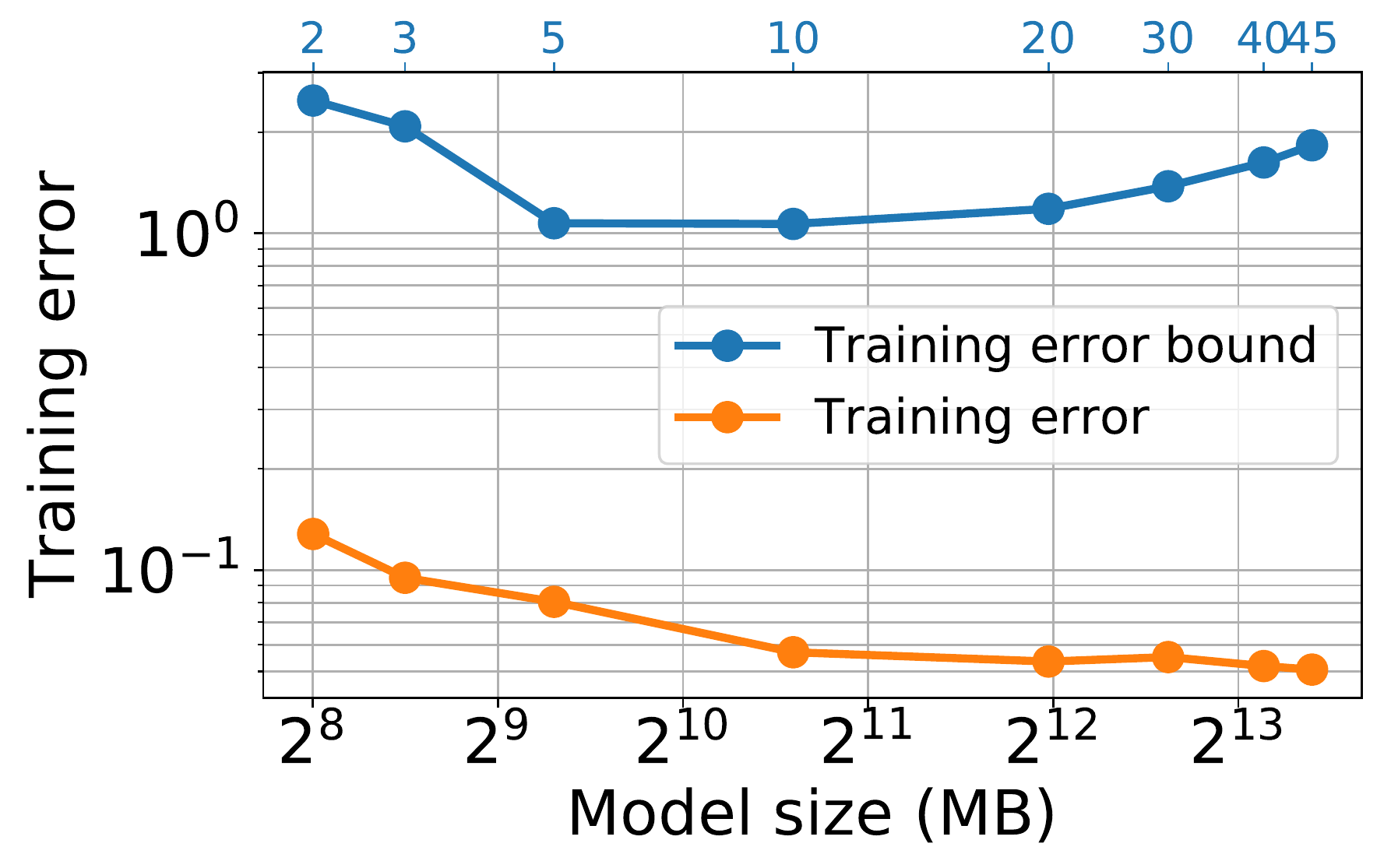}

\caption{First row: Average binary loss ($\varepsilon$) vs Model size.
%The error bars show one standard deviation over the $\ell$ predictors and 5 repetitions.
Second row: Multiclass training error and multiclass training error bound (on a logarithmic scale) vs Model size.
The secondary x-axes (top axes, \textcolor[rgb]{0.17,0.50,0.72}{blue}) indicate the slice widths (b) used for the W-LTLS trellis graphs.}
\label{fig:avgBinaryLossVsPredictorsNumber}
\end{center}
\vskip -0.2in
\end{figure*}

\subsection{Multiclass training error}
\label{averageBinaryLossVsBalance}
%

%In all datasets we observe a decrease of the average error as the number
%of functions $\ell$ (and also the width $b$) grows, and the trend is
%very clear for the three larger datasets : {\tt aloi\_bin, Dmoz,
%  LSHTC1}.
% The decrease is sharp for low values of $\ell$ and then
%practically almost converges (to zero).

% We also observe that the
%standard deviation also decreases as a function of $\ell$, implying
%that the induced binary subproblems become easier jointly -- not only
%in their expectation, but also as a collection.

In the second row of \figref{fig:avgBinaryLossVsPredictorsNumber} we plot the multiclass training
error (when using loss-based decoding defined in \eqref{loss_based} with the squared hinge loss)
and its bound \eqref{bound}
%we compare the multiclass training error and its bound \eqref{bound}
for different model sizes \footnote{
The model size is linear in the number of predictors,
which in turn depends on the slice width like $\frac{b^2}{\log b}$.
%Similar to previous works, we report the model sizes of the final models used for inference
%(i.e. for AROW we neglect the covariance matrix used only for training).
} (MBytes).

%Note, in
%\figref{fig:avgBinaryLossVsPredictorsNumber} the x-axis is the number
%of predictors $\ell$ while in \figref{fig:totalBinaryLoss-vs-sliceWidth}
%the x-axis is the size of the model.
%This was done for better comparison with other algorithms,
%whose reported results use the model size,
%which is the actual quantity of interest.

For the bound, we set the minimum distance $\rho=4$ as explained in
\secref{sec:Wltls}, and $L_{SH}(0)=1$.
%The number of binary functions $\ell$ is determined by setting $b$,
%also determining the model size.
The average binary loss $\varepsilon$ was computed as in \eqref{avgBinaryLoss}.

%Each line in each panel (dataset) in the first row of \figref{fig:wltls} consists of the same number of
%markers, corresponding to the same configurations.
%For example, the right panel (corresponding to {\tt LSHTC1}) shows two lines, each connecting eight  markers.
%The left marker in each line corresponds to the same
%run, and so on.  The fourth marker in the right panel of
%\figref{fig:avgBinaryLossVsPredictorsNumber} corresponds to a run with
%$\ell=338$ binary functions, and average binary error of $0.011$. This
%result also corresponds to the fourth marker in the first row of
%\figref{fig:wltls} with a model of size
%$2320\mathrm{MB}$ ($\gtrapprox 2^{11} \mathrm{MB}$), multiclass training error of $0.0675$,
%and a multiclass training error bound of $0.959$.

For all datasets, the multiclass training error follows qualitatively its bound.
For the two larger datasets, shown in the two right panels,
both the error and its bound decrease to some point,
and then
%after approximately a similar number of binary learners,
start to increase.
This can be explained as follows:
%the convergence of $\varepsilon$ to $0$ shown in the first row of \figref{fig:avgBinaryLossVsPredictorsNumber}.
at some point, the increase in the slice widths (and $\ell$ and the model size), stops to
significantly decrease $\epsilon$ (see first row of \figref{fig:avgBinaryLossVsPredictorsNumber}), such that the term $\ell\times\varepsilon$ appears in
the training error bound \eqref{bound} overall starts increasing (recall that the denominator $\rho\times L\left(0\right)$ is constant).
By comparing these plots to the multiclass \emph{test} accuracy plots in \figref{fig:wltls}, we observe that at the same point where the training error and its bound start to increase,
the test accuracy does not increase significantly anymore.
For example, for {\tt LSHTC1} and {\tt Dmoz} datasets, the training error bounds start to increase at around model size of
% $2^{11}$
 \revised{$2^{12}$}
 , and at the same time the test accuracy stops increasing significantly. This suggests that model size of
 %$2^{11}$
 \revised{$2^{12}$}
   is a good point in terms of accuracy/model size tradeoff.

\subsection{Average predictions margin}
\label{averageMargin}

As discussed in \secref{lbd-simulations},
the following \figref{fig:marginIncrease} shows that for larger values of $b$,
the predictions margin increases.

\begin{figure*}[h]
\begin{center}
\includegraphics[width=.45\linewidth]{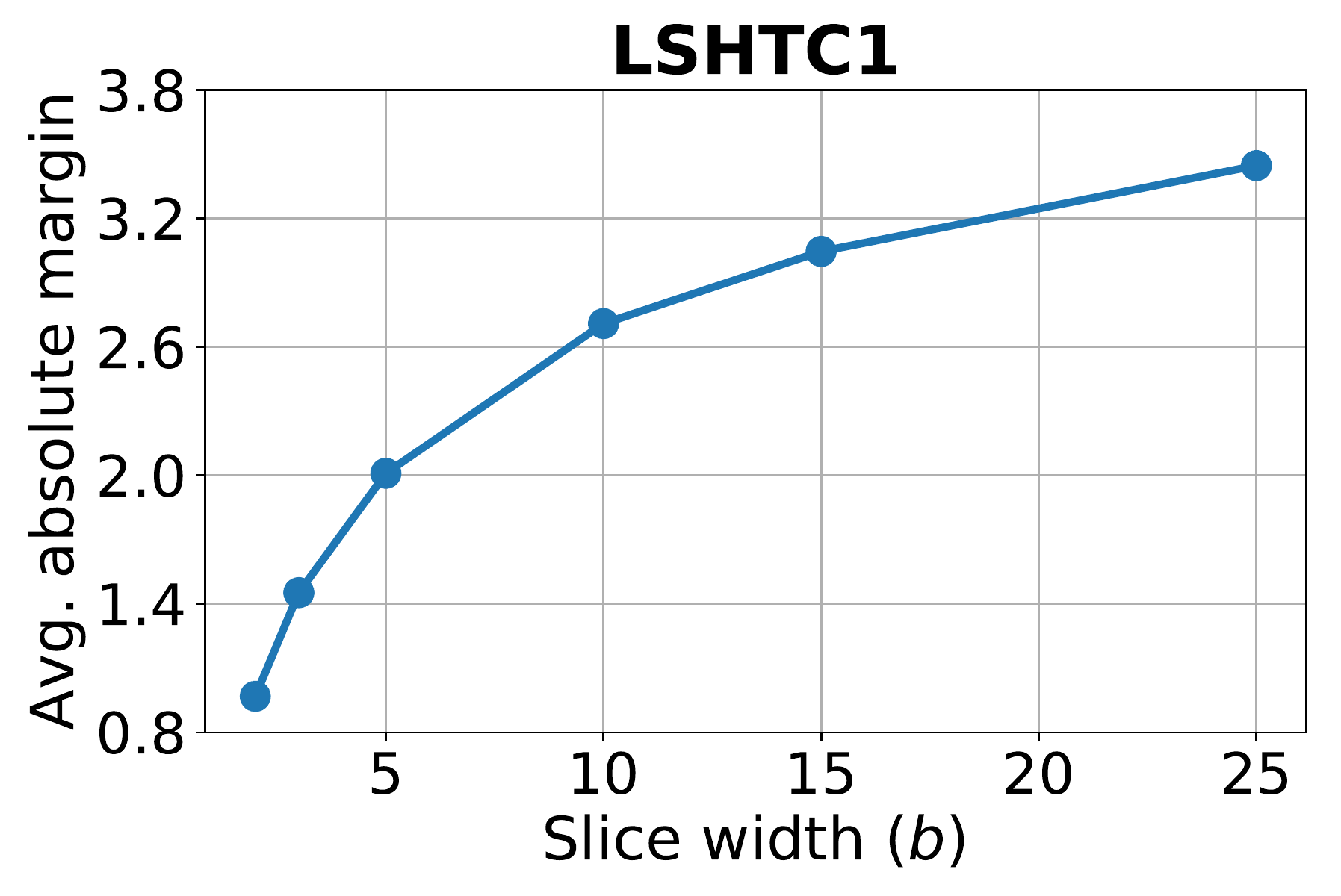}

\caption{
The average absolute margin, i.e.
$\frac{1}{m\ell}\sum_{i=1}^{m}{\sum_{j=1}^{\ell}{\left|f_j\left(x_i\right)\right|}}$,
vs the slice width $b$.
}
\label{fig:marginIncrease}
\end{center}
\vskip -0.2in
\end{figure*}

%%%%%%%%%%%%%%%%%%%%%%%%%%%%%%%%%%%%%%%%%%%%%%%%%%%%%%%%%%%%%%%%%%%%%%%%%%%%%%%%%%%%%%%%

\newpage

\subsection{\revised{Experimental results of the multiclass test accuracy experiments}}
\label{sec:experimentalDenseResults}

Following are the results of \secref{WLTLS-simulations} organized in a tabular form --
the model sizes and prediction times of the tested algorithms and their test accuracy.
The results of W-LTLS are averaged on five runs.

\begin{table}[h!]
\begin{center}
\begin{small}
\begin{sc}
\begin{tabular}{|l|l|l|r|r|r|r|}
\toprule
Dataset  & \multicolumn{2}{l|}{Algorithm} & Model  & Prediction  & Test    \\
  &  \multicolumn{2}{l|}{}  & size (Bytes) & time (sec) &  accuracy (\%)    \\
 \midrule
 \hline

%######################################################################
% sector
\multirow{9}{*}{
    \textbf{sector}
}
 &  \multicolumn{2}{l|}{LTLS} & 5.9 MB & 0.14 & 88.45 \\
 \cline{2-6}
& \multirow{5}{*}{
    W-LTLS
    }
     & $b=2$ & 5.9 MB & 0.14 & 91.63 \\
  \cline{3-6}
     & & $b=4$ & 9.7 MB & 0.16 & 93.92 \\
  \cline{3-6}
     & & $b=5$ & 11.6 MB & 0.16 & 94.32 \\
  \cline{3-6}
     & & $b=7$ & 15.4 MB & 0.18 & 94.86 \\
  \cline{3-6}
     & & $b=10$ & 26.5 MB & 0.23 & 94.88 \\
  \cline{2-6}
 &  \multicolumn{2}{l|}{LOMTree} & 17.0 MB & 0.16 & 82.10 \\
 \cline{2-6}
 &  \multicolumn{2}{l|}{FastXML} & 8.9 MB & 0.32 & 86.26 \\
 \cline{2-6}
 &  \multicolumn{2}{l|}{OVR} & 41.3 MB & - & 94.79 \\
 \cline{2-6}
\hline
\hline
%######################################################################
% aloi_bin
\multirow{11}{*}{
    \textbf{aloi.bin}
}
 &  \multicolumn{2}{l|}{LTLS} & 102.1 MB & 1.0 & 82.24 \\
 \cline{2-6}
& \multirow{7}{*}{
    W-LTLS
    }
     & $b=2$ & 102.1 MB & 1.4 & 85.03 \\
  \cline{3-6}
     & & $b=3$ & 133.6 MB & 1.5 & 89.39 \\
  \cline{3-6}
     & & $b=5$ & 216.2 MB & 1.7 & 92.49 \\
  \cline{3-6}
     & & $b=7$ & 328.0 MB & 2.1 & 93.70 \\
  \cline{3-6}
     & & $b=10$ & 537.0 MB & 2.7 & 94.92 \\
  \cline{3-6}
     & & $b=15$ & 777.5 MB & 3.5 & 94.93 \\
  \cline{3-6}
     & & $b=20$ & 1.1 GB & 5.0 & 95.16 \\
  \cline{2-6}
 &  \multicolumn{2}{l|}{LOMTree} & 106.0 MB & 1.6 & 89.47 \\
 \cline{2-6}
 &  \multicolumn{2}{l|}{FastXML} & 522.0 MB & 5.3 & 95.38 \\
 \cline{2-6}
 &  \multicolumn{2}{l|}{OVR} & 2.4 GB & - & 95.90 \\
 \cline{2-6}
\hline
\hline
%######################################################################
% imageNet
\multirow{10}{*}{
    \textbf{imageNet}
}
 &  \multicolumn{2}{l|}{LTLS} & 0.16 MB & 15.0 & 0.75 \\
 \cline{2-6}
& \multirow{6}{*}{
    W-LTLS
    }
     & $b=2$ & 0.16 MB & 21.4 & 0.42 \\
  \cline{3-6}
     & & $b=5$ & 0.34 MB & 26.3 & 1.59 \\
  \cline{3-6}
     & & $b=10$ & 0.84 MB & 40.3 & 4.36 \\
  \cline{3-6}
     & & $b=15$ & 1.2 MB & 52.6 & 4.08 \\
  \cline{3-6}
     & & $b=20$ & 1.8 MB & 76.2 & 7.01 \\
  \cline{3-6}
     & & $b=30$ & 3.7 MB & 194.2 & 12.60 \\
  \cline{2-6}
 &  \multicolumn{2}{l|}{LOMTree} & 35.0 MB & 37.7 & 5.37 \\
 \cline{2-6}
 &  \multicolumn{2}{l|}{FastXML} & 1.6 GB & 172.3 & 6.84 \\
 \cline{2-6}
 &  \multicolumn{2}{l|}{OVR} & 3.8 MB & - & 15.60 \\
 \cline{2-6}
\hline
\hline
%######################################################################
% Dmoz
\multirow{12}{*}{
    \textbf{Dmoz}
}
 &  \multicolumn{2}{l|}{LTLS} & 193.9 MB & 5.2 & 23.04 \\
 \cline{2-6}
& \multirow{8}{*}{
    W-LTLS
    }
     & $b=2$ & 193.9 MB & 7.5 & 25.76 \\
  \cline{3-6}
     & & $b=3$ & 248.0 MB & 8.1 & 29.56 \\
  \cline{3-6}
     & & $b=5$ & 429.2 MB & 9.8 & 33.92 \\
  \cline{3-6}
     & & $b=10$ & 1.1 GB & 16.6 & 37.44 \\
  \cline{3-6}
     & & $b=20$ & 2.8 GB & 41.4 & 38.43 \\
  \cline{3-6}
     & & $b=30$ & 4.3 GB & 73.6 & 38.89 \\
  \cline{3-6}
     & & $b=40$ & 6.3 GB & 164.5 & 38.81 \\
  \cline{3-6}
     & & $b=50$ & 9.0 GB & 332.3 & 38.89 \\
  \cline{2-6}
 &  \multicolumn{2}{l|}{LOMTree} & 1.8 GB & 28.0 & 21.27 \\
 \cline{2-6}
 &  \multicolumn{2}{l|}{FastXML} & 1.2 GB & 60.7 & 38.58 \\
 \cline{2-6}
 &  \multicolumn{2}{l|}{OVR} & 38.0 GB & - & 35.50 \\
 \cline{2-6}
\hline
\hline
%######################################################################
% LSHTC1
\multirow{12}{*}{
    \textbf{LSHTC1}
}
 &  \multicolumn{2}{l|}{LTLS} & 256.3 MB & 0.65 & 9.50 \\
 \cline{2-6}
& \multirow{8}{*}{
    W-LTLS
    }
     & $b=2$ & 256.3 MB & 1.1 & 10.56 \\
  \cline{3-6}
     & & $b=3$ & 361.6 MB & 1.1 & 13.72 \\
  \cline{3-6}
     & & $b=5$ & 631.6 MB & 1.3 & 16.58 \\
  \cline{3-6}
     & & $b=10$ & 1.5 GB & 2.2 & 20.44 \\
  \cline{3-6}
     & & $b=20$ & 4.0 GB & 5.1 & 21.76 \\
  \cline{3-6}
     & & $b=30$ & 6.3 GB & 9.9 & 22.08 \\
  \cline{3-6}
     & & $b=40$ & 9.0 GB & 20.4 & 22.30 \\
  \cline{3-6}
     & & $b=45$ & 10.8 GB & 29.3 & 22.47 \\
  \cline{2-6}
 &  \multicolumn{2}{l|}{LOMTree} & 744.0 MB & 6.8 & 10.56 \\
 \cline{2-6}
 &  \multicolumn{2}{l|}{FastXML} & 366.6 MB & 10.7 & 21.66 \\
 \cline{2-6}
 &  \multicolumn{2}{l|}{OVR} & 56.3 GB & - & 21.90 \\ 

 \bottomrule

 %%%%%%%%%%%%%%%%%%%%%%%%%%%%%%%%%%%%%%%%%%%%%%%%%%%%%%%

\end{tabular}
\end{sc}
\end{small}
\end{center}
\end{table}

\subsection{\revised{Experimental results of the sparsity experiments}}
\label{app:sparsityExperimental}

%\begin{figure*}[h]
%\begin{center}
%\includegraphics[width=.25\linewidth]{hist_sector.pdf}
%\includegraphics[width=.24\linewidth]{hist_aloi_bin.pdf}
%\includegraphics[width=.24\linewidth]{hist_Dmoz.pdf}
%\includegraphics[width=.24\linewidth]{hist_LSHTC1.pdf}
%\caption{
%Histogram of the number of different features used by all the training samples of the same class.
%The vast majority of classes use a very small percentage of features.
%}
%\label{fig:features_hist}
%\end{center}
%\vskip -0.2in
%\end{figure*}

In the following \figref{fig:nonzero} we show that wider graphs
induce models which use less features, even before pruning.

This may help understand the results in \secref{sec:sparsityExperiment},
where we show that the larger slice widths allow pruning more weights
without accuracy degradation.

\begin{figure*}[h]
\begin{center}
\includegraphics[width=.25\linewidth]{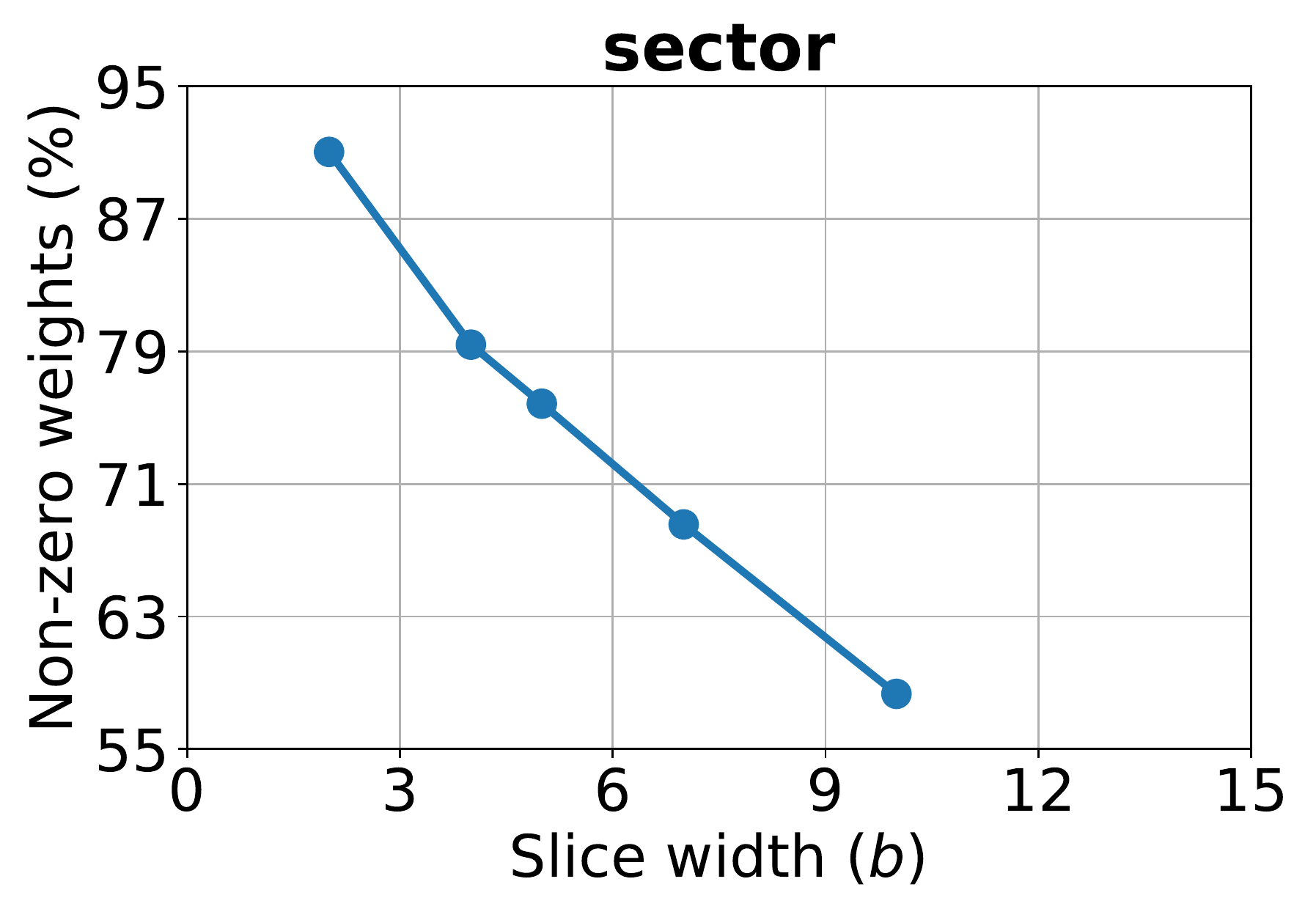}
\includegraphics[width=.24\linewidth]{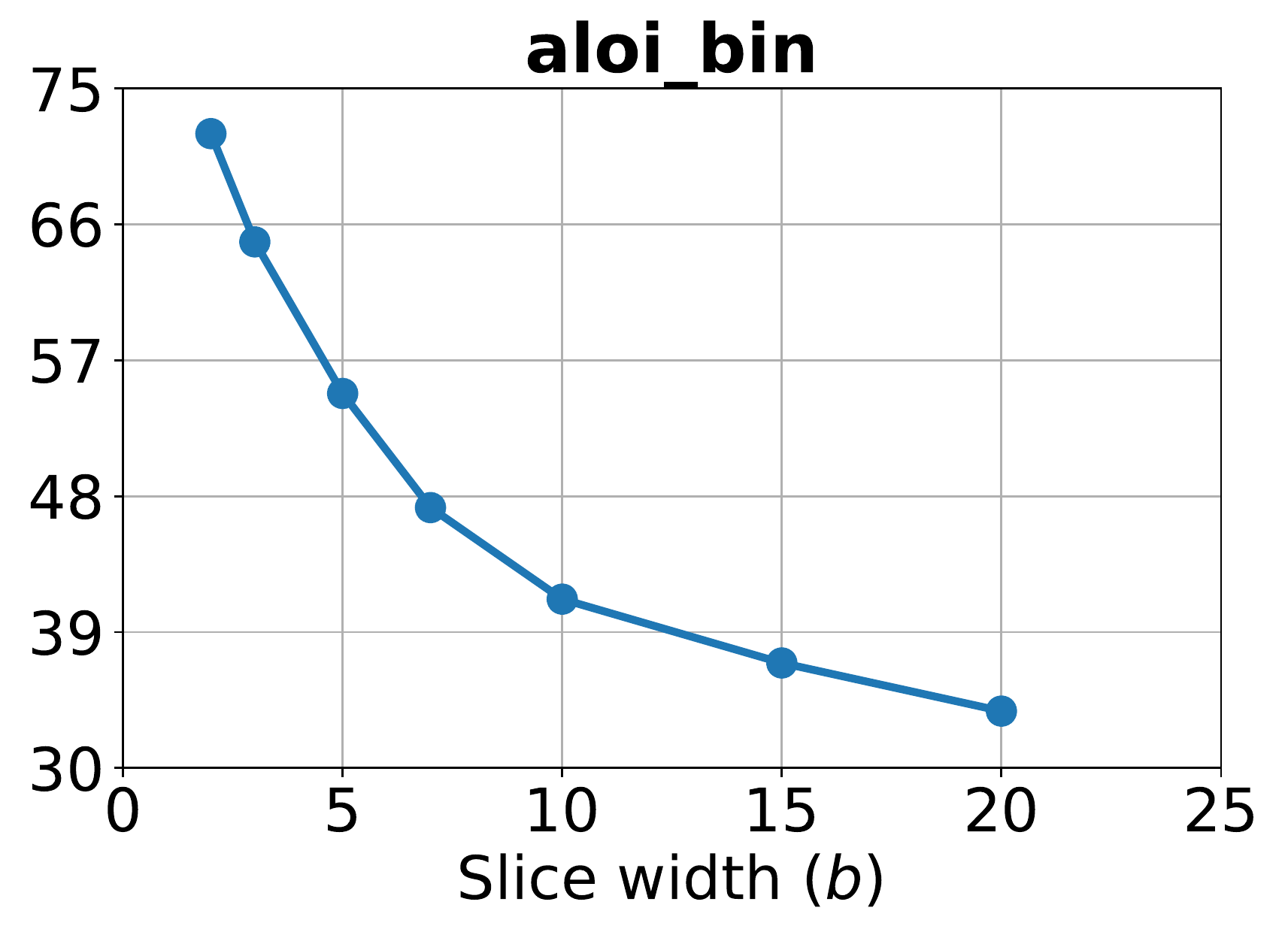}
\includegraphics[width=.24\linewidth]{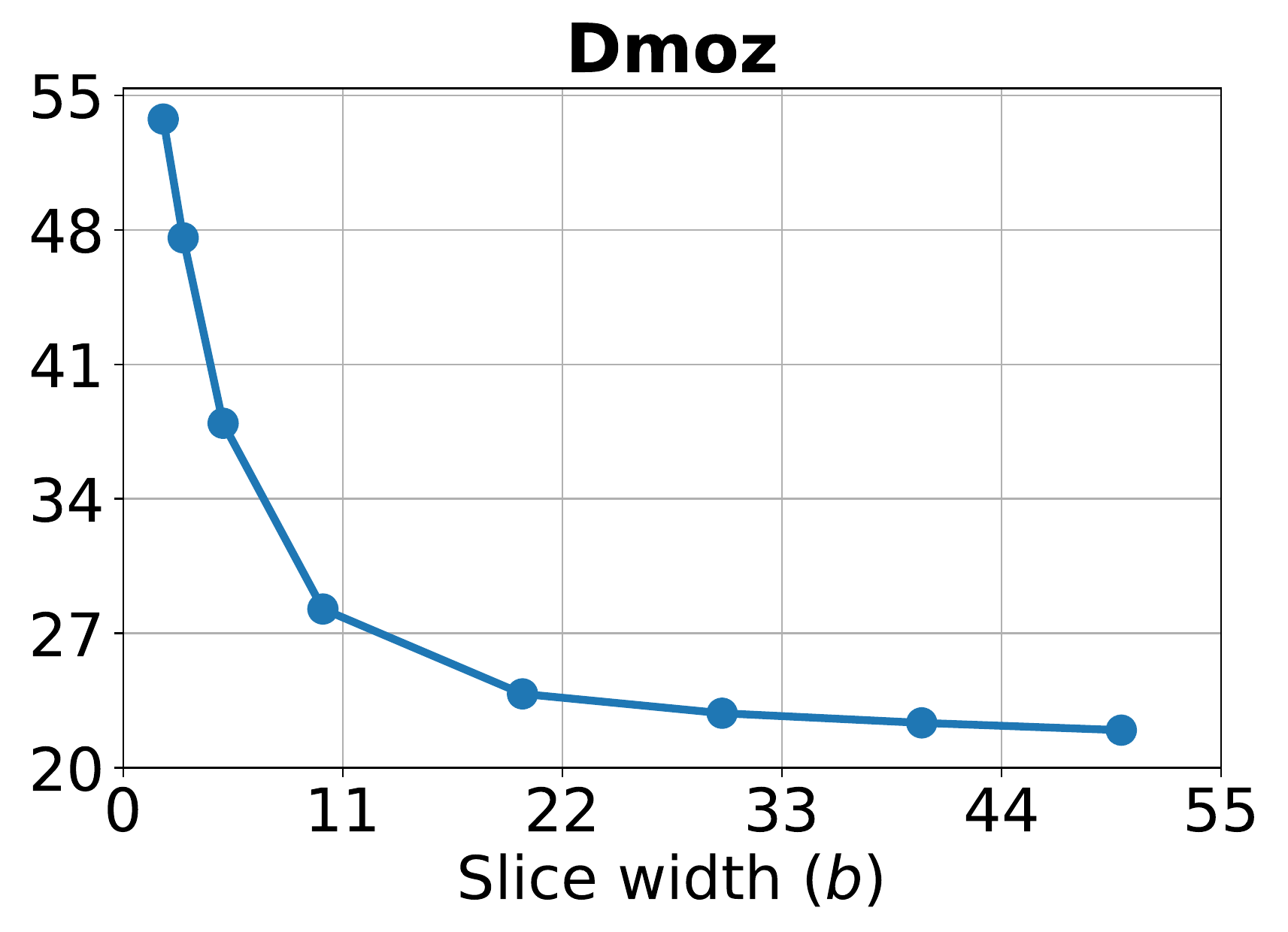}
\includegraphics[width=.24\linewidth]{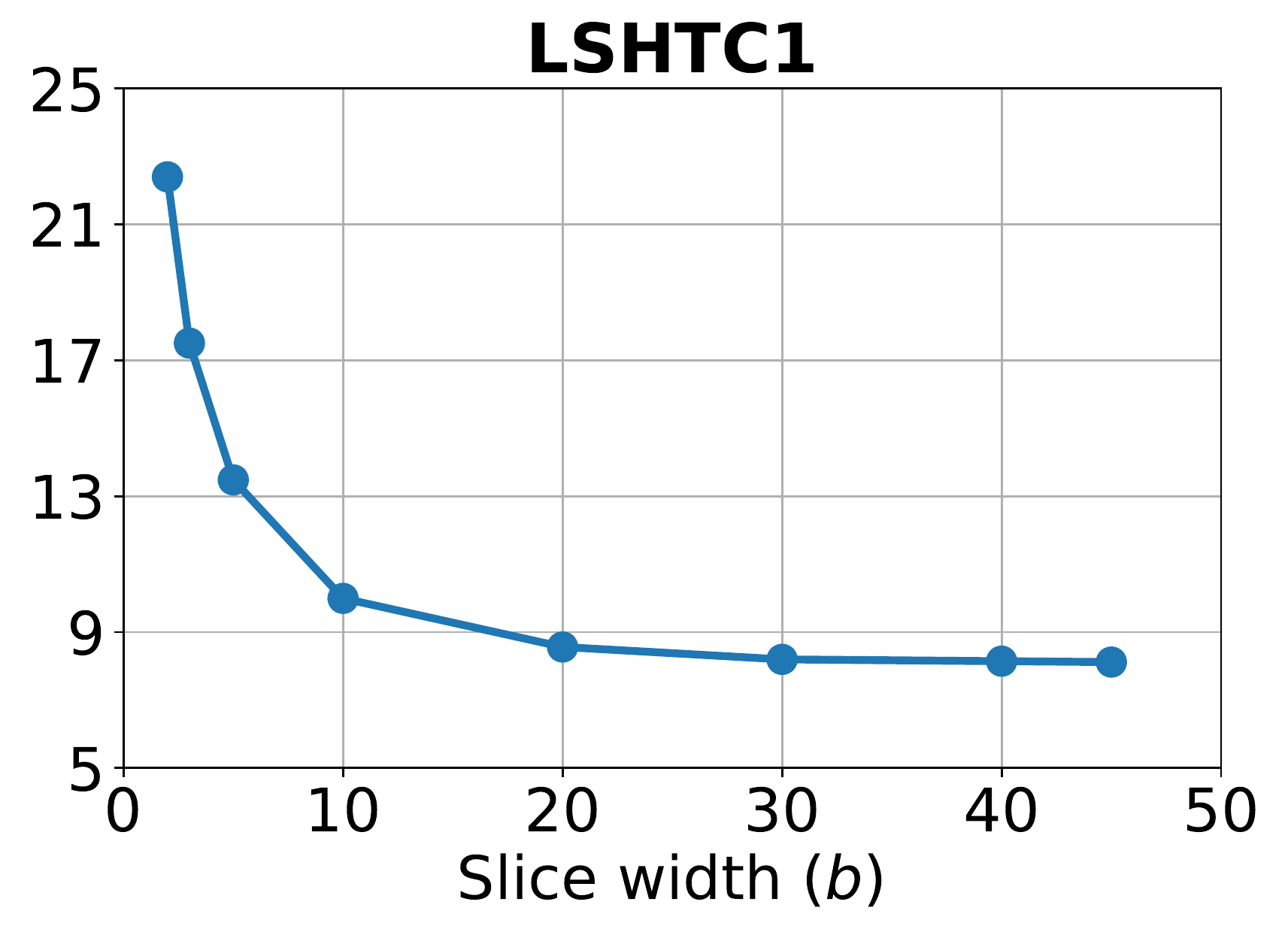}

\caption{
Percentage of non-zero weights at end of training (before pruning) vs
the slice width $b$.
}
\label{fig:nonzero}
\end{center}
\vskip 0.4in
\end{figure*}

In the following \tabref{table:sparseResults},
the results of \secref{sec:sparsityExperiment} organized in a tabular form.

\begin{table}[h!]
\begin{center}
\begin{small}
\begin{sc}
\begin{tabular}{|l|l|l|l|r|r|}
\toprule
Dataset  & \multicolumn{3}{l|}{Algorithm}   & Model        & Test    \\
         &  \multicolumn{3}{l|}{}           & size (Bytes) &  accuracy (\%)    \\
 \midrule
 \hline

%######################################################################
% sector
\multirow{11}{*}{
    \textbf{sector}
}
& \multirow{5}{*}{
    Sparse W-LTLS
    }
     & $b=2$ & ${\lambda}=0.324$ & 1.4 MB & 90.66 \\
  \cline{3-6}
     & & $b=4$ & ${\lambda}=0.303$ & 1.0 MB & 93.42 \\
  \cline{3-6}
     & & $b=5$ & ${\lambda}=0.327$ & 1.1 MB & 93.86 \\
  \cline{3-6}
     & & $b=7$ & ${\lambda}=0.341$ & 1.0 MB & 94.11 \\
  \cline{3-6}
     & & $b=10$ & ${\lambda}=0.358$ & 0.74 MB & 94.82 \\
  \cline{2-6}
& \multirow{6}{*}{
    FastXML
    }
     & \multicolumn{2}{l|}{$T=25$} & 4.5 MB & 86.26 \\
  \cline{3-6}
     & & \multicolumn{2}{l|}{$T=50$} & 8.9 MB & 86.26 \\
  \cline{3-6}
     & & \multicolumn{2}{l|}{$T=75$} & 13.4 MB & 86.78 \\
  \cline{3-6}
     & & \multicolumn{2}{l|}{$T=100$} & 17.9 MB & 86.58 \\
  \cline{3-6}
     & & \multicolumn{2}{l|}{$T=150$} & 26.8 MB & 86.57 \\
  \cline{3-6}
     & & \multicolumn{2}{l|}{$T=200$} & 35.8 MB & 87.20 \\
  \cline{2-6}
\hline
\hline
%######################################################################
% aloi_bin
\multirow{12}{*}{
    \textbf{aloi.bin}
}
& \multirow{7}{*}{
    Sparse W-LTLS
    }
     & $b=2$ & ${\lambda}=0.015$ & 37.3 MB & 84.09 \\
  \cline{3-6}
     & & $b=3$ & ${\lambda}=0.015$ & 32.7 MB & 88.59 \\
  \cline{3-6}
     & & $b=5$ & ${\lambda}=0.017$ & 25.4 MB & 91.74 \\
  \cline{3-6}
     & & $b=7$ & ${\lambda}=0.016$ & 25.9 MB & 92.89 \\
  \cline{3-6}
     & & $b=10$ & ${\lambda}=0.013$ & 30.3 MB & 94.48 \\
  \cline{3-6}
     & & $b=15$ & ${\lambda}=0.014$ & 31.2 MB & 94.27 \\
  \cline{3-6}
     & & $b=20$ & ${\lambda}=0.016$ & 25.4 MB & 94.55 \\
  \cline{2-6}
& \multirow{5}{*}{
    FastXML
    }
     & \multicolumn{2}{l|}{$T=5$} & 52.1 MB & 92.67 \\
  \cline{3-6}
     & & \multicolumn{2}{l|}{$T=10$} & 104.2 MB & 94.12 \\
  \cline{3-6}
     & & \multicolumn{2}{l|}{$T=25$} & 260.9 MB & 95.19 \\
  \cline{3-6}
     & & \multicolumn{2}{l|}{$T=50$} & 522.0 MB & 95.38 \\
  \cline{3-6}
     & & \multicolumn{2}{l|}{$T=100$} & 1.0 GB & 95.66 \\
  \cline{2-6}
 &  \multicolumn{3}{l|}{DiSMEC} & 10.7 MB & 96.28 \\
 \cline{2-6}
 &  \multicolumn{3}{l|}{PD-Sparse} & 12.7 MB & 96.20 \\
 \cline{2-6}
 &  \multicolumn{3}{l|}{PPDSparse} & 9.3 MB & 96.38 \\
 \cline{2-6}
\hline
\hline
%######################################################################
% Dmoz
\multirow{14}{*}{
    \textbf{Dmoz}
}
& \multirow{8}{*}{
    Sparse W-LTLS
    }
     & $b=2$ & ${\lambda}=0.347$ & 43.2 MB & 24.64 \\
  \cline{3-6}
     & & $b=3$ & ${\lambda}=0.325$ & 56.1 MB & 28.63 \\
  \cline{3-6}
     & & $b=5$ & ${\lambda}=0.314$ & 69.3 MB & 33.04 \\
  \cline{3-6}
     & & $b=10$ & ${\lambda}=0.323$ & 93.7 MB & 36.75 \\
  \cline{3-6}
     & & $b=20$ & ${\lambda}=0.345$ & 126.3 MB & 37.72 \\
  \cline{3-6}
     & & $b=30$ & ${\lambda}=0.385$ & 145.6 MB & 38.08 \\
  \cline{3-6}
     & & $b=40$ & ${\lambda}=0.374$ & 193.1 MB & 38.08 \\
  \cline{3-6}
     & & $b=50$ & ${\lambda}=0.323$ & 324.3 MB & 38.16 \\
  \cline{2-6}
& \multirow{6}{*}{
    FastXML
    }
     & \multicolumn{2}{l|}{$T=1$} & 24.5 MB & 27.09 \\
  \cline{3-6}
     & & \multicolumn{2}{l|}{$T=2$} & 49.2 MB & 31.17 \\
  \cline{3-6}
     & & \multicolumn{2}{l|}{$T=5$} & 123.0 MB & 35.84 \\
  \cline{3-6}
     & & \multicolumn{2}{l|}{$T=10$} & 246.5 MB & 37.47 \\
  \cline{3-6}
     & & \multicolumn{2}{l|}{$T=50$} & 1.2 GB & 38.58 \\
  \cline{3-6}
     & & \multicolumn{2}{l|}{$T=300$} & 7.4 GB & 38.63 \\
  \cline{2-6}
 &  \multicolumn{3}{l|}{DiSMEC} & 259.3 MB & 39.38 \\
 \cline{2-6}
 &  \multicolumn{3}{l|}{PD-Sparse} & 453.3 MB & 39.91 \\
 \cline{2-6}
 &  \multicolumn{3}{l|}{PPDSparse} & 526.7 MB & 39.32 \\
 \cline{2-6}
\hline
\hline
%######################################################################
% LSHTC1
\multirow{15}{*}{
    \textbf{LSHTC1}
}
& \multirow{8}{*}{
    Sparse W-LTLS
    }
     & $b=2$ & ${\lambda}=0.288$ & 24.7 MB & 9.48 \\
  \cline{3-6}
     & & $b=3$ & ${\lambda}=0.237$ & 27.6 MB & 12.66 \\
  \cline{3-6}
     & & $b=5$ & ${\lambda}=0.271$ & 31.5 MB & 15.60 \\
  \cline{3-6}
     & & $b=10$ & ${\lambda}=0.226$ & 46.3 MB & 19.78 \\
  \cline{3-6}
     & & $b=20$ & ${\lambda}=0.207$ & 85.0 MB & 21.13 \\
  \cline{3-6}
     & & $b=30$ & ${\lambda}=0.192$ & 130.5 MB & 21.56 \\
  \cline{3-6}
     & & $b=40$ & ${\lambda}=0.183$ & 184.7 MB & 21.81 \\
  \cline{3-6}
     & & $b=45$ & ${\lambda}=0.246$ & 152.1 MB & 21.88 \\
  \cline{2-6}
& \multirow{7}{*}{
    FastXML
    }
     & \multicolumn{2}{l|}{$T=1$} & 7.3 MB & 12.50 \\
  \cline{3-6}
     & & \multicolumn{2}{l|}{$T=5$} & 36.6 MB & 17.98 \\
  \cline{3-6}
     & & \multicolumn{2}{l|}{$T=10$} & 73.3 MB & 19.36 \\
  \cline{3-6}
     & & \multicolumn{2}{l|}{$T=50$} & 366.6 MB & 21.66 \\
  \cline{3-6}
     & & \multicolumn{2}{l|}{$T=150$} & 1.1 GB & 21.94 \\
  \cline{3-6}
     & & \multicolumn{2}{l|}{$T=300$} & 2.2 GB & 22.04 \\
  \cline{3-6}
     & & \multicolumn{2}{l|}{$T=1200$} & 8.8 GB & 21.92 \\
  \cline{2-6}
 &  \multicolumn{3}{l|}{DiSMEC} & 94.7 MB & 22.74 \\
 \cline{2-6}
 &  \multicolumn{3}{l|}{PD-Sparse} & 58.7 MB & 22.46 \\
 \cline{2-6}
 &  \multicolumn{3}{l|}{PPDSparse} & 254.0 MB & 22.70 \\

 \bottomrule

 %%%%%%%%%%%%%%%%%%%%%%%%%%%%%%%%%%%%%%%%%%%%%%%%%%%%%%%

\end{tabular}
\end{sc}
\end{small}
\end{center}
\caption{Simulation results for the sparse models.}
\label{table:sparseResults}
\end{table}

\end{document}